\documentclass[11pt]{article}
\usepackage[utf8]{inputenc}

\usepackage[top=1in, bottom=1in, left=1in, right=1in]{geometry}

\usepackage[round]{natbib}
\usepackage{hyperref}		%
\hypersetup{
  colorlinks,
  linkcolor={red!50!black},
  citecolor={blue!50!black},
  urlcolor={blue!80!black}
}
\usepackage{amsmath,amsfonts,amssymb,mathrsfs,mathtools}

\usepackage{stmaryrd}
\usepackage{amsthm, textcomp}
\usepackage{color}
\usepackage{changes}%
\usepackage[ruled,vlined,linesnumbered]{algorithm2e}
\usepackage[noabbrev, capitalise]{cleveref} %

\usepackage{dsfont}

\usepackage[font=small,labelfont=bf,margin=1cm]{caption} 

\usepackage{enumitem}

\usepackage{listings,xcolor}
\newcommand{\E}{\mathbb{E}}
\newcommand{\N}{\mathbb{N}}
\newcommand{\norm}[1]{\left\|#1\right\|}

\newcommand{\de}{\partial}
\newcommand{\wt}{\widetilde}

\newcommand{\mb}{\mathbf}
\newcommand{\ol}{\overline}

\DeclareMathOperator{\pol}{pol}

\newcommand{\R}{\mathbb{R}}
\newcommand{\abs}[1]{\left|#1\right|}

\numberwithin{equation}{section}

\theoremstyle{plain}
\newtheorem{theorem}{Theorem}[section]

\newtheorem{corollary}[theorem]{Corollary}
\newtheorem{lemma}[theorem]{Lemma}
\newtheorem{proposition}[theorem]{Proposition}

\newtheorem{assumption}[theorem]{Assumption}
\crefname{assumption}{Assumption}{Assumption}

\theoremstyle{definition}
\newtheorem{definition}[theorem]{Definition}

\crefname{equation}{}{}

\usepackage{tikz}
\usetikzlibrary{arrows,chains,matrix,positioning,scopes,fit}
\usetikzlibrary{arrows.meta}
\makeatletter
\tikzset{join/.code=\tikzset{after node path={%
\ifx\tikzchainprevious\pgfutil@empty\else(\tikzchainprevious)%
edge[every join]#1(\tikzchaincurrent)\fi}}}
\makeatother
\tikzset{>=stealth',every on chain/.append style={join},
         every join/.style={->}}
\tikzstyle{labeled}=[execute at begin node=$\scriptstyle,
   execute at end node=$]

\newcommand{\Max}[1]{}
\newcommand{\red}{\textcolor{red}}
\definecolor{OliveGreen}{rgb}{0,0.6,0}
\newcommand{\bill}[1]{}
\title{Higher Order Generalization Error for \\ 
First Order Discretization of Langevin Diffusion}
  \author{Mufan (Bill) Li\thanks{
  Department of Statistical Sciences at
   University of Toronto, and Vector Institute, \texttt{mufan.li@mail.utoronto.ca}
   }
   \thanks{This work was completed during an internship at Borealis AI, 
          supported by the MITACS Accelerate Fellowship.}
   \and
   Maxime Gazeau\thanks{
    LG Toronto AI Lab, \texttt{maxime.yves.gazeau@gmail.com}
 }
}

\begin{document}

\maketitle

\begin{abstract}

We propose a novel approach to analyze generalization error 
for discretizations of Langevin diffusion, 
such as the stochastic gradient Langevin dynamics (SGLD). 
For an $\epsilon$ tolerance of expected generalization error, 
it is known that a first order discretization can reach this target 
if we run $\Omega(\epsilon^{-1} \log (\epsilon^{-1}) )$ iterations 
with $\Omega(\epsilon^{-1})$ samples. 
In this article, we show that with additional smoothness assumptions, 
even first order methods can achieve arbitrarily runtime complexity. 
More precisely, for each $N>0$, 
we provide a sufficient smoothness condition on the loss function 
such that a first order discretization 
can reach $\epsilon$ expected generalization error 
given $\Omega( \epsilon^{-1/N} \log (\epsilon^{-1}) )$ iterations 
with $\Omega(\epsilon^{-1})$ samples. 

\end{abstract}

\tableofcontents

\section{Introduction}
\label{sec:intro}

Let $f:\mathbb{R}^d \times \mathcal{Z} \to \mathbb{R}$ 
be a known loss function with respect to 
parameters $x \in \mathbb{R}^d$ 
and a single data point $z\in \mathcal{Z}$. 
We define $F(x) := \mathbb{E} \, f(x, z)$ 
where the expectation is over $z \sim \mathcal{D}$ 
for an unknown distribution $\mathcal{D}$ over $\mathcal{Z}$. 
We consider the problem of minimizing an expected loss function 
\begin{equation}
	\min_{x \in \mathbb{R}^d} F(x) \,. 
\end{equation}

Since the distribution is unknown, 
we rely on a sample of $n$-independent and 
identically distributed (i.i.d.) data points 
$\mb{z} = \{z_i\}_{i=1}^n \sim \mathcal{D}^n$, 
and evaluate the empirical loss function denoted by 
$F_{\mb{z}}(x) := \frac{1}{n} \sum_{i=1}^n f(x, z_i)$. 
In this setting, a learning algorithm is a random map 
$X : \mathcal{Z}^n \to \mathbb{R}^d$, 
and we use $X_{\mb{z}}$ to denote the output. 
Since the empirical loss $F_{\mb{z}}(X_{\mb{z}})$ can be evaluated, 
it remains to study the expected generalization error 
\begin{equation}
\label{eq:gen_err}
	\mathbb{E} \, F(X_{\mb{z}}) 
	- \mathbb{E} \, F_{\mb{z}}(X_{\mb{z}}) \,, 
\end{equation}
where the expectation is over both 
the randomness of $X$ and $\mb{z} \sim \mathcal{D}^n$. 

In this article, we are interested in the class of algorithms 
that can be viewed as a discretization of the (overdamped) Langevin diffusion, 
defined by the stochastic differential equation (SDE) 
\begin{equation}
\label{eq:langevin_diffusion}
	dX(t) 
	= 
	- \nabla F_{\mb{z}}( X(t) ) \, dt 
	+ \sqrt{ \frac{2}{\beta} } \, dW(t) \,, 
\end{equation}
where $\beta > 0$ is the inverse temperature parameter, 
and $\{ W(t) \}_{t \geq 0}$ is a standard Brownian motion in $\mathbb{R}^d$. 
It is well known that the Langevin diffusion 
converges (as $t \to \infty$) to the Gibbs distribution 
$\rho_{\mb{z}}(x) \propto \exp( -\beta F_{\mb{z}}(x) )$ 
\citep{bakry2013analysis}. 
Furthermore, using the uniform stability condition 
introduced by \citet{Bousquet:2002}, 
the Gibbs distribution is shown to have generalization error $O(n^{-1})$ 
\citep{raginsky2017nonconvex}. 
Many optimization results also rely on 
the Gibbs distribution's minimizing property 
\citep{raginsky2017nonconvex,XuCG17,erdogdu2018global}. 

Using this approach, \citet{XuCG17} studied 
a first order discretization of 
the Langevin diffusion $\{X_{\mb{z},k}\}_{k \geq 0}$ 
with step size $\eta > 0$, 
and showed the approximation error between 
the algorithm and the Gibbs distribution 
is on the order of $O( e^{ - \Omega( k \eta) } + \eta )$, 
where $k > 0$ is the number of iterations. 
In terms of runtime and sample complexity for an $\epsilon$ tolerance 
on generalization error, 
this implies we need to choose a small step size $\eta = O(\epsilon)$, 
leading to $k \geq \Omega( \epsilon^{-1} \log \epsilon^{-1} )$ 
and $n \geq \Omega( \epsilon^{-1} )$. 
This approach corresponds to the top path of approximation steps 
to generalization error in \cref{fig:approx_diagram}.

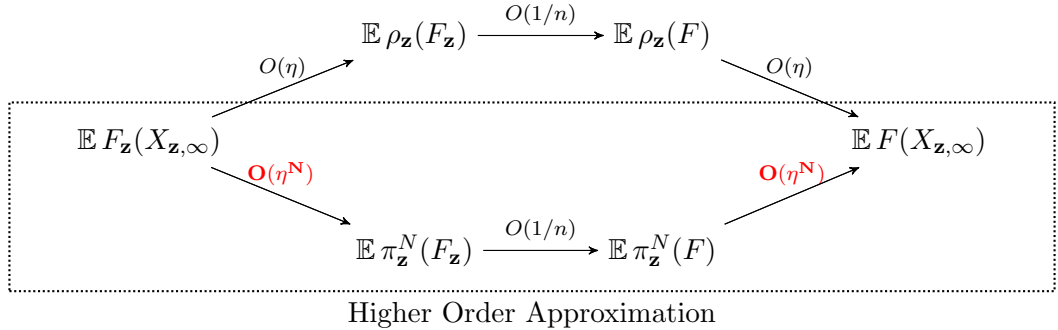
\begin{figure}[h!]
\centering
\begin{tikzpicture}[
ourbox/.style={draw=red, thick, densely dotted,
    inner xsep=2em,
    inner ysep=.5em},
prevbox/.style={draw=black, thick, densely dotted,
    inner xsep=2em,
    inner ysep=.5em}
]
\matrix (m) [matrix of math nodes, row sep=2em, column sep=4em, 
ampersand replacement=\&]
{ 
	\& 
	{\mathbb{E} \, \rho_{\mb{z}}(F_{\mb{z}})}  \& 
	\mathbb{E} \, \rho_{\mb{z}}(F) \& \\ 
	\mathbb{E} \, F_{\mb{z}}(X_{\mb{z}, \infty}) \& \& \& 
	\mathbb{E} \, F(X_{\mb{z}, \infty})
	\\ 
	\& \mathbb{E} \, \pi^N_{\mb{z}}(F_{\mb{z}}) \& 
	\mathbb{E} \, \pi^N_{\mb{z}}(F) \& \\ 
};
{ [start chain] \chainin (m-2-1);
	{[start branch=A] 
		\chainin (m-1-2) [join={node[midway, above, labeled] 
			{O(\eta)}}]; 
		\chainin (m-1-3) [join={node[midway, above,labeled] 
			{O(1/n)}}]; 
		\chainin (m-2-4) [join={node[midway, above, labeled] 
			{O(\eta)}}]; 
	}
	\chainin (m-3-2) [join={node[midway, above, labeled] 
		{\mathbf{\red{ O( \eta^N ) }} } }]; 
	\chainin (m-3-3) [join={node[above,labeled] 
		{O(1/n)}}]; 
	\chainin (m-2-4) [join={node[midway, above, labeled] 
		{\mathbf{\red{ O( \eta^N ) }} } }]; 
}
\node[prevbox, fit=(m-2-1) (m-3-2) (m-2-4), label={270:Higher Order Approximation}] {};
\end{tikzpicture}
\caption{
	Overview of approximation steps taken to bound 
	the generalization error 
	$\mathbb{E} \, F_{\mb{z}}(X_{\mb{z},\infty}) 
	- \mathbb{E} \, F(X_{\mb{z}, \infty})$. 
	Here we use $X_{\mb{z},\infty}$ to denote a sample 
	from the stationary distribution, 
	$O(\cdot)$ to display an approximation error bound with dependence  
	on step size $\eta$ and number of samples $n$, 
	and $\rho_{\mb{z}}(F)$ to denote the integral 
	$\int F \, d\rho_{\mb{z}}$. 
}
\label{fig:approx_diagram}
\end{figure}

In contrast to previous works, 
our approach avoids viewing first order methods as 
a discretization of Langevin diffusion. 
Instead, we view first order methods as 
higher order discretizations of a modified process, 
which we construct via weak backward error analysis 
\citep{DebusscheF11,Kopec13,kopec2015}. 
With the explicit construction, 
we can recover the stationary distribution $ \pi^N_{\mb{z}} $ 
of the modified process, 
which approximates the first order methods to order $O(\eta^N)$, 
where $N>0$ depends on the smoothness of $f(\cdot,z)$. 
Then our main result establishes a generalization bound of $O(n^{-1})$
for this distribution $\pi^N_{\mb{z}}$, 
which is of the same order as Gibbs. 
As a result, our approach shows that the generalization property 
of Gibbs is actually a generic property of Poisson equations. 
Putting these together, our results imply 
we can choose a much larger stepsize of $\eta = O(\epsilon^{1/N})$, 
leading to an improved runtime complexity of 
$k \geq \Omega( \epsilon^{-1/N} \log \epsilon^{-1} )$ 
while keeping the same sample complexity of 
$n \geq \Omega( \epsilon^{-1} )$. 
This is described in the bottom path in 
\cref{fig:approx_diagram}.

We summarize our main contributions as follows. 
We provide an explicit construction of 
the modified process approximating SGLD, 
as well as the stationary distribution $\pi^N_{\mb{z}}$, 
up to an error of order $O(\eta^N)$. 
Under additional smoothness conditions, 
we provide an improved generalization bound 
for first order discretizations of Langevin diffusion. 

The rest of the article is organized as follows. 
We discuss related works and comparison next in 
\cref{subsec:related}. 
In \cref{sec:main_results}, 
we introduce the precise notation and main results. 
In \cref{sec:sketch}, 
we provide an overview of the proofs, 
deferring technical details to the appendix. 
In \cref{sec:discussion}, 
we discuss the assumptions and further extensions 
of this work. 
In \cref{sec:appendix_weak_backward,sec:appendix_uniform_stability,sec:appendix_proof_corollary}, 
we provide the proof of the main results in full detail.

\subsection{Related Works}
\label{subsec:related}

Langevin algorithms for sampling and optimization 
have been very well studied 
\citep{gelfand1991recursive,raginsky2017nonconvex,cheng2018sharp,dalalyan2019user,durmus2017nonasymptotic,erdogdu2018global,li2019stochastic,vempala2019rapid,erdogdu2020convergence}. 
Most of these articles establish approximations 
using the Langevin diffusion, 
either in finite time or in terms of stationary distributions. 
In particular, several existing works have studied 
the approximation of stationary distributions 
\citep{talay1990expansion,MST10,erdogdu2018global}. 
This line of work is most similar to our approach in 
studying the error in the distributional sense. 

\begin{table}[h]
\centering
\renewcommand{\arraystretch}{1.8}
\begin{tabular}{|l|l|l|}
\hline
Paper 
& Regularity Assumptions 
& 
	\bgroup
	\def\arraystretch{1}%
	\begin{tabular}{@{}c@{}}
		Steps $k \geq \Omega(\cdot)$ to reach \\ 
		$\epsilon$ generalization error 
	\end{tabular}
	\egroup
\\ \hline
\cite{raginsky2017nonconvex} 
& Gradient Lipschitz 
& $\epsilon^{-4} (\log (\epsilon^{-1}))^{3}$ 
\\ \hline
\cite{XuCG17} 
& Gradient Lipschitz 
& $\epsilon^{-1} \log(\epsilon^{-1})$ 
\\ \hline
\cite{erdogdu2018global} 
& $ f(\cdot, z) \in C^4_{\pol} $
& $\epsilon^{-2}$ 
\\ \hline 
Present Work 
& Gradient Lipschitz, $ f(\cdot, z) \in C^{6N+2}_{\pol} $ 
& $\epsilon^{-1/N} \log (\epsilon^{-1}) $ \\ \hline
\end{tabular}
\caption{
Comparison of the assumptions and resulting runtime complexities 
to reach $\epsilon$ approximation error by recent works. 
Here $C^\ell_{\pol}$ denotes the space of $C^\ell$ functions 
where all derivatives have polynomial growth, 
and $\Omega(\,\cdot\,)$ notation hides constants 
independent of $\epsilon$. 
We emphasize that all previous works have approximated 
the Langevin algorithm using the Gibbs density, 
while the present work constructs a modified density. 
We provide precise details on assumptions and results 
in \cref{sec:main_results}. 
}
\label{tb:related}
\end{table}

The idea of backward error analysis traces back to 
the study of numerical linear algebra \citep{wilkinson1960error} 
and numerical methods for ordinary differential equations (ODEs) 
\citep{hairer2006geometric}. 
This approach views each update step of a numerical algorithm as 
the exact solution of a modified differential equation. 
For example, if we solve the ODE $y'(t) = y(t)^2$ 
with the Euler update $y_{k+1} = y_k + \eta y_k^2$ 
for some $\eta > 0$, 
then the modified equation can be written as an infinite series 
\citep[Chapter IX, Example 1.1]{hairer2006geometric} 
\begin{equation*}
	\wt{y}' 
	= \wt{y}^2 - \eta \wt{y}^3 
		+ \eta^2 \frac{3}{2} \wt{y}^4 
		- \eta^3 \frac{8}{3} \wt{y}^5 
		+ \eta^4 \frac{31}{6} \wt{y}^6 
		+ \cdots \,, 
\end{equation*}
where given initial condition $\wt{y}(0) = y_k$, 
we will have that $\wt{y}(\eta) = y_{k+1}$. 

However, directly extending this construction to SDEs 
pathwise is not straightforward \citep{shardlow2006modified}. 
Instead, \citet{DebusscheF11} introduced a construction approximating 
the numerical method in distribution. 
The authors derived a partial differential equation (PDE) 
describing the evolution of the distribution for the numerical algorithm, 
and consequently the stationary distribution as well. 
We will provide more details in \cref{sec:sketch}. 
This work has been extended to implicit Langevin algorithms 
\citep{Kopec13,kopec2015}, 
higher order discretizations 
\citep{abdulle2012high,abdulle2014high,laurent2020exotic}, 
and stochastic Hamiltonian systems with symplectic schemes 
\citep{wang2016modified,anton2017error,anton2019weak}. 

We summarize a comparison of related approximation results for SGLD 
in \cref{tb:related}.

\section{Main Results}
\label{sec:main_results}

Throughout the article we denote the Euclidean inner product 
by $\langle x, y \rangle$ for $x, y \in \mathbb{R}^d$, 
and the corresponding norm by $\abs{x} := \langle x, x \rangle^{1/2}$. 
Unless otherwise specified, all expectations $\mathbb{E}[\,\cdot\,]$ 
are with respect to all sources of randomness 
including ${\mb{z}}$. 
We denote the conditional expectation on a random variable 
using the subscript notation, more precisely we use 
$\mathbb{E}_{\mb{z}} [ \,\cdot\, ] := \mathbb{E} [\,\cdot\,| {\mb{z}} ]$. 
For any measure or density $\pi$ on $\mathbb{R}^d$, 
we denote the integral $\int \phi \, d\pi$ as $\pi(\phi)$. 
We also use the subscript notation on distributions $\pi_{\mb{z}}$ 
(and other objects) 
to denote the dependence on the dataset $\mb{z} \in \mathcal{Z}^n$. 

Given a multi-index $\alpha = (\alpha_1, \ldots, \alpha_d) \in\mathbb{N}^d$, 
we define $|\alpha| = \alpha_1 + \cdots + \alpha_d$. 
For any function $\phi \in C^{\infty}(\R^d)$, 
we define the short hand derivative notation
$\de_\alpha \phi(x) 
= (\de_{x_1})^{\alpha_1} \cdots (\de_{x_d})^{\alpha_d} \phi(x)$. 
For all $k, \ell \in \mathbb{N}$, 
we also define the function norm 
$  \| \phi \|_{k, \ell} := 
    \sup_{\substack{
            |\alpha| \leq k}} \,
    \sup_{x \in \mathbb{R}^d}
    |\de_\alpha \phi(x)| (1 + |x|^\ell)^{-1}. 
$
This leads to the natural function space 
    $C^k_\ell(\mathbb{R}^d) := 
    \{ \phi \in C^k(\mathbb{R}^d) : 
        \| \phi \|_{k,\ell} < \infty 
    \}$.

We are now ready to state the main assumptions on 
the loss function $f(x,z)$. 
\begin{assumption}[Regularity]
\label[assumption]{asm:smooth_loss}
There exists positive integers $N,\ell$ 
such that $f(\cdot, z) \in C^{6N+2}_{\ell}(\mathbb{R}^d)$ and 
$\sup_{z \in \mathcal{Z}} \| f(\cdot, z) \|_{6N+2,\ell} < \infty$. 
Furthermore, there exists a constant $ M > 0$ such that 
\begin{equation}
    \abs{ \nabla f(x,z) - \nabla f(y,z)}
    \leq M \abs{x - y}, 
    \quad \forall x,y \in \mathbb{R}^d \,. 
\end{equation}
\end{assumption}

Without loss of generality, 
this assumption implies 
for each $\alpha \in \mathbb{N}^d$ with $|\alpha| \leq 6N+2$, 
there exists a constant $M_\alpha > 0$ such that 
$\sup_{z \in \mathcal{Z}} |\de_\alpha f(x, z)|
\leq M_\alpha (1 + |x|^{\ell})$. 
For each $k \in \mathbb{N}$, we define
$M_k = \sum_{|\alpha| = k} M_\alpha$.

\begin{assumption}[Dissipative]
\label[assumption]{asm:dissipative}
There exist constants $m>0$ and $b\geq 0$ such that 
\begin{equation}
    \langle x, \nabla f(x, z) \rangle
    \geq m \abs{x}^2 - b, 
    \quad \forall x \in \mathbb{R}^d, z \in \mathcal{Z}. 
\end{equation}
\end{assumption}

We remark this is a commonly used sufficient condition 
for exponential convergence of Langevin diffusion 
\citep{raginsky2017nonconvex,bakry2008simple}, 
and it can be replaced by any other sufficient condition 
\citep{villani2009hypocoercivity,bakry2013analysis}. 

We define stochastic gradient Langevin dynamics (SGLD) 
by the following update rule 
\begin{equation}
\label{eq:gld}
    X_{k+1}  = X_{k} - \eta \nabla F_{\zeta_k}(X_k) 
    + \sqrt{\frac{2\eta}{\beta}} \xi_k \,, 
\end{equation}
where $X_0 = x \in \mathbb{R}^d$ is 
a deterministic initial condition, 
$\eta > 0$ is a constant step size (or learning rate), 
$\beta > 0$ is the inverse temperature parameter, 
$\{\zeta_k\}_{k \in \mathbb{N}}$ 
are uniform (minibatch) subsamples of ${\mb{z}}$ 
(with replacement) of size $n_b \leq n$, 
and $\{\xi_k\}_{k \in \mathbb{N}}$ are 
i.i.d. samples from $\mathcal{N}(0, I_d)$. 
Here we let $\{\zeta_k\}_{k \in \mathbb{N}}, 
\{\xi_k\}_{k \in \mathbb{N}}$ be independent 
conditioned on ${\mb{z}}$. 

We are ready to state our first main result.

\begin{theorem}
[Approximation of SGLD]
\label{thm:backward_analysis_main}
Suppose $f(x,z)$ satisfies \cref{asm:smooth_loss,asm:dissipative} 
with order of approximation $N \in \mathbb{N}$. 
Then there exist positive constants $C,\lambda,\ell'$ (depending on $N$), 
such for all step sizes $0 < \eta < \frac{2m}{M^2}$ 
and $\mb{z} \in \mathcal{Z}^n$, 
we can construct a modified stationary measure $\pi^N_{\mb{z}}$, 
with the property that 
for all steps $k \geq 0$, initial condition $X_0 = x$, 
and test function $\phi \in C^{6N + 2}_{\ell}(\mathbb{R}^d)$, 
the following approximation bound on 
the SGLD algorithm $\{X_k\}_{k \geq 0}$ \eqref{eq:gld} holds 
\begin{equation}
	\left| \, \mathbb{E}_{\mb{z}} \, \phi(X_k) 
		- \pi^N_{\mb{z}}(\phi) \, 
	\right| 
    \leq 
    	C 
        \left( 
            e^{ - \lambda k \eta / 2 } + \eta^N
        \right) 
        (1 + |x|)^{\ell'} \, 
        \| \phi - \rho_{\mb{z}} (\phi) \|_{6N+2, \ell} 
        \,.
\end{equation}
In particular, the above result holds for $\phi \in \{ F, F_{\mb{z}} \}$.
\end{theorem}

The full proof can be found in \cref{sec:appendix_weak_backward}. 

Our second main result is on bounding the generalization error 
of the approximate stationary distribution. 
Here we note that beyond SGLD, 
many discretizations of Langevin diffusion admits 
a higher order approximate stationary distribution $\pi^N_{\mb{z}}$ 
via the same weak backward error analysis construction. 
In particular, the implicit Euler method \citep{Kopec13} 
and weak Runge-Kutta methods \citep{laurent2020exotic} are well studied. 
Therefore, we state the result for all such methods. 

\begin{theorem}
[Generalization Bound of $\pi^N_{\mb{z}}$]
\label{thm:pi_gen_bound}
Suppose $\{X_k\}_{k\geq 0}$ is any discretization 
of Langevin diffusion \eqref{eq:langevin_diffusion} 
with an approximate stationary distribution $\pi^N_{\mb{z}}$ 
as in \cref{thm:backward_analysis_main}. 
Then there exists a constant $C>0$ (depending on $N$), 
such that for all choices of $k,n$ and $\eta \in (0,1)$ 
the following expected generalization bound holds 
\begin{equation}
	\left| \, 
		\mathbb{E} 
		\left[ \pi^N_{\mb{z}}( F ) 
			- \pi^N_{\mb{z}}( F_{\mb{z}} ) 
		\right] \, 
	\right| 
	\leq 
		\frac{C}{n (1-\eta) } \,, 
\end{equation}
where the expectation is with respect to 
$\mb{z} \sim \mathcal{D}^n$. 

\end{theorem}

The full proof can be found in \cref{sec:appendix_uniform_stability}. 
As a corollary of the two results, 
we have a runtime complexity as follows. 

\begin{corollary}
[Runtime Complexity]
\label{cor:runtime}
Suppose $\{X_k\}_{k\geq 0}$ is any discretization 
of Langevin diffusion \eqref{eq:langevin_diffusion} 
admiting an approximate stationary distribution $\pi^N_{\mb{z}}$ 
of the type in \cref{thm:backward_analysis_main}. 
Then there exists a constant $C>0$ (depending on $N$), 
such that for all 
\begin{equation}
    \epsilon > 0 \,, \quad 
	0 < \eta < \min\left\{\frac{2m}{M^2} \,, C \epsilon^{1/N} \right\} \,, \quad 
	n \geq \frac{C}{\epsilon (1-\eta)} \,, \quad 
	k \geq \frac{C}{\epsilon^{1/N}} \log \frac{1}{\epsilon} \,, 
\end{equation}
we achieve the following expected generalization bound 
\begin{equation}
	\left| \, 
		\mathbb{E} 
		\left[ F( X_k )
			- F_{\mb{z}}( X_k )
		\right] \, 
	\right| 
	\leq \epsilon \,. 
\end{equation}
\end{corollary}

The proof can be found in \cref{sec:appendix_proof_corollary}. 
Once again, we remark this implies a runtime complexity of 
$k \geq \Omega( \epsilon^{-1/N} \log (\epsilon^{-1}))$ 
despite $\{X_k\}_{k\geq 0}$ being a first order discretization 
of Langevin diffusion.

\section{Proof Overview}
\label{sec:sketch}

In this section, 
we provide a sketch of the main results. 
Here we omit most of the technical details, 
with the goal of explaining the core ideas clearly and concisely.

\subsection{Construction of $\pi^N_{\mb{z}}$} 
\label{subsec:sketch_construction}

Before we go to the distributional setting, 
it is instructive to build intuitions from ODEs. 
In particular, we return to 
\citep[Chapter IX, Example 1.1]{hairer2006geometric}, 
where we consider solving $y'(t) = y(t)^2$ 
with the Euler update $y_{k+1} = y_k + \eta y_k^2$ 
for some $\eta > 0$. 
Following this example, we hypothesize the existence of 
a modified ODE as the formal series  
\begin{equation}
\label{eq:formal_series}
	\wt{y}'(t) 
	= \wt{y}(t)^2 
		+ \sum_{\ell = 1}^\infty c_\ell( \wt{y}(t) ) \, \eta^\ell \,, 
\end{equation}
such that $\wt{y}(0) = y_k$ and $\wt{y}(\eta) = y_{k+1}$. 
If $\wt{y}$ has a Taylor expansion satisfying 
the constraint $\wt{y}(\eta) = \wt{y}(0) + \eta \wt{y}(0)^2$, 
we can solve for all the coefficients $c_\ell(\wt{y}(t))$ 
by matching the terms with the same polynomial order of $\eta^\ell$. 

Observe that if this formal series converges, 
we have an exact reconstruction of the Euler method 
$\{y_k\}_{k\geq 1}$ via a modified ODE. 
However, since this is often not easy (or even possible), 
we alternatively consider a truncation of this series, 
leading to a high order approximation. 
This is exactly the approach known as backward error analysis 
\citep{hairer2006geometric}. 
We plot this particular example in \cref{fg:backward_error}.

\begin{figure}[h]
\centering     %
(a)
\includegraphics[width=0.46\textwidth]{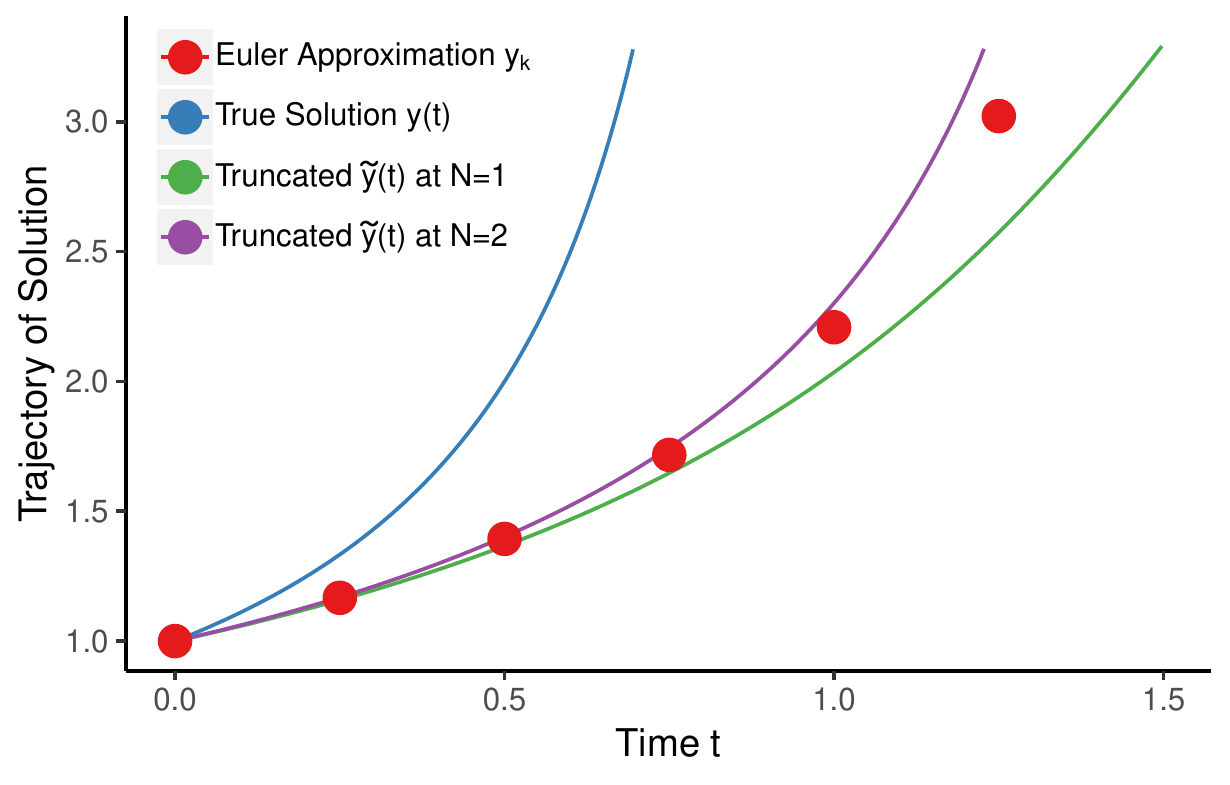}
(b)\includegraphics[width=0.46\textwidth]{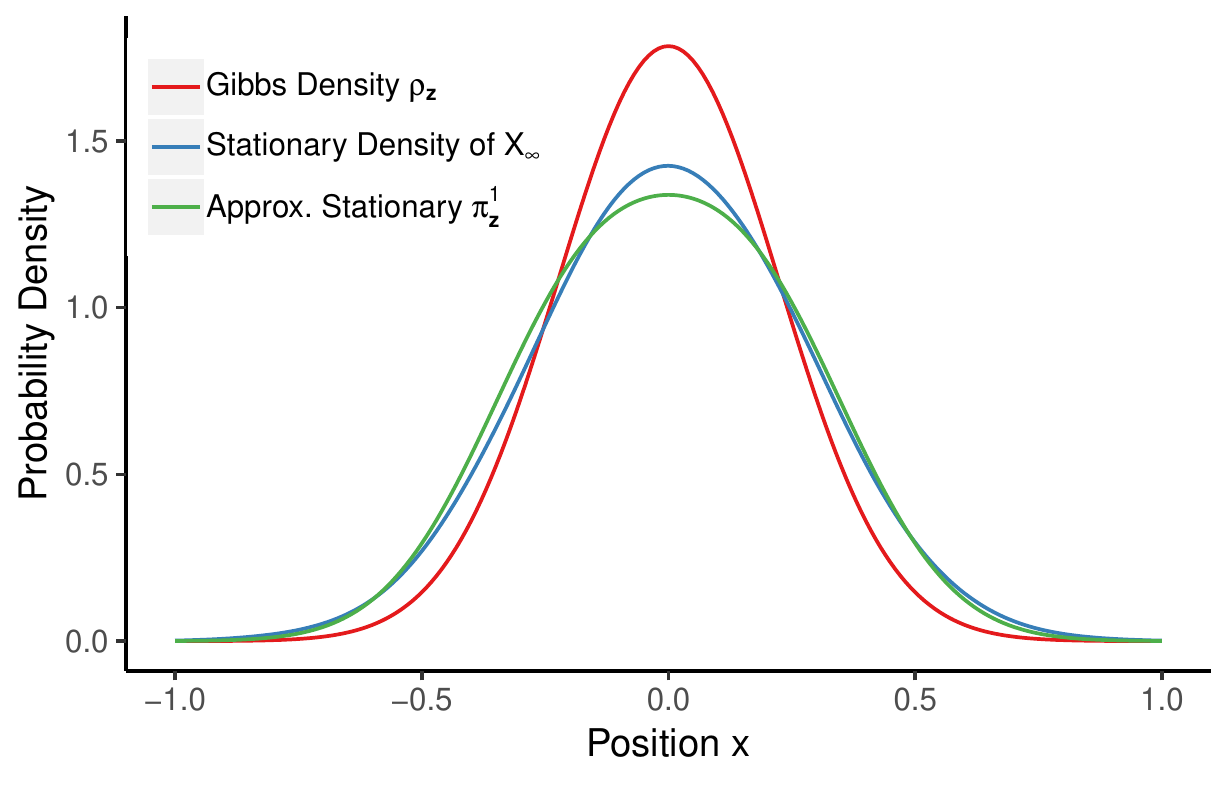}
\caption{
(a): A plot of \emph{backward error analysis} 
for the example ODE \eqref{eq:formal_series}, 
where truncations of $\wt{y}(t)$ 
are taken to better approximate the Euler discretization. 
(b): A toy example to show how the distribution 
$\pi^N_\mathbf{z}$ constructed using 
\emph{weak backward error analysis} 
compares against the true stationary distribution 
of SGLD $\mathcal{L}(X_{\infty})$. 
In this case, we choose the loss to be 
$f(x,z) = \frac{1}{2} x^2$
so we can obtain explicit formulas for the densities. 
We observe a large discrepancy between the Gibbs measure 
$\rho_\mathbf{z}$ and $\mathcal{L}(X_{\infty})$ (for SGLD), 
where as the first order approximation 
$\pi^1_\mathbf{z}$ is a much better approximation 
to $\mathcal{L}(X_{\infty})$. 
See \cref{sec:calc} for details of the calculations. 
}
\label{fg:backward_error}
\end{figure}

Once again, we emphasize that extending this method directly through pathwise 
approximation of the diffusion will be difficult \citep{shardlow2006modified}. 
However, if we forego the pathwise information of 
the Langevin diffusion $\{X(t)\}_{t\geq 0}$ \eqref{eq:langevin_diffusion}, 
and only consider its distribution via $\mathbb{E} \phi(X(t))$ 
for some test function $\phi$, 
we can study the evolution via the Kolmogorov backward equation 
\begin{equation}
\label{eq:kolmogorov}
	\de_t u(t, x) = L_{\mb{z}} \, u(t,x) \,, \quad 
	u(0, x) = \phi(x) \,, 
\end{equation}
where $L_{\mb{z}} \, u := \langle -\nabla F_{\mb{z}}, \nabla u \rangle 
+ \frac{1}{\beta} \Delta u$ is the It\^{o} generator, 
and $u(t,x) = \mathbb{E}[ \phi(X(t)) | X_0 = x ]$ 
is the well known stochastic representation 
\citep[Theorem 3.43]{pardoux2014stochastic}. 

Since the evolution of Langevin diffusion's distribution 
can be interpreted as Wasserstein gradient flow 
in the space of probability distributions 
\citep{jordan1998variational}, 
it is natural to consider extending 
backward error analysis in the space of distributions. 
Indeed, this is the approach taken by \citet{DebusscheF11} 
using a modified PDE instead. 

Similar to \cref{eq:formal_series}, 
we instead write down a modified Kolmogorov equation 
as a formal series 
\begin{equation}
\label{eq:kolmogorov_formal_series}
	\de_t v 
	= 
		L_{\mb{z}} v 
		+ \sum_{\ell = 1}^\infty 
			\eta^\ell \, L_{\ell,\mb{z}} \, v 
			\,, \quad 
	v(0,x) = \phi(x) \,, 
\end{equation}
where $\{L_{\ell,\mb{z}}\}_{\ell \geq 1}$ are differential operators 
playing the same role as the coefficients $c_\ell$ in \cref{eq:formal_series}. 
These operators are solved such that an It\^{o} type Taylor expansion 
matches the distribution of SGLD in the sense 
$v(\eta,x) = \mathbb{E} [ \phi(X_1) | X_0 = x ]$ formally, 
where $X_1$ is the one step update of \cref{eq:gld} 
and $\phi$ is the same test function as in \cref{eq:kolmogorov_formal_series}.

Once again, this formal series is an exact reconstruction 
of SGLD's evolution if the series is convergent. 
Similar to the modified equation for ODEs, 
we avoid justifying the convergence by 
taking a truncation at order $N$, 
matching the regularity condition of \cref{asm:smooth_loss}, 
which is required for the Taylor type expansion in $\eta$. 
More precisely, we construct the truncation using 
functions $v_m(t,x)$ defined recursively as 
the solutions of the following equations 
\begin{equation}
\label{eq:kolmogorov_truncation}
	\de_t v_m - L_{\mb{z}} v_m 
	= \sum_{\ell=1}^m L_{\ell,\mb{z}} v_{m-\ell} \,, \quad 
	v_m(0,x) = \begin{cases}
		\phi(x) \,, & m = 0 \,, \\ 
		0 \,, & m > 0 \,. 
	\end{cases}
\end{equation}
Using this construction, 
we show that the truncation of $v$ defined by 
$v^{(N)} = \sum_{\ell=0}^N \eta^\ell v_\ell$ 
approximates $\mathbb{E}[ \phi(X_1) | X_0 = x ]$ 
up to an error of $O(\eta^{N+1})$ 
. 

Next, we establish the convergence of the truncation terms $v_m$ 
via a standard spectral gap argument 
for non-homogenous parabolic PDEs \citep{pardoux2003poisson}. 
More precisely, we can write 
\begin{equation}
\label{eq:truncation_convergence}
	\left| v_m(t,x) - \lim_{s\to\infty} v_m(s, x) \right| 
	\leq 
		C \, e^{ -\lambda t / 2 } \,, 
\end{equation}
where $\lambda > 0$ is the Poincar\'{e} constant 
arising as a consequence of \cref{asm:dissipative} 
\citep{bakry2008simple,raginsky2017nonconvex}. 
By standard adjoint equation arguments
\citep[Theorem 2.2]{pavliotis2014stochastic}, 
we can recover the stationary measure from 
\eqref{eq:kolmogorov_truncation} 
by setting $\de_t v_m = 0$ and 
replacing all differential operators with their 
corresponding adjoint. 
More precisely, we have 
$\lim_{t\to\infty} v_m(t,x) = \rho_{\mb{z}}( \phi \, \mu_{m,\mb{z}} )$, 
where the Radon-Nikodym derivative $\mu_{m,\mb{z}}(x)$ 
is the unique solution of the Poisson equation 
 
\begin{equation}
\label{eq:poisson_truncation}
	L_{\mb{z}} \mu_{m,\mb{z}} 
	= - \sum_{\ell=1}^m L_{\ell,\mb{z}}^* \mu_{m-\ell,\mb{z}} \,, 
\end{equation}
where we use $L_{\ell,\mb{z}}^*$ to denote 
the adjoint operator of $L_{\ell,\mb{z}}$ in $L^2(\rho_{\mb{z}})$, 
and $\mu_{0,\mb{z}} = 1$. 
Observe the above equation is exactly 
stationary adjoint of \eqref{eq:kolmogorov_truncation}, 
where we note $L_{\mb{z}}^* = L_{\mb{z}}$ is self-adjoint. 

Using linearity of expectation over $\rho_{\mb{z}}$, 
we can write 
\begin{equation}
	\lim_{t\to\infty} v^{(N)}(t,x) 
	= 
		\rho_{\mb{z}}\left( 
			\phi \, 
			\sum_{\ell=0}^N \eta^\ell \, \mu_{\ell,\mb{z}}
		\right) \,, 
\end{equation}
hence we recover the approximate stationary measure as 
$\pi^N_{\mb{z}} = \rho_{\mb{z}} \sum_{\ell=0}^N \eta^\ell \mu_{\ell,\mb{z}}$. 

Finally, to extend the one step error bound of $O(\eta^{N+1})$ 
to arbitrary steps, we use a telescoping argument to write 
\begin{equation}
\label{eq:error_telescope}
\begin{aligned}
	\left| \, \mathbb{E}_{\mb{z}} \, \phi(X_k) 
	    - v^{(N)}(t,x) \, 
	\right| 
	&\leq 
		\mathbb{E} \, \sum_{\ell=1}^k \, 
		\left| \mathbb{E}[ v^{(N)}((\ell-1)\eta, X_{k-\ell+1}) - 
			v^{(N)}(\ell\eta, X_{k-\ell}) 
			| \mathcal{F}_{k-\ell} ] 
		\right| \\ 
	&\leq 
		C \, \eta^{N+1} \sum_{\ell=1}^k e^{-\lambda \ell \eta / 2 } \,, 
\end{aligned}
\end{equation}
where $\{\mathcal{F}_{k}\}_{k \geq 0}$ 
is the filtration generated by the algorithm $\{X_k \}_{k\geq 0}$. 
Here, we used the one step error bound 
and exponential convergence \eqref{eq:truncation_convergence}. 
Note that we can bound the geometric series by 
$\frac{1}{1 - e^{-\lambda \eta / 2}} = O( \eta^{-1} )$, 
hence leading to the error order of $O(\eta^N)$ 
in \cref{thm:backward_analysis_main}. 
We plot an example construction of $\pi^1_{\mb{z}}$, 
where we can compute the density explicitly in \cref{fg:backward_error}.

\subsection{Generalization Bound}

We start by recalling the definition and generalization property of 
the uniform stability \citep{Bousquet:2002} for distributions. 

\begin{definition}[Uniform Stability]
\label{def:uniform_stability}
A collection of distributions $\{\pi_\mathbf{z}\}$ 
on $\mathbb{R}^d$ indexed by $\mathbf{z} \in \mathcal{Z}^n$ 
is said to be \textbf{$\epsilon$-uniformly stable} if 
for all $\mathbf{z}, \overline{\mathbf{z}} \in \mathcal{Z}^n$
with only one differing coordinate 
\begin{equation}
	\sup_{z \in \mathcal{Z}} 
	\left| 
		\pi_{\mb{z}}( f(\cdot, z) ) - 
		\pi_{\overline{\mb{z}}}( f(\cdot, z) )
	\right| 
	\leq \epsilon.
\end{equation}
\end{definition}

\begin{proposition}
[Generalization]
\label{prop:signed_stability}
Suppose the collection of distributions $\{\pi_\mathbf{z}\}$
is $\epsilon$-uniformly stable, 
and that for all $(\mathbf{z}, z) \in \mathcal{Z}^{n+1}$, 
we also have $f(\cdot, z) \in L^1( \pi_{\mb{z}} )$. 
Then the expected generalization error of $\{\pi_\mathbf{z}\}$ 
is bounded by $\epsilon$, or more precisely 
\begin{equation}
	\left| \, 
		\mathbb{E} \left[
		\pi_\mathbf{z}( F_{\mb{z}} )
		- 
		\pi_\mathbf{z}( F ) 
		\right] \, 
	\right| 
	\leq 
		\epsilon \,.
\end{equation}
\end{proposition}

Additional details can be found in \cref{sec:appendix_signed_measure}. 
With this approach in mind, we introduce several additional notations. 
Without loss of generality, 
we let $\mb{z} = \{ z_1,\cdots,z_i,\cdots,z_n \}$ 
and $\ol{\mb{z}} = \{ z_1,\cdots,\ol{z}_i,\cdots,z_n \}$ 
such that they only differ in the $i^\text{th}$ coordinate. 
We also define $\mb{z}^{(i)} := \mb{z} \cap \ol{\mb{z}}$ 
and extend previous notations 
\begin{equation}
	F_{\mb{z}^{(i)}} 
	:= 
		\frac{1}{n} \sum_{j=1,j\neq i}^n f(x,z_i) \,, \quad 
	\rho_{\mb{z}^{(i)}} 
	:= 
		\frac{1}{Z^{(i)}} \, \exp ( -\beta F_{\mb{z}^{(i)}}) \,, \quad 
	q_{z_i} 
	:= 
		\frac{ Z^{(i)} }{ Z } \, \exp\left( -\frac{\beta}{n} f(x,z_i) \right) \,, 
\end{equation}
such that we can write $F_{\mb{z}^{(i)}} + f(\cdot,z_i) = F_{\mb{z}}$, 
$\rho_{\mb{z}^{(i)}} q_{z_i} = \rho_{\mb{z}}$, 
and vice versa for the $\ol{z}_i$ and $\ol{\mb{z}}$ terms. 
We also define the norm 
$\| \phi \|_{L^2(\rho_{\mb{z}^{(i)}})} 
:= \left[ \rho_{\mb{z}^{(i)}}( \phi^2 ) \right]^{1/2}$. 

We start by putting 
the two integrals with respect to 
$\pi^N_{\mb{z}}, \pi^N_{\ol{\mb{z}}}$ 
under one integral with respect to $\rho_{\mb{z}^{(i)}}$, 
and use triangle and Cauchy-Schwarz inequalities to get 
\begin{equation}
\begin{aligned}
	\left| 
		\pi_{\mb{z}}( f ) - \pi_{\ol{\mb{z}}}( f ) 
	\right| 
	&= 
		\left| 
		\rho_{\mb{z}^{(i)}} \left( 
			f \sum_{\ell=0}^N \eta^\ell \,  
			( \mu_{\ell,\mb{z}} q_{z_i} 
			- \mu_{\ell,\ol{\mb{z}}} q_{\ol{z}_i} ) 
		\right) 
		\right| \\ 
	&\leq 
		\sum_{\ell=0}^N \eta^\ell \, 
		\| f \|_{L^2(\rho_{\mb{z}^{(i)}})} 
		\, 
		\| \mu_{\ell,\mb{z}} q_{z_i} 
			- \mu_{\ell,\ol{\mb{z}}} q_{\ol{z}_i} 
		\|_{L^2(\rho_{\mb{z}^{(i)}})} \,. 
\end{aligned}
\end{equation}

This implies that it is sufficient to bound the norm 
$\| \mu_{\ell,\mb{z}} q_{z_i} - \mu_{\ell,\ol{\mb{z}}} q_{\ol{z}_i} 
\|_{L^2(\rho_{\mb{z}^{(i)}})}$ 
to achieve uniform stability (\cref{lm:bound_l2_norm}). 
At the same time, we recall from \eqref{eq:poisson_truncation} 
that $\mu_{m,\mb{z}}, \mu_{m,\ol{\mb{z}}}$ 
solve very similar Poisson equations 
\begin{equation}
\begin{aligned}
	L_{\mb{z}} \mu_{m,\mb{z}} 
	&= 
		- \sum_{\ell=1}^m L_{\ell,\mb{z}}^* \, 
		\mu_{m-\ell,\mb{z}} 
	=: 
		G_{m, \mb{z}}
		\,, \\ 
	L_{\ol{\mb{z}}} \mu_{m,\ol{\mb{z}}} 
	&= 
		- \sum_{\ell=1}^m L_{\ell,\ol{\mb{z}}}^* \, 
		\mu_{m-\ell,\ol{\mb{z}}} 
	=: 
		G_{m, \ol{\mb{z}}}
		\,, 
\end{aligned}
\end{equation}
where we observe the left hand side operator only differs by 
$( L_{\mb{z}} - L_{\ol{\mb{z}}} ) \mu 
= \frac{1}{n} \langle \nabla f(x,\ol{z}_i) - \nabla f(x,z_i), 
\nabla \mu \rangle $, 
which is of order $O(\frac{1}{n})$. 
Furthermore, the right hand side only depends on 
$\mu_{\ell, \mb{z}}, \mu_{\ell,\ol{\mb{z}}}$ for $\ell < m$. 
This suggests the following induction structure  
\begin{equation}
\begin{aligned}
	& 
		\mu_{\ell,\mb{z}} q_{z_i} - \mu_{\ell,\ol{\mb{z}}} q_{\ol{z}_i} 
		= O\left( \frac{1}{n} \right) 
		\text{ for all } {{0 \leq \ell \leq m-1}} 
		\\ 
	\implies& 
		G_{m, \mb{z}} q_{z_i} - G_{m, \ol{\mb{z}}} q_{\ol{z}_i} 
		= O\left( \frac{1}{n} \right) 
		\\ 
	\implies& 
		\mu_{\ell,\mb{z}} q_{z_i} - \mu_{\ell,\ol{\mb{z}}} q_{\ol{z}_i} 
		= O\left( \frac{1}{n} \right) 
		\text{ for all } {{0 \leq \ell \leq m}} \,, 
\end{aligned}
\end{equation}
where we observe the induction step incremented 
the set of $\ell \in \{0,1,\cdots,m-1\}$ to $\{0,1,\cdots,m\}$. 

To complete the proof of uniform stability, 
we need to make the above sketch precise. 
This requires the control of norms for higher order derivatives
of $\mu$ and $G$ terms, 
which is detailed in \cref{lm:energy_est2}.

\section{Discussion}
\label{sec:discussion}

\textbf{On the Poincar\'{e} Constant}. 
Without a careful analysis, a non-convex potential $f(\cdot, z)$ generally 
lead to a Poincar\'{e} constant with exponentially poor dependence on 
the inverse temperature $\beta$ and dimension $d$ \citep{raginsky2017nonconvex}. 
However, there are many useful applications with universal Poincar\'{e} constants 
(independent of $\beta,d$). 
Most famously, when $F$ is strongly convex, 
we can use the Bakry-\'{E}mery curvature condition 
to achieve an universal constant \citep{bakry2013analysis}. 
\citet{cattiaux2021functional} and references within have studied 
many perturbations of convex potential $f$ 
with applications to Bayesian inference. 
\citet{menz2014poincare,li2020riemannian} 
have also extended universal Poincar\'{e} constant to non-convex $f$, 
where all critical points are either strict saddle 
or the unique secondary order stationary point. 
In particular, this class now contains the Burer--Monteiro relaxation 
of semidefinite programs \citep{burer2003nonlinear,boumal2016non}.

\vspace{0.2cm}
\noindent
\textbf{On the Smoothness Conditions}. 
In \cref{asm:smooth_loss}, we assumed $f(\cdot,z) \in C^{6N+2}$. 
Indeed, without additional smoothness, 
there is a lower bound on the runtime complexity 
\citep{cao2020complexity}, 
and therefore higher order analysis is not appropriate in this setting. 
In fact, higher order discretizations in general require higher order smoothness 
\citep{hairer2006geometric}. 
Therefore, any application that calls for a higher order discretization 
can be studied using the approximation method of this work. 
Furthermore, we remark that one can always consider smoothing 
the Gibbs distribution via convolution with a Gaussian 
\citep{chaudhari2019entropy,block2020fast}, 
which leads to an infinitely smooth potential.

\vspace{0.2cm}
\noindent
\textbf{On Further Extensions}. 
While this article is focused on the analysis of generalization error, 
the framework can be extended to other analyses of interest. 
In general, weak backward error analysis saves the approximation error 
between the discrete time algorithm to the diffusion process. 
Therefore, any property of the algorithm can be studied via 
the approximate stationary distribution $\pi^N_{\mb{z}}$. 
For example, the expected suboptimality of Langevin discretizations 
is often analyzed via the Gibbs distribution 
\citep{raginsky2017nonconvex,erdogdu2018global,li2020riemannian}, 
which implies an opportunity to analyze 
the suboptimality of $\pi^N_{\mb{z}}$ instead.

\section{Weak Backward Error Analysis: Proof of \cref{thm:backward_analysis_main}}
\label{sec:appendix_weak_backward}

In this section, we will complete the proof for 
\cref{thm:backward_analysis_main}. 
Here, we adopt the notation 
$C^\infty_{\pol}(\mathbb{R}^d) := 
\cap_{m \geq 0} \cup_{\ell \geq 0} C^m_\ell(\mathbb{R}^d)$ 
for all smooth functions with polynomial growth. 
We remark while all the results in this sections 
are stated for $C^\infty_{\pol}(\mathbb{R}^d)$ functions, 
we only ever differentiate $6N+2$ times, 
therefore it does not contradict \cref{asm:smooth_loss}. 

We start by stating a few technical estimates required for the main result. 
The first result states that the moments of all order of 
the continuous process $(X(t))_{t \geq 0}$ are all uniformly bounded in time.
This result is a minor modification 
to \citet[Proposition 2.2]{Kopec13}. 
The proof is added in \cref{subsec:ME_continuous} for completeness and simply consists on adding the dependence on $\beta$.

\begin{proposition}
[Moment Estimates on the Continuous Process]
\label{prop:ME_continuous}
Let $x_0 \in  \R^d$ and $(X(t))_{t \geq 0}$ satisfying \cref{eq:langevin_diffusion}.
Under Assumption \cref{asm:dissipative}, for each $p \geq 1$ and $0 < \gamma < 2 m$, there exists a positive constant $C_p$ such that
\begin{equation}
\label{eq:sde_moment_est}
    \E\left( \left| X(t) \right|^{2p} \right) 
    \leq 
    C_p \left( |x_0|^{2p} \exp(-\gamma t) +1 \right), \quad \forall t > 0 \,.
\end{equation}
In particular, we have the recursive formula for $C_p$ 
\[ C_p = \max\left( \left( 2pb + \frac{2p}{\beta}(d + 2p - 2) \right) C_{p-1} + 1, 1 \right)
    \max\left( \frac{1}{2pm - \gamma_1}, 1 \right),
\]
where we observe $C_p$ is on the order of 
$\mathcal{O}((d/\beta)^p)$. 
\end{proposition}

Before the next result, 
we will state a technical lemma from \citet[Lemma 3.1]{raginsky2017nonconvex}. 
\begin{lemma}
[Quadratic Bounds on $f(\cdot, z)$]
\label{lm:quad_bd}
Under \cref{asm:smooth_loss,asm:dissipative}, 
for all $x\in \mathbb{R}^d$ and 
$z \in \mathcal{Z}$, we have 
\begin{equation*}
    | \nabla f(x,z) | \leq M |x| + M_1,
\end{equation*}
and
\begin{equation*}
    \frac{m}{3}|x|^2 - \frac{b}{2} \log 3
    \leq f(x, z) \leq \frac{M}{2} |x|^2 + M_1 |x| + M_0,
\end{equation*}
\end{lemma}

We will need the similar estimates for the solution of the discrete equation as the ones obtained in \cref{eq:sde_moment_est}. 
For this proof, we followed similar arguments from 
\citet[Lemma 3.2]{raginsky2017nonconvex} and \cite[Proposition 2.5]{Kopec13}.
\begin{proposition}
[Moment Estimates on the Discrete Process]
\label{prop:me_discrete}
Let $x_0 \in  \R^d$ and $(X_k)_{k \in \mathbb{N}}$ 
be the discrete Langevin algorithm satisfying \labelcref{eq:gld}. 
Under \cref{asm:smooth_loss,asm:dissipative},
if we set $0 < \eta < \frac{2 m}{M^2}$,
then for each $p \geq 1$, 
there exists a positive constant $C_p$ 
uniform in $k$, such that
\[
    \E\left( \left| X_k \right|^{2p} \right) 
    \leq 
    C_p \left( |x_0|^{2p} + 1 \right), 
    \quad \forall k > 0.
\]
\end{proposition}

The proof can be found in \cref{subsec:me_discrete}.

The moments estimates, together with the Markov property, 
are used to extend the local analysis to the global analysis. 
The next result corresponds to a small modification of 
\citet[Theorem 3.2]{DebusscheF11}, 
where we need to apply a moment estimate using 
Proposition \ref{prop:me_discrete}. 

Before we state the result, we will define the semi-norm 
\[ |\phi|_{l,k} := 
    \sup_{\substack{\alpha \in \mathbb{N}^d \\
            0 < |\alpha| \leq l}} \,
    \sup_{x \in \mathbb{R}^d}
    |\de_\alpha \phi(x)| (1 + |x|^k)^{-1}.
\]

This next result is key to developing the modified PDE 
\eqref{eq:kolmogorov_formal_series}, 
as we prove an asymptotic expansion of the discrete time 
SGLD process \eqref{eq:gld}. 

\begin{proposition}
[Asymptotic Expansion for SGLD]
\label{prop:AsExp}
Let $\phi \in C^\infty_{\pol}(\mathbb{R}^d)$, 
then for all $N \in \mathbb{N}$ there exists an integer $l_{2N + 2}$
such that $\phi \in C^{2N+2}_{l_{2N + 2}}(\mathbb{R}^d)$.
Let $\{X_k\}_{k\in \mathbb{N}}$ 
be the discrete time SGLD algorithm from \labelcref{eq:gld}.
Then for every integer $j \geq 0$, 
there exist differential operators $A_j$ of order $2j$ 
with coefficients from $C^\infty_{\pol}(\mathbb{R}^d)$, 
such that for all integer $N \geq 1$
there exist constants $C_N$ and integer $\alpha$ depending on 
N and the polynomial growth rate of 
$F_\mathbf{z}(\cdot)$ and its derivatives,
such that for all $\eta < \frac{2m}{M^2}$ 
we have 
\[ \left| \mathbb{E} \phi(X_1) - \sum_{j=0}^N \eta^j A_j \phi(x) 
    \right|
    \leq C_N \eta^{N+1} (1 + |x|^\alpha) |\phi|_{2N+2, l_{2N+2}},
    \quad \forall x \in \mathbb{R}^d,
\]
where in particular we have $A_0 = I$ and $A_1 = L$.
\end{proposition}

The proof can be found in \cref{subsec:AsExp}.

The rest of the results will follow directly from 
the steps of \citet{Kopec13}, 
which is an extension of \citet{DebusscheF11} 
from a torus to $\mathbb{R}^d$. 
We will sketch the steps here, 
and check the conditions that leads to the results 
of \citet{Kopec13}.

Specifically, we will construct operators $\{L_{j}\}_{j \in \mathbb{N}}$ 
such that the operator 
\[
\mathcal{L} := L + \eta L_1 + \cdots + 
    \eta^j L_j + \cdots \,,
\]
satisfies the following identity in a formal sense 
\begin{equation}
\label{eq:appendix_formal_equiv}
  \exp(\eta \mathcal{L}) 
    = \sum_{\ell=0}^\infty \frac{\eta^\ell}{\ell!} 
      (L + \eta L_1 + \eta^2 + L_2 + \cdots)^\ell
    = \sum_{j=0}^\infty \eta^j A_j \, .
\end{equation}

Observe that it is sufficient to the coefficients 
to each term of the power series, i.e. $\eta^j$. 
Using a formal series inverse approach from 
\citet{hairer2006geometric}, 
we obtain the following formal equivalence 
\begin{equation}
\label{eq:appendix_L_j}
    L_j = A_{j+1} + \sum_{\ell = 1}^j 
    \frac{B_\ell}{ \ell ! }
    \sum_{n_1 + \cdots + n_{\ell + 1}
    = n - \ell}
    L_{n_1} \cdots L_{n_\ell} 
    A_{n_{\ell + 1} + 1} \,,
\end{equation}
where $\{B_\ell\}$ are the Bernoulli numbers. 

Using this construction, we can define the following modified PDE 
as before in \eqref{eq:kolmogorov_formal_series}
\begin{equation}
  \frac{\de v}{\de t} = (L_{\mathbf{z}} + \eta L_{1, \mathbf{z}}
    + \eta^2 L_{2, \mathbf{z}} + \cdots ) v, 
  \quad v(0,x) = \phi(x), 
\end{equation}
such that we can write the formal solution as 
\[ v(\eta, x) = e^{\eta \mathcal{L}} \phi(x). 
\]

Using the formal equivalence for the operators in 
Equation \eqref{eq:appendix_formal_equiv}, 
we can obtain the formal equivalence of the solutions as well, i.e. 
\[ v(\eta, x) = e^{\eta \mathcal{L}} \phi(x)
  = \sum_{j=0}^\infty \eta^j A_j \phi(x)
  = \mathbb{E}_\mathbf{z}[ \phi(X_1) | X_0 = x ].
\]

Finally, we can study the stationary distribution 
$\pi_\mathbf{z} = \mu_\mathbf{z} \rho_\mathbf{z}$ by writing down 
the adjoint equation that $\mu_\mathbf{z}$ must satisfy 
\[ \mathcal{L}^* \mu_\mathbf{z} = 0,
\]
where $\mathcal{L}^*$ is the adjoint operator 
with respect to $\rho_\mathbf{z}$. 

While the above construction is formal, 
\citet{DebusscheF11, Kopec13} made 
these statements precise for truncated series 
by proving error bounds with the desired order in $\eta$. 
In particular, the same authors constructed the truncation 
by the following decomposition of $\pi^N_\mathbf{z}$
\begin{equation}
\label{eq:appendix_mu_decomp}
  \pi^N_\mathbf{z} 
  = \rho_\mathbf{z} \mu^N_\mathbf{z}
  := \rho_\mathbf{z} (1 + \eta \mu_1 + \cdots 
    + \eta^N \mu_N), 
\end{equation}
and if each $\mu_k$ satisfies the Poisson equation 
\begin{equation}
\label{eq:appendix_poisson_equation}
  L \mu_k = - \sum_{\ell = 1}^k 
    L_{\ell}^* \mu_{k-\ell} \,, 
\end{equation}
then we can show that 
\[ (L^* + \eta L_1^* + \cdots + \eta^N L_N^*) \mu^N = \mathcal{O}(\eta^{N+1}),
\]
which is sufficient close to the desired truncation to preserve 
the order of approximation error. 

To summarize, we will restate and prove the main result
of \cref{thm:backward_analysis_main}. 
\begin{theorem}
[Approximation of SGLD]
Suppose $f(x,z)$ satisfies \cref{asm:smooth_loss,asm:dissipative} 
with order of approximation $N \in \mathbb{N}$. 
Then there exist positive constants $C,\lambda,\ell'$ (depending on $N$), 
such for all step size $0 < \eta \leq \frac{2m}{M^2}$ 
and $\mb{z} \in \mathcal{Z}^n$, 
we can construct a modified stationary measure $\pi^N_{\mb{z}}$, 
with the property that 
for all steps $k \geq 0$, initial condition $X_0 = x$, 
and test function $\phi \in C^{6N + 2}_{\ell}(\mathbb{R}^d)$, 
the following approximation bound on 
the SGLD algorithm $\{X_k\}_{k \geq 0}$ \eqref{eq:gld} holds 
\begin{equation}
    \left| \, \mathbb{E}_{\mb{z}} \, \phi(X_k) 
        - \pi^N_{\mb{z}}(\phi) \, 
    \right| 
    \leq 
        C 
        \left( 
            e^{ - \lambda k \eta / 2 } + \eta^N
        \right) 
        (1 + |x|)^{\ell'} \, 
        \| \phi - \rho_{\mb{z}} (\phi) \|_{6N+2, \ell} 
        \,.
\end{equation}
In particular, the above result holds for $\phi \in \{ F, F_{\mb{z}} \}$.
\end{theorem}

\begin{proof}

It is sufficient to check that we have  
all the technical conditions to use the main result of 
\citet[Proposition 5.4]{Kopec13}. 

We start by checking that the modified flow result 
of \citet[Theorem 4.1]{Kopec13} only requires 
the asymptotic expansion result of Proposition \ref{prop:AsExp}. 
In particular, the only difference between our results are the construction of the operators $\{A_j\}$.
Since all the definitions of $\{L_j\}$ are in terms of $\{A_j\}$, 
the result follows by the exact same proof. 

Next we observe that \citet[Proposition 5.1 and 5.3]{Kopec13} 
does hinge on any earlier technical estimates. 

And finally the construction of the invariant measure in \citet[Proposition 5.4]{Kopec13} 
relies on the above intermediate results, 
and additionally requires bounding the discrete moments $\mathbb{E}|X_k|^{2p}$ given by Proposition \ref{prop:me_discrete}. 
Specifically, our discrete moment estimate replaces \citet[Proposition 2.5]{Kopec13}. 
Therefore, the result follows from the same proof. 

\end{proof}

\subsection{Proof of Proposition \ref{prop:ME_continuous} }
\label{subsec:ME_continuous}

\begin{proof}

Let $N \in \mathbb{N}$ be a positive integer, 
and we define the stopping time 
\[ \tau_{{N}} := \inf \{t \geq 0 : {\abs{X(t)}} \geq N \}.
\]

We will prove by induction both the main statement 
\cref{eq:sde_moment_est} and the following:
for all $p \in \mathbb{N}$ positive, 
$0 < \gamma < 2m$, there exists a constant $C_p > 0$  
such for all $t \geq 0$ we have
\begin{equation}
\label{eq:sde_moment_est_int}
    \mathbb{E} \int_0^{t \wedge \tau_N}
        |X(s)|^{2p} \exp(\gamma s) {ds}
    \leq C_p \left(
        |x_0|^{2p} + 1 + \mathbb{E} 
            \exp(\gamma(t\wedge \tau_N))
    \right).
\end{equation}

To prove the case for $p=1$, 
we start by making the following computation
\begin{equation}
\label{eq:gen_moment_bd}
    L |x|^2 = - 2 \langle x, {\nabla}  F_\mathbf{z}(x) \rangle + \frac{2d}{\beta}
    \leq -2m |x|^2 + 2b + \frac{2d}{\beta},
\end{equation}
where we used \cref{asm:dissipative}. 

Next we let $0 < \gamma_1 < 2m$, and apply It\^{o}'s Lemma to \\
$|X(t \wedge \tau_N)|^2 \exp(\gamma_1 (t\wedge \tau_N))$ we obtain
$\forall t \geq 0$
\begin{align*}
    |X(t \wedge \tau_N)|^2 \exp(\gamma_1 (t\wedge \tau_N))
    =& \, |X(0)|^2 + \gamma_1 \int_0^{t\wedge \tau_N} 
        |X(s)|^2 \exp(\gamma_1 s) ds \\
    & + \int_0^{t\wedge \tau_N} L(|X(s)|^2) \exp(\gamma_1 s) ds \\
    & + \int_0^{t\wedge \tau_N} {2X(s)\exp(\gamma_1 s)}  dW(s).
\end{align*}

Using the fact that stopping at $\tau_N$ bounds all terms of the integrals, 
we have that the final stochastic integral is a martingale. 
Next we take the expectation and use the above computation to get 
$\forall t \geq 0$
\begin{equation}
\label{eq:sde_moment_int_step1}
\begin{aligned}
    \mathbb{E} |X(t \wedge \tau_N)|^2 \exp(\gamma_1 (t\wedge \tau_N))
    \leq & \, |x_0|^2 + (\gamma_1 - 2m) 
        \mathbb{E} \int_0^{t\wedge \tau_N}
        |X(s)|^2 \exp( \gamma_1 s ) ds \\
    & + \left(2b + \frac{2d}{\beta} \right) 
        \mathbb{E} \int_0^{t\wedge \tau_N}
        \exp(\gamma_1 s) ds.
\end{aligned}    
\end{equation}

Since $\gamma_1 < 2m$, we can drop the first integral term. 
Next we use Fatou's Lemma on the left hand side, 
and Monotone Convergence Theorem on the right hand side to get 
$\forall t \geq 0$
\begin{equation*}
    \mathbb{E} |X(t)|^2 \exp(\gamma_1 t)
    \leq |x_0|^2 + \frac{2b + \frac{2d}{\beta}}{\gamma_1} \exp(\gamma_1 t),
\end{equation*}
which proves \cref{eq:sde_moment_est} for $p=1$.

To prove \cref{eq:sde_moment_est_int} for $p=1$, 
we return to \cref{eq:sde_moment_int_step1} and move 
the term $(\gamma_1 - 2m)$ to the left hand side and use 
the estimate above to get 
\[ (2m - \gamma_1) \mathbb{E} \int_0^{t\wedge \tau_N} 
    |X(s)|^2 \exp(\gamma_1 s) ds
    \leq 
    |x_0|^2 + \frac{2b + \frac{2d}{\beta}}{\gamma_1}
    \mathbb{E} \exp(\gamma_1 (t\wedge \tau_N)).
\]

This implies we have the constant 
\[ C_1 = \max\left(\frac{2b + \frac{2d}{\beta}}{\gamma_1}, 1 \right) 
    \cdot \max\left( \frac{1}{2m - \gamma_1}, 1 \right).
\]

Now we prove the induction step for $p$, 
assuming the results \cref{eq:sde_moment_est} and \cref{eq:sde_moment_est_int}
holds for $p-1$.
Similarly we make the following computations
\begin{align*}
    \nabla |x|^{2p} &= 2p |x|^{2p-2} x, \\
    (\nabla \cdot \nabla) |x|^{2p} 
    &= 2p |x|^{2p-2} d + 2p(2p-2) |x|^{2p-4} \langle x, x \rangle, \\
    L |x|^{2p} &= - 2p |x|^{2p-2} \langle x, \nabla F_\mathbf{z}(x) \rangle
    + \frac{2p}{\beta}(d + 2p - 2) |x|^{2p-2} \\
    &\leq - 2pm |x|^{2p} + 2pb |x|^{2p-2} 
    + \frac{2p}{\beta}(d + 2p - 2) |x|^{2p-2},
\end{align*}
where we used \cref{asm:dissipative} in the last inequality.

Then we can apply It\^{o}'s Lemma to 
$|X(t\wedge \tau_N)|^{2p} \exp(\gamma_1(t\wedge \tau_N))$
for some $0 < \gamma_1 < 2m$ to get that 
$\forall t \geq 0$
\begin{equation}
\label{eq:sde_moment_int_step2}
\begin{aligned}
    & \quad 
    \mathbb{E} |X(t\wedge \tau_N)|^{2p} \exp(\gamma_1(t\wedge \tau_N)) \\
    &= |x_0|^{2p} + \gamma_1 \mathbb{E} \int_0^{t \wedge \tau_N}
        \exp(\gamma_1 s) |X(s)|^{2p} ds
        + \mathbb{E} \int_0^{t \wedge \tau_N} L(|X(s)|^{2p}) 
            \exp(\gamma_1 s) ds \\
    &\leq |x_0|^{2p} + (\gamma_1 - 2pm) \mathbb{E} \int_0^{t \wedge \tau_N}
        \exp(\gamma_1 s) |X(s)|^{2p} ds \\
        &\quad + \left( 2pb + \frac{2p}{\beta}(d + 2p - 2) \right)
            \mathbb{E} \int_0^{t \wedge \tau_N}
            |X(s)|^{2p-2} \exp(\gamma_1 s) ds,
\end{aligned}
\end{equation}
where we used the above computation to in the last step. 

Using the fact that $\gamma < 2m \leq 2pm$, 
and \cref{eq:sde_moment_est_int} for $p-1$, we get that 
\begin{align*}
    &\quad \mathbb{E} |X(t\wedge \tau_N)|^{2p} \exp(\gamma_1(t\wedge \tau_N)) \\
    &\leq |x_0|^{2p} + \left( 2pb + \frac{2p}{\beta}(d + 2p - 2) \right)
        C_{p-1} \left( |x_0|^{2p-2} + 1 
        + \mathbb{E}\exp(\gamma_1 (t\wedge \tau_N)) \right).
\end{align*}

Once again using Fatou's Lemma and Monotone Convergence Theorem, 
we have proved \cref{eq:sde_moment_est} for $p$, with constant 
\[ \widetilde C_p = \max\left( \left( 2pb + \frac{2p}{\beta}(d + 2p - 2) \right)
        C_{p-1} + 1, 1 \right).
\]

To prove \cref{eq:sde_moment_est_int}, we return to an earlier step 
at \cref{eq:sde_moment_int_step2} and using the above estimate to get 
\begin{align*}
    &\quad (2pm - \gamma_1) \mathbb{E} \int_0^{t \wedge \tau_N} 
        |X(s)|^{2p} \exp(\gamma_1 s) ds \\
    &\leq \widetilde C_p \left( |x_0|^{2p} + 1 
        + \mathbb{E} \exp(\gamma_1 (t\wedge \tau_N))
        \right).
\end{align*}

This completes the induction proof with the constant 
\begin{align*}
    C_p &= \widetilde C_p \max\left( \frac{1}{2pm - \gamma_1}, 1 \right) \\
    &= \max\left( \left( 2pb + \frac{2p}{\beta}(d + 2p - 2) \right) C_{p-1} + 1, 1 \right)
    \max\left( \frac{1}{2pm - \gamma_1}, 1 \right) \,.
\end{align*}
\end{proof}

\subsection{Proof of Proposition \ref{prop:me_discrete} }
\label{subsec:me_discrete}

\begin{proof}

We start by showing a basic inequality: 
for every $l \in \mathbb{N}^*, \epsilon > 0$, 
there exist a constant $C_{l,\epsilon} > 0$, 
such that for all $x \in \mathbb{R}^d$ we have 
\begin{equation}
\label{eq:pol_est}
    |x|^{2(l-1)} \leq \epsilon |x|^{2l} + C_{l,\epsilon}.
\end{equation}

The result follows from the fact that 
$|x|^{2(l-1)} = \epsilon |x|^{2l}$ when $|x| = \epsilon^{-1}$, 
and therefore it is sufficient to choose 
$C_{l,\epsilon} = \epsilon^{-2(l-1)}$ 
to satisfy \cref{eq:pol_est}.

Next we consider expanding $\mathbb{E} |X_{k+1}|^{2p} $ directly
\begin{equation*}
\begin{aligned}
    &\quad \; \mathbb{E} |X_{k+1}|^{2p} \\
    &= \mathbb{E} \left|X_k - \eta \nabla F_{\zeta_k}(X_k) 
        - \sqrt{ \frac{2\eta}{\beta} } \xi_k
        \right|^{2p} \\
    &= \mathbb{E} \left(
        | X_k - \eta \nabla F_{\zeta_k}(X_k) |^2
        + \frac{2\eta}{\beta} |\xi_k|^2 
        - 2 \left\langle X_k - \eta \nabla F_{\zeta_k}(X_k), 
            \sqrt{ \frac{2\eta}{\beta} } \xi_k
            \right\rangle
        \right)^p \\
    &= \mathbb{E} \sum_{i+j+l = p} \frac{p!}{i!j!l!}
        | X_k - \eta \nabla F_{\zeta_k}(X_k) |^{2i}
        \left( \frac{2\eta}{\beta} |\xi_k|^2 \right)^j
        (-2)^l \left\langle X_k - \eta \nabla F_{\zeta_k}(X_k), 
            \sqrt{ \frac{2\eta}{\beta} } \xi_k
            \right\rangle^l.
\end{aligned}
\end{equation*}

Here we observe that whenever $l$ is odd, 
the term has zero mean due to an odd power of $\xi_k$.
Therefore using the Cauchy-Schwarz inequality we get 
\begin{equation}
\label{eq:moments_exp_terms}
\begin{aligned}
    \mathbb{E} |X_{k+1}|^{2p} 
    &\leq \mathbb{E} \sum_{i + j + 2l = p} \frac{p!}{i!j!(2l)!} 
        | X_k - \eta \nabla F_{\zeta_k}(X_k) |^{2(i + l)} 
        \left( \frac{2\eta}{\beta} |\xi_k|^2 \right)^{j+l}
        2^{2l} \\
    &= \mathbb{E} |X_k - \eta \nabla F_{\zeta_k}(X_k)|^{2p} \\
    &\quad + \sum_{\substack{i+j+2l = p \\ i > 0}}
        \frac{p!}{i!j!(2l)!} 
        | X_k - \eta \nabla F_{\zeta_k}(X_k) |^{2(i + l)} 
        \left( \frac{2\eta}{\beta} |\xi_k|^2 \right)^{j+l}
        2^{2l}
\end{aligned}
\end{equation}

Then we observe after replacing the index $l$ with $2l$, 
whenever $i \neq p$, we have that $i+l < p$, 
so we can isolate the only term with $|X_k - \eta \nabla F_{\zeta_k}(X_k)|^{2p}$.
At the same time, since $| \xi_k |^2 \sim \chi_d$, 
all the moments are bounded. 
This implies we have for all $i+j+2l = p, i > 0$ we have 
\begin{equation*}
\begin{aligned}
    &\quad \mathbb{E}_{X_k} \frac{p!}{i!j!(2l)!} 
        | X_k - \eta \nabla F_{\zeta_k}(X_k) |^{2(i + l)} 
        \left( \frac{2\eta}{\beta} |\xi_k|^2 \right)^{j+l}
        2^{2l} \\
    &= C |X_k - \eta \nabla F_{\zeta_k}(X_k)|^{2(i+l)} \\
    &\leq C \epsilon |X_k - \eta \nabla F_{\zeta_k}(X_k)|^{2p} + C_{2(i+l),\epsilon},
\end{aligned}
\end{equation*}
where $\mathbb{E}_{X_k}$ denotes the expectation conditioned on ${X_k}$, 
and $\epsilon > 0$ can be chosen arbitrarily small 
using \cref{eq:pol_est}.

It is then sufficient to control the terms of the form 
$| X_k - \eta \nabla F_{\zeta_k}(X_k) |^{2p}$
\begin{equation*}
\begin{aligned}
    | X_k - \eta \nabla F_{\zeta_k}(X_k) |^{2p}
    &= \left\langle X_k - \eta \nabla F_{\zeta_k}(X_k), 
            X_k - \eta \nabla F_{\zeta_k}(X_k)
        \right\rangle^{p} \\
    &= \left( |X_k|^2 + \eta^2 |\nabla F_{\zeta_k}(X_k)|^2 
        - 2 \eta \langle X_k, \nabla F_{\zeta_k}(X_k) \rangle
        \right)^p \\
    &\leq \left( |X_k|^2 + \eta^2 \left(
            M |X_k| + M_1 \right)^2
        - 2 \eta m |X_k|^2 + 2\eta b
        \right)^p \\
    &= \left( (1 - 2\eta m + \eta^2 M^2) |X_k|^2 
        + 2\eta^2 M M_1 |X_k| + \eta^2 M_1^2 + 2\eta b
        \right)^p
\end{aligned}
\end{equation*}
where for the inequality we used 
Lemma \ref{lm:quad_bd} and \cref{asm:dissipative} on $F_{\zeta_k}(x)$. 
We remark that this is possible since 
$F_{\zeta_k}(x) = \frac{1}{n_b}\sum_{z \in \zeta_k} f(x,z)$, 
and clearly an empirical average satisfies the same properties.

At this point we can use the condition $0 < \eta < \frac{2m}{M^2}$, 
which implies we can get $r := (1 - 2\eta m + \eta^2 M^2) < 1$. 
Now we separate into two cases, 
first when $0 < r < 1$ we have 
\begin{equation*}
    | X_k - \eta \nabla F_{\zeta_k}(X_k) |^{2p}
    \leq r^p |X_k|^{2p} + C \epsilon |X_k|^{2p} + C
    \leq R |X_k|^{2p} + C,
\end{equation*}
where we used \cref{eq:pol_est} with $\epsilon$ sufficiently small 
such that $0 < R < 1$.

In the second case when we have $r \leq 0$, 
observe 
\[ 0 \leq | X_k - \eta \nabla F_{\zeta_k}(X_k) |^2
    \leq r |X_k|^2 + C |X_k| + C
    \leq C |X_k| + C.
\]
From here we can use \cref{eq:pol_est} again to control the coefficients 
such that for some $0 < R < 1$ we have the same desired result
\[ | X_k - \eta \nabla F_{\zeta_k}(X_k) |^{2p} \leq R |X_k|^{2p} + C.
\]

To complete the proof we return to \cref{eq:moments_exp_terms}, 
and rewrite the terms as 
\begin{align*}
    \mathbb{E} |X_{k+1}|^{2p} 
    &\leq R \, \mathbb{E} |X_k|^2p + C + 
        \sum_{\substack{i+j+2l = p \\ i > 0}} \epsilon \, C_{ijl} \,
            \mathbb{E} |X_k|^{2p} + C_{ijl, \epsilon} \\
    &\leq \widetilde R \, \mathbb{E} |X_k|^{2p} + \widetilde C,
\end{align*}
where we choose $\epsilon$ sufficiently small such that $0 < \widetilde R < 1$.
Then we can simply expand the $X_k$ terms recursively to get 
\begin{align*}
    \mathbb{E} |X_{k+1}|^{2p} 
    &\leq \widetilde{R}^{k+1} |x_0|^{2p} + 
        \sum_{l = 0}^{k+1} \widetilde{R}^{l} \widetilde C
    \leq |x_0|^{2p} + \frac{\widetilde C}{1 - \widetilde R},
\end{align*}
where choosing $C_p = \max\left( \frac{\tilde C}{1 - \tilde R}, 1 \right)$ 
leads to desired result from the statement.

\end{proof}

\subsection{Proof of Proposition \ref{prop:AsExp}}
\label{subsec:AsExp}

The proof to construct an asymptotic expansion of $\E\phi\left( X_k \right)$ follows along the lines of \cite{DebusscheF11}, 
for any $\phi \in C^\infty_{\pol}(\R^d)$. 
Before we start the proof, we define a continuous time process $\widetilde X(t)$
corresponding to the discrete SGLD algorithm $\{X_k\}$ as the following
\begin{equation}
\label{eq:tilde_x}
    \widetilde X(t) := X_k - (t - k \eta) \nabla F_{\zeta_k}(X_k)
        + \sqrt{\frac{2}{\beta}} 
            \left( W(t) - W(k\eta) \right), 
    \quad \forall t \in [k \eta, (k+1)\eta ].
\end{equation}

This leads to the following SDE representation
\begin{equation*}
    d\widetilde X(t) = - \nabla F_{\zeta_k}(X_k) dt 
        + \sqrt{ \frac{2}{\beta} } dW(t), 
    \quad \forall t \in [k \eta, (k+1)\eta ],
\end{equation*}
where most importantly when $0\leq t \leq \eta$ 
we have the infinitesimal generator $\widetilde{L}$ 
only depending on the initial condition $X_0 = x$ 
\[ \widetilde{L}(x, \zeta_0) \phi(\widetilde X(t)) := 
    \sum_{i=1}^d - \de_i F_{\zeta_0}(x) \de_i \phi(\widetilde X(t))
    + \sum_{i=1}^d \frac{1}{\beta} \de_{ii} \phi(\widetilde X(t)), 
\]
and furthermore we also have $\mathbb{E} \widetilde{L}(x,\zeta_0) = L(x)$, 
since the subsample gradient $\nabla F_{\zeta_0}$
is an unbiased estimate of the true gradient $\nabla F_\mathbf{z}(x)$.

\begin{proof}

To start the proof we apply It\^{o}'s Lemma on 
$\phi( \widetilde X(t) )$ for $0 \leq t \leq \eta$ to get
\begin{equation}
\label{eq:ito_phi}
    \phi(\widetilde X(t)) = \phi(x) 
    + \int_0^t \widetilde{L}(x,\zeta_0) \phi(\widetilde X(s)) ds 
    + \sum_{i=1}^d \int_0^t \sqrt{\frac{2}{\beta}} 
        \de_i \phi(\widetilde X(s)) dW^i(s).
\end{equation}

Here we define the operator 
\[ R_{0,i}(x) := \sqrt{\frac{2}{\beta}} \de_i,
\]
and continue to expand by applying \cref{eq:ito_phi}
on $\de_i \phi(\widetilde X(t))$ and $\de_{ii} \phi(\widetilde X(t))$
to get

\begin{align*}
    & \widetilde{L}(x, \zeta_0) \phi(\widetilde X(s_1)) \\
    =& \widetilde{L}(x, \zeta_0) \phi(x) \\
    &+ \int_0^{s_1} \sum_{i_1, i_2 = 1}^d 
        \de_{i_1} F_{\zeta_0}(x) \de_{i_2} F_{\zeta_0}(x)
        \de_{i_1 i_2} \phi(\widetilde X(s_2))
        - \de_{i_1} F_{\zeta_0}(x) \frac{1}{\beta} 
            \de_{i_1 i_2 i_2} \phi(\widetilde X(s_2)) ds_2 \\
    &+ \int_0^{s_1} \sum_{i_1, i_2 = 1}^d 
        - \frac{1}{\beta} \de_{i_2} F_{\zeta_0}(x) 
            \de_{i_1 i_1 i_2} \phi(\widetilde X(s_2))
        + \frac{1}{\beta^2} \de_{i_1 i_1 i_2 i_2} \phi(\widetilde X(s_2)) ds_2 \\
    &+ \sum_{i_1,i_2 = 1}^d \int_0^{s_1}
        - \de_{i_1} F_{\zeta_0}(x) \sqrt{\frac{2}{\beta}} 
            \de_{i_1 i_2} \phi(\widetilde X(s_2))
        + \frac{1}{\beta} \sqrt{\frac{2}{\beta}}
            \de_{i_1 i_1 i_2} \phi(\widetilde X(s_2)) dW^{i_2}(s_2).
\end{align*}

Here we can define 
\begin{align*}
  \widetilde{A}_1(x,\zeta_0) &:= \widetilde{L}(x,\zeta_0), \\
    \widetilde{A}_2(x,\zeta_0) &:= \sum_{i_1, i_2 = 1}^d 
        \de_{i_1} F_{\zeta_0}(x) \de_{i_2} F_{\zeta_0}(x)
        \de_{i_1 i_2} 
        - \de_{i_1} F_{\zeta_0}(x) \frac{1}{\beta} 
            \de_{i_1 i_2 i_2} \\
        &\quad\quad
        + \sum_{i_1, i_2 = 1}^d - \frac{1}{\beta} \de_{i_2} F_{\zeta_0}(x) 
            \de_{i_1 i_1 i_2} 
        + \frac{1}{\beta^2} \de_{i_1 i_1 i_2 i_2},
    \\
    \widetilde{R}_{1,i_2}(x, \zeta_0) &:= \sum_{i_1 = 1}^d 
        - \de_{i_1} F_{\zeta_0}(x) \sqrt{\frac{2}{\beta}} 
            \de_{i_1 i_2} 
        + \frac{1}{\beta} \sqrt{\frac{2}{\beta}}
            \de_{i_1 i_1 i_2},
\end{align*}
which would lead to 
\begin{align*}
    \phi(\widetilde X(t)) &= \phi(x) + t \widetilde{A}_1(x,\zeta_0) \phi(x)
        + \int_0^t \int_0^{s_1} \widetilde{A}_2(x,\zeta_0) 
          \phi(\widetilde X(s_2)) ds_2 \\
    &\quad + \sum_{i_2 = 1}^d \int_0^t \int_0^{s_1} 
        \widetilde{R}_{1,i_2}(x,\zeta_0) \phi(\widetilde X(s_2)) dW^{i_2}(s_2) ds_1 \\
    &\quad + \sum_{i_1 = 1}^d \int_0^t R_{0,i_1} \phi(\widetilde X(s_1)) 
            dW^{i_1}(s_1).
\end{align*}

At this point we observe that the last two integrals are local martingales, 
therefore a localization argument can remove them in expectation. 
To be precise, we will define the stopping time 
$\tau_c := \inf\{ t \geq 0 : |\widetilde X(t)| \geq c \}$
for some $c > 0$. 
Therefore we have that 
\[ \mathbb{E} \phi(\widetilde X(t \wedge \tau_c ) )
    = \phi(x) + (t \wedge \tau_c) A_1(x) \phi(x)
    + \mathbb{E} \int_0^{t \wedge \tau_c} \int_0^{s_1}
        A_2(x) \phi(\widetilde X(s_2)) ds_2, 
\]
where $A_i(x) = \mathbb{E} \widetilde{A}_i(x,\zeta_0)$ 
is the expectation over the subsampling randomness $\zeta_0$.

Here we observe that $A_2(x)$ is a $4^\text{th}$ order 
differential operator, with $C_2^\infty(\mathbb{R}^d)$ coefficients, 
since we have $|\de_i F_{\zeta_0}(x)| \leq C(1 + |x|)$
from Lemma \ref{lm:quad_bd}. 
Therefore with $N=1$ we have 
\[ |A_2(x) \phi(\widetilde X(s_2))|
    \leq C ( 1 + |x|^2 ) (1 + |\widetilde X(s_2)|^{l_{2N + 2}}) 
        |\phi|_{2N + 2, l_{2N + 2}},
\]
for some $l_{2N + 2}$ satisfying the Proposition statement.

Now applying the moment estimate from 
Proposition \ref{prop:me_discrete} with $\eta = s_2$ 
we have that 
\begin{align*}
    & \quad \left| \mathbb{E} \phi(\widetilde X(t \wedge \tau_c))
        - \phi(x) - (t \wedge \tau_c) A_1(x) \phi(x)
    \right| \\
    & \leq C (t\wedge \tau_c)^2 (1 + |x|^2) (1 + \mathbb{E} |\widetilde X(s_2)|^{l_{2N+2}} )
        |\phi|_{2N+2, l_{2N+2}} \\
    & \leq C_N \eta^{N+1} (1 + |x|^\alpha) |\phi|_{2N+2, l_{2N+2}},
\end{align*}
where $\alpha = l_{2N + 2} + 2$, and we take $N=1$ here.
And since we can take $t = \eta$ and $c > 0$ is arbitrary, 
we have proven the Proposition statement for $N=1$.

For the general statement, we will prove inductively 
the following statement for all $N$
\begin{equation}
\label{eq:asymp_exp_induction}
\begin{aligned}
    \phi(\widetilde X(t)) =& \phi(x)
        + t \widetilde{L}(x,\zeta_0) \phi(x) 
        + \sum_{j=2}^N t^j \widetilde{A}_j(x,\zeta_0) \phi(x) \\
    & + \int_0^t \cdots \int_0^{s_N} \widetilde{A}_{N+1}(x,\zeta_0) 
        \phi(\widetilde X(s_{N+1}))
        ds_{N+1} \cdots ds_1 \\
    & + \sum_{j=0}^N \sum_{i=1}^d \int_0^t \cdots \int_0^{s_j}
        \widetilde{R}_{j,i}(x,\zeta_0) 
        \phi(\widetilde X(s_{j+1})) dW^i(s_{j+1}) ds_j \cdots ds_1.
\end{aligned}
\end{equation}

Assume the above statement is true for $N$, 
with $\widetilde{A}_{N+1}(x,\zeta_0)$ and 
$\widetilde{R}_{N,i}(x,\zeta_0)$ are known, 
and we will proceed to prove the case for $N+1$.

Here we start by decomposing $\widetilde{A}_{N+1}(x,\zeta_0)$ into 
\[ \widetilde{A}_{N+1}(x,\zeta_0) = 
  \sum_{\mathbf{k}} \widetilde{A}_{N+1}^\mathbf{k}(x,\zeta_0) \de_\mathbf{k},
\]
where $\mathbf{k} \in \mathbb{N}^d$ are multi-indices, 
and each $\widetilde{A}_{N+1}^\mathbf{k}(\cdot, \zeta_0) 
\in C^\infty_{\pol}(\mathbb{R}^d)$ 
since they are products of $\de_i F_{\zeta_0}(x)$ and $\frac{1}{\beta}$.

Then we similarly apply \cref{eq:ito_phi} to 
each of the terms $\de_\mathbf{k} \phi(\widetilde X(s_{N+1}))$ to get 
\begin{align*}
    & \sum_\mathbf{k} \widetilde{A}^\mathbf{k}_{N+1}(x,\zeta_0) \de_\mathbf{k} 
        \phi(\widetilde X(s_{N+1})) \\
    =& \sum_\mathbf{k} \widetilde{A}^\mathbf{k}_{N+1}(x,\zeta_0) 
    \de_\mathbf{k} \phi(x) \\
    &+ \sum_\mathbf{k} \int_0^{s_{N+1}} \sum_{i=1}^d 
        \widetilde{A}^\mathbf{k}_{N+1}(x,\zeta_0) 
        ( - \de_i F_{\zeta_0}(x) )
        \de_i \de_\mathbf{k} \phi(\widetilde X(s_{N+2})) ds_{N+2} \\
    &+ \sum_\mathbf{k} \int_0^{s_{N+1}} \sum_{i=1}^d
        \widetilde{A}^\mathbf{k}_{N+1}(x,\zeta_0) 
        \frac{1}{\beta} \de_{ii} \de_\mathbf{k}
        \phi(\widetilde X(s_{N+2})) ds_{N+2} \\
    &+ \sum_\mathbf{k} \sum_{i=1}^d \int_0^{s_{N+1}} 
        \widetilde{A}^\mathbf{k}_{N+1}(x,\zeta_0) \sqrt{\frac{2}{\beta}} \de_i \de_\mathbf{k} \phi(\widetilde X(s_{N+2}))
            dW^i(s_{N+2}).
\end{align*}

This implies the following definitions 
\begin{align*}
    \widetilde{A}_{N+2}(x,\zeta_0) &:= \sum_\mathbf{k} \sum_{i=1}^d 
        \widetilde{A}^\mathbf{k}_{N+1}(x,\zeta_0) 
        (- \de_i F_{\zeta_0}(x))
        \de_i \de_\mathbf{k}
        + \widetilde{A}^\mathbf{k}_{N+1}(x,\zeta_0) 
        \frac{1}{\beta} \de_{ii} \de_\mathbf{k}, \\
    \widetilde{R}_{N+1,i}(x,\zeta_0) &:= \sum_\mathbf{k} \sum_{i=1}^d 
        \widetilde{A}^\mathbf{k}_{N+1}(x,\zeta_0) 
        \sqrt{\frac{2}{\beta}} \de_i \de_\mathbf{k},
\end{align*}
which implies all coefficients of $\widetilde{A}_N(\cdot, \zeta_0)$ 
are in $C^\infty_N(\mathbb{R}^d)$, 
hence we obtained the result from \cref{eq:asymp_exp_induction} 
with $N+1$.

Using the same localization argument for $N=1$, we can obtain the desired result.

\end{proof}
\section{Uniform Stability: Proof of \cref{thm:pi_gen_bound}}
\label{sec:appendix_uniform_stability}

\subsection{Proof Overview}
\label{subsec:appendix_uniform_stability}

To help make the lengthy proof more readable, 
we start with a high level section outlining 
the key technical Lemmas containing the most important 
high level ideas. 

\subsubsection{Notations and Steps of the Proof}
Here, we adopt the notation 
$C^\infty_{\pol}(\mathbb{R}^d) := 
\cap_{m \geq 0} \cup_{\ell \geq 0} C^m_\ell(\mathbb{R}^d)$ 
for all smooth functions with polynomial growth. 
We remark while all the results in this sections 
are stated for $C^\infty_{\pol}(\mathbb{R}^d)$ functions, 
we only ever differentiate $6N+2$ times, 
therefore it does not contradict \cref{asm:smooth_loss}. 

Without loss of generality, let $i$ be the differing 
coordinate between two data points $\mathbf{z}$ and $\overline{\mathbf{z}}$, 
i.e. $\mathbf{z} = (z_1, \ldots, z_i, \ldots, z_n)$ and 
$\overline{\mathbf{z}} = (z_1, \ldots, \overline{z_i}, \ldots, z_n)$. 
We will add the subscript $\mathbf{z}$ or $\overline{\mathbf{z}}$
to make the dependence on the data explicit, for example 
$L_\mathbf{z}, \mu_{k,\mathbf{z}}, G_{k, \mathbf{z}}, \pi^N_\mathbf{z}$ etc.
We also define $\mathbf{z}^{(i)} := \mathbf{z} \cap \overline{\mathbf{z}}$ 
as the common set of data, and let 
\[ F_{\mathbf{z}^{(i)}} := \frac{1}{n} \sum_{j=1, j\neq i}^n 
    f(x, z_j).
\]
Thus $F_\mathbf{z} = F_{\mathbf{z}^{(i)}} + \frac{1}{n} f(x,z_i)$
and we define a new Gibbs probability measure $\rho_{\mathbf{z}^{(i)}}$ as 
\[ \rho_{\mathbf{z}^{(i)}} := \frac{1}{Z^{(i)}} \exp( -\beta F_{\mathbf{z}^{(i)}} ),
\]
where $Z^{(i)} = \int_{\mathbb{R}^d} \exp( -\beta F_{\mathbf{z}^{(i)}} ) dx $ 
is the normalizing constant. Finally, we define Radon-Nikodym derivatives of $\rho_{\mathbf{z}}$ and $\rho_{\overline{\mathbf{z}}}$ with respect to $\rho_{\mathbf{z}^{(i)}}$
\[ q_{z_i} := \frac{d\rho_{\mathbf{z}}}{d\rho_{\mathbf{z}^{(i)}}} =  \frac{Z^{(i)}}{Z} \exp\left( - \frac{\beta}{n} f(x,z_i) \right),
    \quad 
    q_{\overline{z_i}} := \frac{d\rho_{\overline{\mathbf{z}}}}{d\rho_{\mathbf{z}^{(i)}}}= \frac{Z^{(i)}}{ \overline{Z} } 
    \exp\left( - \frac{\beta}{n} f(x,\overline{z_i}) \right),
\]
where $\overline{Z}$ is the normalizing constant for $\rho_{\overline{\mathbf{z}}}$, 
so that we can write $\rho_\mathbf{z} = \rho_{\mathbf{z}^{(i)}} q_{z_i}$
and $\rho_{\overline{\mathbf{z}}} = \rho_{\mathbf{z}^{(i)}} q_{\overline{z_i}}$.

\vspace{1\baselineskip}

In this section, we prove a uniform stability bound 
for the modified invariant measure $\pi^N_\mathbf{z}$. 
This will imply a generalization bound on $T_2$ 
using \cref{prop:signed_stability}. 
The result is given below.

\begin{theorem}
[Generalization Bound of $\pi^N_{\mb{z}}$]
Suppose $\{X_k\}_{k\geq 0}$ is any discretization 
of Langevin diffusion \eqref{eq:langevin_diffusion} 
with an approximate stationary distribution $\pi^N_{\mb{z}}$ 
of the type in \cref{thm:backward_analysis_main}. 
Then there exists a constant $C>0$ (depending on $N$), 
such that for all choices of $k,n$ and $\eta \in (0,1)$ 
the following expected generalization bound holds 
\begin{equation}
    \left| \, 
        \mathbb{E} 
        \left[ \pi^N_{\mb{z}}( F ) 
            - \pi^N_{\mb{z}}( F_{\mb{z}} ) 
        \right] \, 
    \right| 
    \leq 
        \frac{C}{n (1-\eta) } \,, 
\end{equation}
where the expectation is with respect to 
$\mb{z} \sim \mathcal{D}^n$. 

\end{theorem}

From \cref{sec:appendix_weak_backward}, 
we recall the invariant measure $\pi^N_\mathbf{z}$ 
is constructed inductively using the terms $\mu_{k,\mathbf{z}}$
\eqref{eq:appendix_mu_decomp}, 
and each terms $\mu_{k,\mathbf{z}}$ in the asymptotic expansion satisfies 
a Poisson equation in $\R^d$ for the elliptic operators 
$L_{\mathbf{z}}$, i.e. for all $k \geq 0$ we have
\begin{align} 
&L_\mathbf{z} \mu_{k, \mathbf{z}} 
    = G_{k, \mathbf{z}}, \label{Eq:poisson_z}\\
   & L_{\overline{\mathbf{z}}} \mu_{k, \overline{\mathbf{z}}} = G_{k, \overline{\mathbf{z}}}. \label{Eq:poisson_zbar}
\end{align}
The main idea of the proof relies on comparing the different of solutions for the pairs of Poisson equations in $\R^d$.
We will breakdown the proof of this result into several steps: 
\begin{enumerate}
    \item \emph{Sufficient condition for uniform stability and spectral gap:} We derive 
    a sufficient condition for uniform stability, 
    reducing the problem down to studying the pairs of Poisson equations separately. 
    Using the common Gibbs measure $\rho_{\mathbf{z}^{(i)}}$, 
    we provide bounds on the desired norms using 
    a spectral gap on this measure.  
    Standard results guarantee the existence of a spectral gap of size $\lambda'>0$ for the measure $\rho_{\mathbf{z}^{(i)}}$ and equivalently 
    a Poincar\'e inequality with constant $\lambda'$. This result will be used to derive the following energy estimates.
    \item \emph{Zeroth-order energy estimates:} We give a further sufficient condition for uniform stability
    on the non-homogeneous terms $G_{k,\mathbf{z}}$ and $G_{k,\mathbf{z}}$. 
    The bounds are obtained using standard energy estimate techniques for the Poisson equations \citep[Section 6.2.2]{evans2010partial} and the Poincar\'e inequality. 
    \item \emph{Higher order energy estimates:} 
    We complete the proof by proving that the sufficient condition from the previous part holds, as $G_{k,\mathbf{z}}$ is defined recursively from $\mu_{l,\mathbf{z}}$ with $l < k$. 
    This requires us to obtain higher order energy estimates using similar techniques. 
\end{enumerate}

\subsubsection{Sufficient Condition for Uniform Stability and Spectral Gap}
\label{subsubsec:sufficient_condition}
We provide a sufficient condition for uniform stability, 
which allows us to reduce the problem to finding a bound on the difference 
$\mu_{k,\mathbf{z}} q_{z_i} - 
\mu_{k,\overline{\mathbf{z}}} q_{\overline{z_i}}$. 
\begin{lemma} [Sufficient to Bound $L^2$ norms]
\label{lm:bound_l2_norm}
A sufficient condition for uniform stability of $\pi^N_\mathbf{z}$ is 
if for each $k = 1, \ldots, N$, there exist a constant $C_{\mu_k}>0$, 
independent of data and $n$, such that
\begin{align} \label{eq:L2_norm}
\| \mu_{k,\mathbf{z}} q_{z_i} - 
    \mu_{k,\overline{\mathbf{z}}} q_{\overline{z_i}} \|_{
        L^2(\rho_{\mathbf{z}^{(i)}})} 
    := \left[ \int_{\mathbb{R}^d} (\mu_{k,\mathbf{z}} q_{z_i}
    - \mu_{k,\overline{\mathbf{z}}} q_{\overline{z_i}})^2
    d \rho_{\mathbf{z}^{(i)}}
    \right]^{1/2} 
    \leq \frac{ C_{\mu_k} }{n}.
\end{align}
\end{lemma}
The proof can be found in \cref{subsec:bound_l2_norm}.

An important tool is the following Poincar\'e inequality for the space $C^1(\R^d) \cap L^2(\rho_{\mathbf{z}^{(i)}})$.
We will use this Lemma to derive energy estimates and obtain uniform stability bound of order $\mathcal{O}(1/n)$.

\begin{lemma}[Poincar\'e Inequality for $\rho_{\mathbf{z}^{(i)}}$]
\label{lemma:spectral_gap}
There exists a uniform spectral gap constant $\lambda'>0$ so that,
for every $\mathbf{z}^{(i)} \in \mathcal{Z}^{n-1}$
and every function $u \in C^1(\R^d) \cap L^2(\rho_{\mathbf{z}^{(i)}})$,
\[
\int_{\R^d} u d\rho_{\mathbf{z}^{(i)}} = 0  \Longrightarrow \int_{\R^d} \abs{u}^2 d\rho_{\mathbf{z}^{(i)}} \leq \frac{1}{\lambda'} \int_{\R^d} \abs{\nabla u}^2 d\rho_{\mathbf{z}^{(i)}}.
\]
\end{lemma}

\begin{proof}
The proof follows the same argument as in \citet[Appendix B]{raginsky2017nonconvex}, using Lyapunov functional techniques developed by \citet{bakry2008simple}. 
\cref{asm:smooth_loss,asm:dissipative} are sufficient 
to guarantee that $\lambda'>0$ for $\rho_{\mathbf{z}^{(i)}}$.
\end{proof}

\subsubsection{Zeroth-Order Energy Estimates}

Based on Lemma \ref{lm:bound_l2_norm}, we must prove an inequality of the form \eqref{eq:L2_norm}.
Our approach is based on energy estimates for the solutions to the Poisson equations \eqref{Eq:poisson_z} and \eqref{Eq:poisson_zbar}.
The following Lemma simply states that if the difference 
$G_{k, \mathbf{z}} q_{z_i} - 
G_{k, \overline{\mathbf{z}}} q_{\overline{z_i}}$ 
is of order $\mathcal{O}(1/n)$, then inequality \eqref{eq:L2_norm} 
holds, and the measure $\pi^N_{\mathbf{z}}$ is uniformly stable.

\begin{lemma}[Zeroth-Order Energy Estimates]
\label{lm:energy_est1}
For all $k \geq 0$, assume there exists a non-negative constant $C_{G_k}$ such that 
\[ \norm{G_{k, \mathbf{z}} q_{z_i} - 
    G_{k, \overline{\mathbf{z}}} q_{\overline{z_i}} }_{
            L^2\left( \rho_{\mathbf{z}^{(i)}} \right)} \leq \frac{C_{G_k}}{n}.
 \]
Then there exists another non-negative constant $C_{\mu_k}$ such that 
\[
\norm{ \mu_{k,\mathbf{z}} q_{z_i} - 
    \mu_{k,\overline{\mathbf{z}}} q_{\overline{z_i}}
    }_{
            L^2\left( \rho_{\mathbf{z}^{(i)}} \right)} \leq \frac{C_{\mu_k}}{n}.
\]
\end{lemma}
The proof is can be found in \cref{subsec:energy_est1}. 

\subsubsection{Higher Order Energy Estimates}
To complete the proof of \cref{thm:pi_gen_bound}, 
we need to show that 
the sufficient condition in Lemma \ref{lm:bound_l2_norm} holds. 
This implies we need to control 
the $ L^2\left( \rho_{\mathbf{z}^{(i)}} \right)$ norm of 
the difference $G_{k, \mathbf{z}} q_{z_i} - 
G_{k, \overline{\mathbf{z}}} q_{\overline{z_i}}$ 
for all $k \geq 0$. Let us recall that by definition 
$ G_{k, \mathbf{z}} := - \sum_{\ell=1}^k 
    L^*_{\ell, \mathbf{z}} \mu_{k-\ell, \mathbf{z}},
$
which lets us write 
\begin{equation}
\label{eq:higher_order_energy_decomp}
\begin{aligned}
    G_{k, \mathbf{z}} q_{z_i} - 
    G_{k, \overline{\mathbf{z}}} q_{\overline{z_i}}
    &= - \sum_{\ell=1}^k L^*_{\ell, \mathbf{z}} \mu_{k-\ell, \mathbf{z}} q_{z_i}
        - L^*_{\ell, \overline{\mathbf{z}}} \mu_{k-\ell, \overline{\mathbf{z}}} 
            q_{\overline{z_i}} \\
    &= - \sum_{\ell=1}^k ( L^*_{\ell, \mathbf{z}} \mu_{k-\ell, \mathbf{z}} q_{z_i}
            - L^*_{\ell, \mathbf{z}} 
            \mu_{k-\ell, \overline{\mathbf{z}}} q_{\overline{z_i}} )
        + ( L^*_{\ell, \mathbf{z}} - L^*_{\ell, \overline{\mathbf{z}}} )
            \mu_{k-\ell, \overline{\mathbf{z}}} q_{\overline{z_i}} \\
    &=: - \sum_{\ell=1}^k T_{1, \ell} + T_{2, \ell} \,.
\end{aligned}
\end{equation}
Note that in the above expressions all operators $L^*_{\ell, \mathbf{z}}$ only act on the smooth functions $\mu$ and not on $q$.

To control $T_{2, \ell}$, 
notice that all operators $L^*_{\ell, \mathbf{z}}$
can be written in non-divergence form with $C^\infty_{\pol}(\R^d)$ coefficients.
\begin{lemma}
[$C^\infty_{\pol}(\R^d)$ Coefficients]
\label{lm:coefficients}
For all $\ell \geq 0$ and $\alpha \in \N^d$, 
the operator $L_{\ell, \mathbf{z}}^*$ has $C^\infty_{\pol}(\R^d)$ coefficients. 
I.e. there exist functions $\phi_{\ell, \alpha} \in C^{\infty}_{pol}(\R^d)$ 
such that we can write 
\[
L^*_{\ell, \mathbf{z}}  = \sum_{0 \leq \abs{\alpha} \leq 2\ell + 2}  \phi_{\ell, \alpha}(x)\partial_{\alpha} \,.
\]
\end{lemma}

The proof can bound found in \cref{subsec:coefficients}.

It is easy to see that a control on the terms $T_{2, \ell}$ directly follows from the following Lemma. 

\begin{lemma}
\label{lm:closeness_of_operators}
For all $\ell > 0$, there exists a differential operator $\widehat L_\ell:= \widehat L_\ell(x)$ of order $2\ell+2$ with $C^\infty_{\pol}(\mathbb{R}^d)$ coefficients, and independent of $n$, 
such that we can write 
\[ \frac{1}{n} \widehat L_\ell^* 
    = L^*_{\ell, \mathbf{z}} - L^*_{\ell, \overline{\mathbf{z}}} \,.
\]
Furthermore, for all $\ell \geq 0$, 
$\phi \in C^\infty_{\pol}(\R^d)$, 
and $\alpha \in \mathbb{N}^d$, 
there exist non-negative constants $C_{\ell, \phi \de_\alpha}$, depending on the $L^2\left( \rho_{\mathbf{z}^{(i)}} \right)$-norm of $\mu_{k-\ell, \overline{\mathbf{z}}}$ and its derivatives up to order $2\ell+2 + |\alpha|$, such that
\[
\norm{\phi \de_\alpha (( L^*_{\ell, \mathbf{z}} - L^*_{\ell, \overline{\mathbf{z}}} )
            \mu_{k-\ell, \overline{\mathbf{z}}}) q_{\overline{z_i}} }_{ L^2\left( \rho_{\mathbf{z}^{(i)}} \right)} \leq \frac{C_{\ell,\phi\de_\alpha}}{n} \,.
\]
\end{lemma}

The proof is done by induction on the recursive construction of the operators $A_j$ and $L_j$.
All operators are linear in the potential function $F_{\mathbf{z}}$ and $F_{\overline{\mathbf{z}}}$.
The full proof is in \cref{subsec:closeness_of_operators}

It remains to control the terms $T_{1, \ell}$. 
From the expression of $L^*_{\ell, \mathbf{z}}$, it is clear that 
zeroth-order estimates given in Lemma \ref{lm:energy_est1} are not sufficient to conclude the proof. 
we need to obtain estimates on higher order derivatives with non-constant coefficients.
\begin{lemma}[Higher Order Energy Estimates]
\label{lm:energy_est2}
Fix $k \in \mathbb{N}$. 
If for all $\phi \in C^\infty_{\pol}(\mathbb{R}^d)$, 
$\ell < k$, 
and multi-index $\alpha \in \mathbb{N}^d$, 
there exists a constant $C_{\phi \de_\alpha G_\ell} > 0$, 
such that we have
\[ \| \phi (\de_\alpha G_{\ell, \mathbf{z}}) q_{z_i} - 
    \phi (\de_\alpha G_{\ell, \overline{\mathbf{z}}}) q_{\overline{z_i}} \|_{
            L^2( \rho_{\mathbf{z}^{(i)}} )} 
    \leq \frac{C_{\phi \de_\alpha G_\ell}}{n}, 
\]
then for all $J \in \mathbb{N}$, 
and degree-$J$ differential operator $L$ 
with coefficients in $C^\infty_{\pol}(\mathbb{R}^d)$ (i.e. \\
$ L := \sum_{0 \leq |\alpha| \leq J} \phi^\alpha(x) \de_\alpha,
$
where $\phi^\alpha(x) \in C^\infty_{\pol}(\mathbb{R}^d)$ for each 
$|\alpha| \leq J$), 
there exist a constant $C_{L \mu_k} > 0$ such that 
\[ \| L \mu_{k,\mathbf{z}} q_{z_i} - 
    L \mu_{k,\overline{\mathbf{z}}} q_{\overline{z_i}} \|_{
            L^2( \rho_{\mathbf{z}^{(i)}} )} 
    \leq \frac{C_{L \mu_k}}{n}.
\]

In particular, the above inequality holds for 
$L = \phi' \de_{\alpha'} L^*_{\ell, \mathbf{z}}$
with $\phi' \in C^\infty_{\pol}(\R^d)$, 
$\alpha' \in \mathbb{N}^d$, 
and $J = 2\ell + 2 + |\alpha'|$, 
therefore there exists a constant $C_{\phi' \de_{\alpha'} G_k} > 0$ 
such that 
\[ \| \phi' (\de_{\alpha'} G_{k, \mathbf{z}}) q_{z_i} - 
    \phi' (\de_{\alpha'} G_{k, \overline{\mathbf{z}}}) q_{\overline{z_i}} \|_{
            L^2( \rho_{\mathbf{z}^{(i)}} )} 
    \leq \frac{C_{\phi' \de_{\alpha'} G_k}}{n}, 
\]
hence proving the induction step from $k-1$ to $k$. 
\end{lemma}
The proof can be found in \cref{subsec:energy_est2}.

Using the above result, we can finally provide a proof 
for uniform stability. 

\begin{proof}(of \cref{thm:pi_gen_bound})

We first apply Lemma \ref{lm:bound_l2_norm} so that it is sufficient to show that for each $k \in \mathbb{N}$, there exists a positive constant $C_{\mu_k}$ independent of $n$ and data such that 
\[ \| \mu_{k,\mathbf{z}} q_{z_i} - 
    \mu_{k,\overline{\mathbf{z}}} q_{\overline{z_i}} 
    \|_{ L^2( \rho_{\mathbf{z}^{(i)}} ) } 
    \leq \frac{C_{\mu_k}}{n} \, .
\]

Then using Lemma \ref{lm:energy_est1}, 
it is sufficient to show that for all 
$k\in \mathbb{N}, \phi \in C^\infty_{\pol}(\R^d), \alpha \in \mathbb{N}^d$, 
there exists a positive constant $C_{\phi \de_\alpha G_k}$ such that 
\[ \| \phi \de_\alpha G_{k,\mathbf{z}} q_{z_i} - 
    \phi \de_\alpha G_{k,\overline{\mathbf{z}}} q_{\overline{z_i}} 
    \|_{ L^2( \rho_{\mathbf{z}^{(i)}} ) } 
    \leq \frac{ C_{\phi \de_\alpha G_k} }{n} \, .
\]

Finally we will prove the above condition using induction on $k$. 
Recalling the definition of $G_{k,\mathbf{z}} = 0$ when $k=0$, 
then the $k=0$ case follows trivially. 
Now assuming the cases $0, 1, \cdots, k-1$ are true, 
we follow the decomposition in \cref{eq:higher_order_energy_decomp}, 
and apply Lemmas \labelcref{lm:closeness_of_operators,lm:coefficients,lm:energy_est2} on the terms 
$\phi \de_\alpha T_{1,\ell}$ and 
$\phi \de_\alpha T_{2,\ell}$, 
for arbitrary $\phi \in C^\infty_{\pol}(\R^d)$
and $\alpha \in \mathbb{N}^d$. 
This proves the desired control for the difference $\phi \de_\alpha G_{k,\mathbf{z}} q_{z_i} - \phi \de_\alpha G_{k,\overline{\mathbf{z}}} q_{\overline{z_i}}$, 
hence showing the induction step for $k$. 
\end{proof}

\subsection{Poisson Equation Results}

Before we start the proofs, we will state a result of existence, uniqueness, and polynomial growth estimate for solutions of Poisson equations adapted from \citet[Theorem 1]{pardoux2001} and \citet[Proposition 1]{pardoux2003poisson} to fit our assumptions. 

\begin{proposition}
\label{prop:poisson_exist_unique}
Under \cref{asm:smooth_loss,asm:dissipative}, 
for Poisson equations of the form 
\[ L_\mathbf{z} u = G,
\]
where $\int_{\R^d} G d\rho_\mathbf{z} = 0$, 
there exist a unique function $u$ that solves this equation. Furthermore, if we have for some constants $C>0,k>0$ such that
\[ |G(x)| \leq C (1 + |x|^k), \]
then for all $k' > k + 2$, there exists a constant $C'>0$ such that 
\[ |u(x)| + |\nabla u(x)| \leq C' (1 + |x|^{k'}). \]
\end{proposition}

We note here that \citet{pardoux2001, pardoux2003poisson} proved the above result under much more general conditions, and it is straight forward to verify that our assumptions fall under a special case of the original result.

To get higher regularity, 
we will first state standard a higher regularity result from for bounded domains. 

\begin{theorem}
\citep[Section 6.3.1, Theorem 3, Infinite Differentiability in the Interior]{evans2010partial}
\label{thm:evans_poisson_regularity}
Let $U \subset \mathbb{R}^d$ be a bounded open domain with boundary $\de U \in C^1$, 
$L$ be a uniformly elliptic operator with smooth ($C^\infty(U)$) coefficients, 
$f \in C^\infty(U)$, 
and $u$ is a weak solution of the equation 
\[ Lu = f \, . 
\]
Then we have $u \in C^\infty(U)$.
\end{theorem}

We will adapt this classical result to our problem. 

\begin{proposition}
\label{prop:appendix_poisson_regularity}
Suppose \cref{asm:smooth_loss,asm:dissipative} 
for Poisson equations of the form 
\[ L_\mathbf{z} u = G \, ,
\]
where $\int_{\R^d} G d\rho_\mathbf{z} = 0$.
If additionally $F_\mathbf{z}, G \in C^\infty_{\pol}(\R^d)$, then we have 
$u \in C^\infty_{\pol}(\R^d)$.
\end{proposition}

\begin{proof}

Since we already have existence and uniqueness of solutions for the Poisson equation from Proposition \ref{prop:poisson_exist_unique}, 
we can restrict the solution $u$ to any open bounded domain $U$ with smooth boundary. 
Then using Theorem  \ref{thm:evans_poisson_regularity}, on any cover of $\mathbb{R}^d$ using open bounded domains with smooth boundary $\{U_i\}_{i \in I}$, we obtain that $u \in C^\infty(\mathbb{R}^d)$. 

To show polynomial growth, we will consider an induction on $k = |\alpha|$, 
where $\alpha \in \mathbb{N}^d$ is a multi-index. 
We start by observing that the $k = 0$ case follows trivially from Proposition \ref{prop:poisson_exist_unique}. 
Assuming the case is true for $0, 1, \ldots, k-1$, 
we will prove the case for $|\alpha| = k$. 
We start by computing 
\[ \de_\alpha (L_\mathbf{z} u)
    = L_{\mathbf{z}} \de_\alpha u
    + \sum_{l=1}^{J} \sum_{
    \substack{\alpha_1 + \alpha_2 = \alpha \\
        |\alpha_1| = l }
    }
    (\de_{\alpha_1} L_\mathbf{z}) (\de_{\alpha_2} u)
    = \de_\alpha G \,.
\]

This implies $v = \de_\alpha u$ solves the Poisson equation 
\[ L_\mathbf{z} v = \widetilde{G}
    := \de_\alpha G - \sum_{l=1}^{J} \sum_{
    \substack{\alpha_1 + \alpha_2 = \alpha \\
        |\alpha_1| = l }
    }
    (\de_{\alpha_1} L_\mathbf{z}) (\de_{\alpha_2} u)
    \, ,
\]
where we define for any $\phi \in C^1(\mathbb{R}^d)$
\[ (\de_{\alpha_1} L_\mathbf{z}) \phi
    := - \langle \nabla \de_{\alpha_1}
        F_\mathbf{z}, \nabla \phi \rangle. 
\]

Since all the $|\alpha_2| < k$, then by the induction hypothesis, we have that there exist constants $C, k > 0$, such that 
for all $\alpha_1 + \alpha_2 = \alpha, 
|\alpha_2| < k$, we have 
\[ |(\de_{\alpha_1} L_\mathbf{z}) 
    (\de_{\alpha_2} u)|
    \leq C(1 + |x|^k) \,. 
\]

Therefore $\widetilde G$ must also only have polynomial growth. 
Finally, applying Proposition \ref{prop:poisson_exist_unique} on
$v = \de_\alpha u$, 
we obtain that there exist constants 
$C', k' > 0$ such that 
\[ |\de_\alpha u| \leq C'( 1 + |x|^{k'} )\, ,
\]
which is the desired result. 

\end{proof}

\subsection{Moment Bounds for $\rho_\mathbf{z}, \rho_{\mathbf{z}^{(i)}}$}

\begin{lemma}
\label{lm:appendix_moment_bound}
For all $k\in \mathbb{N}$, 
we have that 
\[ \int_{\R^d} |x|^k d\rho_\mathbf{z} 
    < \infty \,.
\]

Additionally, by absorbing any factors of the type 
$\frac{n-1}{n}$ or $\frac{n+1}{n}$ into $\beta$, 
we have for all $\phi \in C^\infty_{\pol}(\R^d)$, 
\[ \|\phi\|_{L^2(\rho_{\mathbf{z}^{(i)}})}
    + \|\phi q_{z_i}\|_{L^2(\rho_{\mathbf{z}^{(i)} })}
    + \|\phi \nabla f(x,z_i) q_{z_i}     \|_{L^2(\rho_{\mathbf{z}^{(i)}})}
    < \infty \,. 
\]
\end{lemma}

\begin{proof}

These bounds follows directly from Proposition \ref{prop:ME_continuous}, 
since we can write for even $k$
\[ \int_{\R^d} |x|^k d\rho_\mathbf{z}
    =
    \lim_{t\to\infty} \mathbb{E} |X(t)|^k 
    \leq C_{k/2} \,.
\]

For an odd $k$, we use Young's inequality to write 
\[ |x|^k \leq C |x|^{k+1},
\]
for some constant $C$. 

Since $\phi < C_\phi (1 + |x|^{k_\phi})$, 
we have $\|\phi\|_{L^2(\rho_{\mathbf{z}^{(i)}})} < \infty$ follows immediately. 

At the same time, we can write 
\[ q_{z_i}^2 \rho_{\mathbf{z}^{(i)}}
    = \frac{1}{Z'} \exp\left( 
        -\beta \left(F_\mathbf{z} + 
        \frac{1}{n} f(x,z_i) 
        \right) \right)
    = \widehat \rho \,. 
\]

Therefore we can obtain a new bound as
\[ \|\phi q_{z_i}\|_{L^2(\rho_{\mathbf{z}^{(i)} })}
= \|\phi\|_{L^2(\widehat\rho)} < \infty \, ,
\]
where the same argument from Proposition \ref{prop:ME_continuous} can be applied. 

To complete the proof, we simply observe that 
$|\phi \nabla f(x,z_i)| \in C^\infty_{\pol}(\R^d)$, 
and the bound can be obtained from the previous case. 

\end{proof}

\subsection{Proof of Lemma \cref{lm:bound_l2_norm}}
\label{subsec:bound_l2_norm}
We begin by restating the Lemma for easier reference.

\begin{lemma} [Sufficient to Bound $L^2$ norms]
A sufficient condition for uniform stability of $\pi^N_\mathbf{z}$ is 
if for each $k = 1, \ldots, N$, there exist a constant $C_{\mu_k}>0$, 
independent of data and $n$, such that
\begin{align} \label{eq:L2_norm}
\| \mu_{k,\mathbf{z}} q_{z_i} - 
    \mu_{k,\overline{\mathbf{z}}} q_{\overline{z_i}} \|_{
        L^2(\rho_{\mathbf{z}^{(i)}})} 
    := \left[ \int_{\mathbb{R}^d} (\mu_{k,\mathbf{z}} q_{z_i}
    - \mu_{k,\overline{\mathbf{z}}} q_{\overline{z_i}})^2
    d \rho_{\mathbf{z}^{(i)}}
    \right]^{1/2} 
    \leq \frac{ C_{\mu_k} }{n}.
\end{align}
\end{lemma}

\begin{proof} We start by rewriting the definition of $\pi^N_\mathbf{z}$
using the new definitions
\[ \pi^N_\mathbf{z} = \rho_\mathbf{z} \sum_{k=0}^N \eta^k \mu_{k, \mathbf{z}}
    = \rho_{\mathbf{z}^{(i)}} \sum_{k=0}^N \eta^k \mu_{k, \mathbf{z}} q_{z_i}.
\]
Using this fact, we can apply the triangle inequality 
on the uniform stability condition to get 
\begin{align*}
    \sup_{z \in \mathcal{Z}}
    \left| \int_{\mathbb{R}^d} f(x,z) 
        (\pi^N_\mathbf{z} - \pi^N_{\overline{\mathbf{z}}}) dx
    \right|
    &\leq \sup_{z \in \mathcal{Z}}
    \sum_{k=0}^N \eta^k
    \left| \int_{\mathbb{R}^d} f(x,z) \rho_{\mathbf{z}^{(i)}}
        (\mu_{k,\mathbf{z}} q_{z_i} - 
        \mu_{k,\overline{\mathbf{z}}} q_{\overline{z_i}}) dx
    \right| \\
    &\leq \sup_{z \in \mathcal{Z}}
    \sum_{k=0}^N \eta^k \| f(x,z) \|_{L^2( \rho_{\mathbf{z}^{(i)}} )}
        \| \mu_{k,\mathbf{z}} q_{z_i} - 
        \mu_{k,\overline{\mathbf{z}}} q_{\overline{z_i}} \|_{
            L^2( \rho_{\mathbf{z}^{(i)}} )},
\end{align*}
where the last line follows from the Cauchy-Schwarz inequality. Since $f(x,z)$ is has a bound independent of $z$ from Lemma \ref{lm:appendix_moment_bound}, 
we can define a bounded constant 
$C_1 := \sup_{z \in \mathcal{Z}} \| f(x,z) \|_{L^2( \rho_{\mathbf{z}^{(i)}} )}$. 
Using the premise of the Lemma statement, we have 
\[ \sup_{z \in \mathcal{Z}}
    \left| \int_{\mathbb{R}^d} f(x,z) 
        (\pi^N_\mathbf{z} - \pi^N_{\overline{\mathbf{z}}}) dx
    \right|
    \leq C_1 \sum_{k=0}^N \eta^k \frac{C_{\mu_k}}{n}
    \leq \frac{\widetilde C}{n(1-\eta)}, 
\]
for some new constant $\widetilde C$, hence proving uniform stability. 

\end{proof}

\subsection{Proof of Lemma \ref{lm:energy_est1}}
\label{subsec:energy_est1}

We will once again restate the Lemma. 

\begin{lemma}[Zeroth Order Energy Estimate]
For all $k \geq 0$, assume there exists a non-negative constant $C_{G_k}$ such that 
\[ \norm{G_{k, \mathbf{z}} q_{z_i} - 
    G_{k, \overline{\mathbf{z}}} q_{\overline{z_i}} }_{
            L^2\left( \rho_{\mathbf{z}^{(i)}} \right)} \leq \frac{C_{G_k}}{n}.
 \]
Then there exists another non-negative constant $C_{\mu_k}$ such that 
\[
\norm{ \mu_{k,\mathbf{z}} q_{z_i} - 
    \mu_{k,\overline{\mathbf{z}}} q_{\overline{z_i}}
    }_{
            L^2\left( \rho_{\mathbf{z}^{(i)}} \right)} \leq \frac{C_{\mu_k}}{n}.
\]
\end{lemma}

\begin{proof} %

We first recall existence and uniqueness results for Poisson equations from Lemma \ref{prop:poisson_exist_unique}.
\begin{align} 
&L_\mathbf{z} \mu_{k, \mathbf{z}} 
    = G_{k, \mathbf{z}}, 
    \label{Eq:appendix_poisson_z}
    \\
   & L_{\overline{\mathbf{z}}} \mu_{k, \overline{\mathbf{z}}} = G_{k, \overline{\mathbf{z}}}. 
    \label{Eq:appendix_poisson_zbar}
\end{align}
Then for all $g \in C^\infty_{\pol}(\mathbb{R}^d)$, we can write the weak formulation associated to Equation \eqref{Eq:appendix_poisson_z}

\[ \int_{\mathbb{R}^d} - g L_\mathbf{z} 
    \mu_{k, \mathbf{z}} d \rho_\mathbf{z}
    = \int_{\mathbb{R}^d} - g G_{k,\mathbf{z}}
        d \rho_\mathbf{z}.
\]
Applying the Green's Theorem to the left hand side term gives
\[ \int_{\mathbb{R}^d} - g L_\mathbf{z} 
    \mu_{k, \mathbf{z}} d \rho_\mathbf{z}
    = \frac{1}{\beta} \int_{\mathbb{R}^d}
        \langle \nabla g, 
            \nabla \mu_{k, \mathbf{z}} \rangle
        d \rho_\mathbf{z}
    = \frac{1}{\beta} \int_{\mathbb{R}^d}
        \langle \nabla g, 
            \nabla \mu_{k, \mathbf{z}} q_{z_i} \rangle
        d \rho_{\mathbf{z}^{(i)}}.
\]
Since $\nabla \mu_{k, \mathbf{z}} q_{z_i}
 = \nabla (\mu_{k, \mathbf{z}} q_{z_i})
    - \mu_{k, \mathbf{z}} \nabla q_{z_i}$, 
we can write 
\begin{equation}\label{var_form_z}
\frac{1}{\beta} \int_{\mathbb{R}^d}
    \langle \nabla g, 
        \nabla (\mu_{k, \mathbf{z}} q_{z_i})
    \rangle d\rho_{\mathbf{z}^{(i)}}
    = \int_{\mathbb{R}^d} - g G_{k,\mathbf{z}}
        q_{z_i} d \rho_{\mathbf{z}^{(i)}}
    + \frac{1}{\beta} \int_{\mathbb{R}^d}
    \langle \nabla g, 
        \mu_{k, \mathbf{z}} \nabla q_{z_i}
    \rangle d\rho_{\mathbf{z}^{(i)}}.
\end{equation}
We proceed similarly for Equation \eqref{Eq:appendix_poisson_zbar} and write
\begin{equation}\label{var_form_zbar}
\frac{1}{\beta} \int_{\mathbb{R}^d}
    \langle \nabla g, 
        \nabla (\mu_{k, \overline{\mathbf{z}}} q_{\overline{z_i}})
    \rangle d\rho_{\mathbf{z}^{(i)}}
    = \int_{\mathbb{R}^d} - g G_{k,\overline{\mathbf{z}}}
        q_{\overline{z_i}} d \rho_{\mathbf{z}^{(i)}}
    + \frac{1}{\beta} \int_{\mathbb{R}^d}
    \langle \nabla g, 
        \mu_{k,\overline{\mathbf{z}}} \nabla q_{\overline{z_i}}
    \rangle d\rho_{\mathbf{z}^{(i)}}.
\end{equation}
We now take the difference of the two integral formulations \ref{var_form_z} and \ref{var_form_zbar}. 
Applying the Cauchy-Schwarz inequality in $L^2\left( \rho_{\mathbf{z}^{(i)}}\right)$, we get 
\begin{equation}
\label{eq:stability_energy_est}
\begin{aligned}
    \frac{1}{\beta}& \int_{\mathbb{R}^d}
    \langle \nabla g, 
        \nabla (\mu_{k, \mathbf{z}} q_{z_i}
    - \mu_{k, \overline{\mathbf{z}}}
        q_{\overline{z_i}} )
    \rangle d\rho_{\mathbf{z}^{(i)}} \\
    &=  \int_{\mathbb{R}^d} - g ( G_{k,\mathbf{z}}
        q_{z_i} - 
        G_{k, \overline{\mathbf{z}}} 
        q_{\overline{z_i}} ) 
        d \rho_{\mathbf{z}^{(i)}}
    + \frac{1}{\beta} \int_{\mathbb{R}^d}
    \langle \nabla g, 
        \mu_{k, \mathbf{z}} \nabla q_{z_i}
    - \mu_{k, \overline{\mathbf{z}}} 
        \nabla q_{\overline{z_i}}
    \rangle d\rho_{\mathbf{z}^{(i)}} \\
    &\leq  \norm{g}_{L^2\left( \rho_{\mathbf{z}^{(i)}}\right)}\norm{ G_{k,\mathbf{z}}
        q_{z_i} - G_{k, \overline{\mathbf{z}}} q_{\overline{z_i}} }_{L^2\left( \rho_{\mathbf{z}^{(i)}}\right)}  
        + \frac{1}{\beta}\norm{\nabla g}_{L^2\left( \rho_{\mathbf{z}^{(i)}}\right)} \norm{\mu_{k, \mathbf{z}} \nabla q_{z_i}
    - \mu_{k, \overline{\mathbf{z}}} 
        \nabla q_{\overline{z_i}}}_{L^2\left( \rho_{\mathbf{z}^{(i)}}\right)} \, .
\end{aligned}
\end{equation}
Taking $g = \mu_{k, \mathbf{z}} q_{z_i}
    - \mu_{k, \overline{\mathbf{z}}}
        q_{\overline{z_i}}$ and using Lemma \ref{lemma:spectral_gap}, on the spectral gap for  $\rho_{\mathbf{z}^{(i)}}$, we obtain that
\begin{equation}
\begin{aligned}
    &\quad\, \frac{1}{\beta} \norm{\nabla \left(\mu_{k, \mathbf{z}} q_{z_i}
    - \mu_{k, \overline{\mathbf{z}}}
        q_{\overline{z_i}} \right) }_{L^2\left( \rho_{\mathbf{z}^{(i)}}\right)} \\
    &\leq \frac{1}{\lambda'}\norm{ G_{k,\mathbf{z}}
        q_{z_i} - G_{k, \overline{\mathbf{z}}} q_{\overline{z_i}} }_{L^2\left( \rho_{\mathbf{z}^{(i)}}\right)}  
        + \frac{1}{\beta} \norm{\mu_{k, \mathbf{z}} \nabla q_{z_i}
    - \mu_{k, \overline{\mathbf{z}}} 
        \nabla q_{\overline{z_i}}}_{L^2\left( \rho_{\mathbf{z}^{(i)}}\right)} \, .
\end{aligned}
\end{equation}
From the assumption in this Lemma and from the equality 
\[ 
\nabla q_{z_i} = \frac{-\beta}{n} \nabla f(x, z_i) q_{z_i},
\]
we conclude that
\begin{equation}
\begin{aligned}
   \norm{\nabla \left(\mu_{k, \mathbf{z}} q_{z_i}
    - \mu_{k, \overline{\mathbf{z}}}
        q_{\overline{z_i}} \right) }_{L^2\left( \rho_{\mathbf{z}^{(i)}}\right)} 
    \leq  \frac{\beta}{\lambda'} \frac{C_{G_k}}{n}  
        +  \frac{1}{n}\norm{h}_{L^2\left( \rho_{\mathbf{z}^{(i)}}\right)},
\end{aligned}
\end{equation}
where $h = \mu_{k, \mathbf{z}} \nabla f(x,z_i) q_{z_i} - \mu_{k, \overline{\mathbf{z}}} \nabla f(x, \overline{z_i}) q_{\overline{z_i}}$, 
and we can use Lemma \ref{lm:appendix_moment_bound} 
to bound the norm of $h$. 
We conclude the proof using the Poincar\'e inequality.

\end{proof}

\subsection{A Couple of Corollaries}

\begin{corollary}
[Generalized Zeroth Order Energy Estimate]
\label{cor:gen_energy_est}
If $u_\mathbf{z}, u_{\overline{\mathbf{z}}}, 
G_\mathbf{z}, G_{\overline{\mathbf{z}}}$ are known functions 
that satisfy the pair of PDEs
\[ \begin{cases}
    L_\mathbf{z} u_\mathbf{z} = G_\mathbf{z}, \\
    L_{\overline{\mathbf{z}}} u_{\overline{\mathbf{z}}} = G_{\overline{\mathbf{z}}},
\end{cases}
\]
and there exist constants $C_{\overline{u}} > 0, C_G > 0$ 
independent of data and $n$, such that 
\begin{align*}
    \left| \int_{\mathbb{R}^d} u_\mathbf{z} d\rho_\mathbf{z}
    - \int_{\mathbb{R}^d} u_{\overline{\mathbf{z}}} d\rho_{\overline{\mathbf{z}}} 
    \right|
    &\leq \frac{C_{\overline{u}}}{n} \, ,  \\
    \| G_{\mathbf{z}} q_{z_i} - 
        G_{\overline{\mathbf{z}}} q_{\overline{z_i}} \|_{
            L^2( \rho_{\mathbf{z}^{(i)}} )} 
    &\leq \frac{C_{G}}{n} \, . 
\end{align*}

Then there exists a new constant $C_u > 0$ 
such that 
\[ \| u_{\mathbf{z}} q_{z_i} - 
    u_{\overline{\mathbf{z}}} q_{\overline{z_i}} \|_{
            L^2( \rho_{\mathbf{z}^{(i)}} )}
    \leq \frac{C_{u}}{n}
    \, .
\]

\end{corollary}

\begin{proof}

Observe the only condition from Lemma \ref{lm:energy_est1}
that we fail to satisfy is that 
$\overline u := \int_{\mathbb{R}^d} u_\mathbf{z} d\rho_\mathbf{z} \neq 0$.
As a result, we need to use the Poincar\'{e} inequality 
on the centered function, 
i.e. for all $g \in C^\infty_{\pol}(\mathbb{R}^d)$ we have 
\[ \| g \|^2_{ L^2( \rho_{ \mathbf{z}^{(i)} } ) }
    = \| g - \overline g \|^2_{ L^2( \rho_{ \mathbf{z}^{(i)} } ) }
    + (\overline g)^2
    \leq \frac{1}{\lambda'}
        \| \nabla g \|^2_{ L^2( \rho_{ \mathbf{z}^{(i)} } ) }
        + (\overline g)^2 \, .
\]

Letting $g = u_{\mathbf{z}} q_{z_i} - 
u_{\overline{\mathbf{z}}} q_{\overline{z_i}}$, 
it is then sufficient to bound $(\overline g)^2$. 
To complete the proof, we can rewrite $g$ 
to match the assumption 
\[ (\overline g)^2 
    = \left( \int_{\mathbb{R}^d} u_{\mathbf{z}} q_{z_i} - 
    u_{\overline{\mathbf{z}}} q_{\overline{z_i}} d \rho_{\mathbf{z}^{(i)}}
    \right)^2
    \leq \left( \frac{ C_{\overline u} }{ n } \right)^2 \, ,
\]
which gives us the desired bound of 
\[ \| u_{\mathbf{z}} q_{z_i} - 
    u_{\overline{\mathbf{z}}} q_{\overline{z_i}} 
    \|^2_{L^2( \rho_{\mathbf{z}^{(i)}} )}
    \leq \frac{ C_{u - \overline u}^2 + C_{\overline{u}}^2}{n^2}
    \, ,
\]
where $C_{u - \overline u}$ is the same as $C_{\mu_k}$ 
from Lemma \ref{lm:energy_est1}.

\end{proof}

\begin{corollary}[First Order Energy Estimate]
\label{cor:energy_est1_improve}
If there exist a constant $C_{G_k}>0$ such that 
\[ \|G_{k, \mathbf{z}} q_{z_i} - 
    G_{k, \overline{\mathbf{z}}} q_{\overline{z_i}} \|_{
            L^2( \rho_{\mathbf{z}^{(i)}} )} \leq \frac{C_{G_k}}{n}, 
\]
then there exist new constant $C_{\nabla \mu_k}>0$ such that 
\[ \| \nabla \mu_{k,\mathbf{z}} q_{z_i} - 
    \nabla \mu_{k,\overline{\mathbf{z}}} q_{\overline{z_i}} \|_{
            L^2( \rho_{\mathbf{z}^{(i)}} )}
    \leq \frac{C_{\nabla \mu_k}}{n}.
\]
\end{corollary}

\begin{proof}

We start by using the product rule 
$\nabla \mu_{k,\mathbf{z}} q_{z_i}
= \nabla ( \mu_{k,\mathbf{z}} q_{z_i} )
- \mu_{k,\mathbf{z}} \nabla q_{z_i}$
to write 
\begin{equation}
\begin{aligned}
    &\quad \| \nabla \mu_{k,\mathbf{z}} q_{z_i} - 
    \nabla \mu_{k,\overline{\mathbf{z}}} q_{\overline{z_i}} \|_{
            L^2( \rho_{\mathbf{z}^{(i)}} )} \\
    &\leq
    \| \nabla ( \mu_{k,\mathbf{z}} q_{z_i} - 
        \mu_{k,\overline{\mathbf{z}}} q_{\overline{z_i}} ) \|_{
            L^2( \rho_{\mathbf{z}^{(i)}} )}
    + 
    \| \mu_{k,\mathbf{z}} \nabla q_{z_i} - 
     \mu_{k,\overline{\mathbf{z}}} \nabla q_{\overline{z_i}} \|_{
            L^2( \rho_{\mathbf{z}^{(i)}} )} \\
    &=: T_1 + T_2.
\end{aligned}
\end{equation}

Now observe that in the proof of Lemma \ref{lm:energy_est1}, 
we already proved a bound for 
$T_1 = \| \nabla g \|_{ L^2( \rho_{\mathbf{z}^{(i)}} ) }$, 
which we can write as 
\[ T_1 \leq \frac{ (\lambda')^{1/2} C_{\mu_k} }{n}.
\]

To control $T_2$, we simply need to compute 
$\nabla q_{z_i} = \frac{-\beta}{n} \nabla f(x, z_i) q_{z_i} $ to get
\[ T_2 \leq \frac{\beta}{n} \| \mu_{k,\mathbf{z}} \nabla f(x, z_i) q_{z_i} - 
     \mu_{k,\overline{\mathbf{z}}} \nabla f(x, \overline{z_i}) 
        q_{\overline{z_i}} \|_{
            L^2( \rho_{\mathbf{z}^{(i)}} )}.
\]

Denoting $h = |\mu_{k,\mathbf{z}} \nabla f(x, z_i) q_{z_i} - 
     \mu_{k,\overline{\mathbf{z}}} \nabla f(x, \overline{z_i}) 
        q_{\overline{z_i}}|$, 
we can put the two bounds together and get 
the desired result  
\[ \| \nabla \mu_{k,\mathbf{z}} q_{z_i} - 
    \nabla \mu_{k,\overline{\mathbf{z}}} q_{\overline{z_i}} \|_{
            L^2( \rho_{\mathbf{z}^{(i)}} )}
    \leq \frac{ (\lambda')^{1/2} C_{\mu_k} + 
        \beta \| h \|_{L^2( \rho_{\mathbf{z}^{(i)}} )} }{ n }, 
\]
where we can provide a bound on $h$ using 
Lemma \ref{lm:appendix_moment_bound}. 

\end{proof}

\subsection{Energy Estimate with Weighted Norm}

\begin{lemma}[Energy Estimate with $C^\infty_{\pol}(\mathbb{R}^d)$ Coefficient]
\label{lm:energy_est1_pol}
If for all $\phi \in C^\infty_{\pol}(\mathbb{R}^d)$, 
there exists a constant $C_{\phi G_k}>0$ such that 
\[ \| \phi G_{k, \mathbf{z}} q_{z_i} - 
    \phi G_{k, \overline{\mathbf{z}}} q_{\overline{z_i}} \|_{
            L^2( \rho_{\mathbf{z}^{(i)}} )} \leq \frac{C_{\phi G_k}}{n}, 
\]
then there exists a new constant $C_{\phi\mu_k}>0$ 
depending on $\phi$ such that 
\[ \| \phi \mu_{k,\mathbf{z}} q_{z_i} - 
    \phi \mu_{k,\overline{\mathbf{z}}} q_{\overline{z_i}} \|_{
            L^2( \rho_{\mathbf{z}^{(i)}} )}
    \leq \frac{C_{\phi\mu_k}}{n}.
\]
\end{lemma}

\begin{proof}

\noindent
\textbf{Step 1. Reduction to a Recursive Argument on the Polynomial Degree}

We will start by making the observation that
since $\phi \in C^\infty_{\pol}(\mathbb{R}^d)$, 
there exist constants $C_\phi > 0, k_\phi \in \mathbb{N}$
such that 
\[ |\phi(x)| \leq C_\phi \left( 1 + \sum_{j=1}^d |x_j|^{k_\phi}
    \right) =: \widehat \phi(x) \,.
\]
Then observe it is sufficient to bound 
the case where $\phi(x)$ is exactly 
a polynomial of the $\widehat \phi(x)$ type, i.e. 
\[ \| \phi \mu_{k,\mathbf{z}} q_{z_i} - 
    \phi \mu_{k,\overline{\mathbf{z}}} q_{\overline{z_i}} \|_{
            L^2( \rho_{\mathbf{z}^{(i)}} )}
    \leq 
    \| \widehat \phi
        ( \mu_{k,\mathbf{z}} q_{z_i} - 
                \mu_{k,\overline{\mathbf{z}}} q_{\overline{z_i}} )
    \|_{
            L^2( \rho_{\mathbf{z}^{(i)}} )}.
\]

The choice for the form of $\widehat \phi$ has 
a couple of advantages.
Observe that for any first order 
differential operator $\de_j$, 
we have that 
$|\de_j \widehat \phi| = C_\phi k_\phi |x_j|^{k_\phi - 1}$ 
only depends on a single coordinate. 
Furthermore, $|\de_j^{k_\phi} \widehat \phi| = C_\phi (k_\phi !)$
is also a constant. 

This implies it is sufficient to prove a bound of the form 
\begin{equation}
\label{eq:energy_est_pol_induction}
\begin{aligned}
    &\quad \| \nabla ( \widehat \phi
            ( \mu_{k,\mathbf{z}} q_{z_i} - 
                    \mu_{k,\overline{\mathbf{z}}} q_{\overline{z_i}} ) )
        \|^2_{L^2( \rho_{\mathbf{z}^{(i)}} )} \\
    &\leq \frac{\widehat C_{1,\phi}}{n^2}
    + \widehat C_{2,\phi} \sum_{j=1}^d \| \nabla (\de_j \widehat \phi
            ( \mu_{k,\mathbf{z}} q_{z_i} - 
                    \mu_{k,\overline{\mathbf{z}}} q_{\overline{z_i}} ) )
        \|^2_{L^2( \rho_{\mathbf{z}^{(i)}} )} \\
    &\quad + \widehat C_{3,\phi} \sum_{j=1}^d
        \| \nabla ( \de_j^2 \widehat \phi
            ( \mu_{k,\mathbf{z}} q_{z_i} - 
                    \mu_{k,\overline{\mathbf{z}}} q_{\overline{z_i}} ) )
        \|^2_{L^2( \rho_{\mathbf{z}^{(i)}} )}.
\end{aligned}
\end{equation}

Observe that since $\de_j \widehat \phi, \de_j^2 \widehat \phi$ 
are lower degree polynomials, 
we can recursively apply the above bound to 
all terms involving $\widehat \phi$ 
until they vanish,  
hence recovering the desired result using Lemma \ref{lm:energy_est1}. 
In particular, when $\widehat \phi$ is a degree-$1$ polynomial, the result follows directly from \cref{eq:energy_est_pol_induction}. 
From this point onwards, we will assume without loss of generality 
$\phi$ is a polynomial of the form $\widehat \phi$
defined previously.

\vspace{0.2cm}
\noindent
\textbf{Step 2. Energy Estimate}

To prove the desired bound in \cref{eq:energy_est_pol_induction}, 
we will follow the same proof structure as Lemma \ref{lm:energy_est1}, 
and write down 
the weak form of the Poisson equation in terms of 
the quantity we want to bound. 
To start, we will first compute 
\begin{align*}
    L_\mathbf{z} ( \phi \mu_{k,\mathbf{z}} ) 
    &= - \langle \nabla F_\mathbf{z},
        \nabla( \phi \mu_{k, \mathbf{z}} )
        \rangle 
        + \frac{1}{\beta} (\Delta \phi \mu_{k,\mathbf{z}}
            + 2 \langle \nabla \phi, 
                \nabla \mu_{k, \mathbf{z}} \rangle 
            + \phi \Delta \mu_{k, \mathbf{z}}) \\
    &= \phi L_\mathbf{z} \mu_{k, \mathbf{z}}
        + \mu_{k, \mathbf{z}} L_\mathbf{z} \phi
        + \frac{2}{\beta} \langle \nabla \phi, 
            \nabla \mu_{k, \mathbf{z}} \rangle \\
    &= \phi G_{k, \mathbf{z}} + \mu_{k, \mathbf{z}} L_\mathbf{z} \phi
        + \frac{2}{\beta} \langle \nabla \phi, 
            \nabla \mu_{k, \mathbf{z}} \rangle
\end{align*}

Then we once again write the equation in integral form 
for any test function $g \in C^\infty_{\pol}(\mathbb{R}^d)$
\begin{align*}
    &\quad \frac{1}{\beta} \int_{\mathbb{R}^d}
        \langle \nabla g,
            \nabla (\phi \mu_{k, \mathbf{z}} )
        \rangle
        d \rho_{\mathbf{z} } \\
    &= \int_{\mathbb{R}^d}
        - g \phi G_{k, \mathbf{z}} 
        - g \mu_{k, \mathbf{z}} L_\mathbf{z} \phi
        - \frac{2}{\beta} g \langle \nabla \phi, 
            \nabla \mu_{k, \mathbf{z}} \rangle
        d \rho_{\mathbf{z}} \\
    &= \int_{\mathbb{R}^d}
        - g \phi G_{k, \mathbf{z}} 
        + \frac{1}{\beta} \langle \nabla (g \mu_{k, \mathbf{z}}), \nabla \phi
            \rangle
        - \frac{2}{\beta} g \langle \nabla \phi, 
            \nabla \mu_{k, \mathbf{z}} \rangle
        d \rho_{\mathbf{z}} \\
    &= \int_{\mathbb{R}^d}
        - g \phi G_{k, \mathbf{z}} 
        + \frac{1}{\beta} \langle 
            \mu_{k, \mathbf{z}} \nabla g, 
            \nabla \phi
            \rangle
        - \frac{1}{\beta} \langle 
            g \nabla \mu_{k, \mathbf{z}}, 
            \nabla \phi
            \rangle
        d \rho_{\mathbf{z}},
\end{align*}
where we used Green's Theorem in the third line above. 

In the same way as Lemma \ref{lm:energy_est1}, 
we will use the product-rule to write 
$\nabla (\phi \mu_{k, \mathbf{z}}) q_{z_i}
= \nabla (\phi \mu_{k, \mathbf{z}} q_{z_i} )
- \phi \mu_{k, \mathbf{z}} \nabla q_{z_i}$, 
and get the equation in the following form
\begin{align*}
    &\quad \frac{1}{\beta} \int_{\mathbb{R}^d}
        \langle \nabla g,
            \nabla (\phi \mu_{k, \mathbf{z}} q_{z_i} )
        \rangle
        d \rho_{\mathbf{z}^{(i)} } \\
    &= \int_{\mathbb{R}^d}
        - g \phi G_{k, \mathbf{z}} q_{z_i}
        + \frac{1}{\beta} \langle 
            \mu_{k, \mathbf{z}} q_{z_i} \nabla g, 
            \nabla \phi
            \rangle
        - \frac{1}{\beta} \langle 
            g \nabla \mu_{k, \mathbf{z}} q_{z_i}, 
            \nabla \phi
            \rangle
        + \frac{1}{\beta} \langle 
            \nabla g, 
            \phi \mu_{k,\mathbf{z}} \nabla q_{z_i}
            \rangle
        d \rho_{\mathbf{z}^{(i)}}.
\end{align*}

We can choose $g = \phi (\mu_{k, \mathbf{z}} q_{z_i}
- \mu_{k, \overline{\mathbf{z}}} q_{\overline{z_i}} )$, 
and taking the difference with the $\overline{\mathbf{z}}$ equation 
to get 
\begin{equation}
\label{eq:energy_est_pol_terms}
\begin{aligned}
    \frac{1}{\beta} \int_{\mathbb{R}^d}
        \langle \nabla g, \nabla g \rangle
        d \rho_{\mathbf{z}^{(i)}} 
    &= \int_{\mathbb{R}^d} - g \phi (G_{k, \mathbf{z}} q_{z_i}
        - G_{k, \overline{\mathbf{z}}} q_{\overline{z_i}}) 
        d \rho_{\mathbf{z}^{(i)}} \\
    &\quad + \frac{1}{\beta} 
        \int_{\mathbb{R}^d} \langle 
            \nabla g (\mu_{k, \mathbf{z}} q_{z_i} - 
                \mu_{k, \overline{\mathbf{z}}} q_{\overline{z_i}} ), 
            \nabla \phi
            \rangle d \rho_{\mathbf{z}^{(i)}} \\
    &\quad - \frac{1}{\beta} 
        \int_{\mathbb{R}^d} \langle 
            g (\nabla \mu_{k, \mathbf{z}} q_{z_i} - 
                \nabla \mu_{k, \overline{\mathbf{z}}} q_{\overline{z_i}}), 
            \nabla \phi
            \rangle d \rho_{\mathbf{z}^{(i)}} \\
    &\quad + \frac{1}{\beta} 
        \int_{\mathbb{R}^d} \langle 
            \nabla g, 
            \phi \mu_{k,\mathbf{z}} \nabla q_{z_i}
            - \phi \mu_{k,\overline{\mathbf{z}}} \nabla q_{\overline{z_i}}
            \rangle
    d \rho_{\mathbf{z}^{(i)}} \\
    &=: T_1 + T_2 + T_3 + T_4.
\end{aligned} 
\end{equation}

\vspace{0.2cm}
\noindent
\textbf{Step 3. Controlling Terms $T_1, T_2, T_3, T_4$}

Since $\overline g := \int_{\mathbb{R}^d} g d \rho_{\mathbf{z}^{(i)}} \neq 0$, 
we need to use a modified Poincar\'{e} inequality, namely 
\begin{equation}
\label{eq:poincare_pol}
    \|g\|^2_{L^2( \rho_{\mathbf{z}^{(i)}} )} 
    \leq \frac{1}{\lambda'} \| \nabla g \|^2_{L^2( \rho_{\mathbf{z}^{(i)}} )}
        + (\overline g)^2.
\end{equation}

Observe that we can apply the Cauchy-Schwarz inequality 
to $\overline g$ and get 
\begin{equation}
\label{eq:g_bar_bound_pol}
\begin{aligned}
    |\overline g| 
    &\leq \int_{\mathbb{R}^d} |\phi| | \mu_{k, \mathbf{z}} q_{z_i} - 
                \mu_{k, \overline{\mathbf{z}}} q_{\overline{z_i}} |
                d \rho_{\mathbf{z}^{(i)}}  \\
    &\leq \|\phi\|_{L^2( \rho_{\mathbf{z}^{(i)}} )}
        \| \mu_{k, \mathbf{z}} q_{z_i} - 
                \mu_{k, \overline{\mathbf{z}}} q_{\overline{z_i}} \|_{
                    L^2( \rho_{\mathbf{z}^{(i)}} )} \\
    &\leq \frac{C_{\mu_k} }{n} \|\phi\|_{L^2( \rho_{\mathbf{z}^{(i)}} )},
\end{aligned}
\end{equation}
where we used the result of Lemma \ref{lm:energy_est1}.
Here we note the norm of $\phi$ can be bounded using Lemma \ref{lm:appendix_moment_bound}.

Returning to \cref{eq:energy_est_pol_terms}, 
we will control $T_1$ using Young's inequality and the assumption, i.e.
\begin{equation}
\label{eq:energy_est_pol_t1}
\begin{aligned}
    T_1 
    &\leq 
        \frac{1}{2 c_{y_1}} \|g\|^2_{L^2( \rho_{\mathbf{z}^{(i)}} )}
        + \frac{c_{y_1}}{2} \| \phi (G_{k, \mathbf{z}} q_{z_i}
        - G_{k, \overline{\mathbf{z}}} q_{\overline{z_i}}) \|^2_{
            L^2( \rho_{\mathbf{z}^{(i)}} )} \\
    &\leq
        \frac{1}{2 c_{y_1} \lambda' } \| \nabla g \|^2_{
            L^2( \rho_{\mathbf{z}^{(i)}} )}
        + \frac{1}{2 c_{y_1}} \frac{C_{\mu_k}^2 }{n^2} 
            \| \phi \|^2_{L^2( \rho_{\mathbf{z}^{(i)}} )}
        + \frac{c_{y_1}}{2} \frac{C_{G_k}^2}{n^2}, 
\end{aligned}
\end{equation}
where we also used the modified Poincar\'{e} inequality \eqref{eq:poincare_pol} above 
with the $|\overline g|$ bound.

To control $T_2$, we will apply Cauchy-Schwarz and Young's inequalities 
to separate $\nabla g$, i.e. 
\begin{align*}
    T_2 
    &\leq 
        \frac{1}{\beta} \int_{\mathbb{R}^d} |\nabla g|
            | (\mu_{k, \mathbf{z}} q_{z_i} - 
                \mu_{k, \overline{\mathbf{z}}} q_{\overline{z_i}} ) \nabla \phi |
            d \rho_{\mathbf{z}^{(i)}} \\
    &\leq 
        \frac{1}{2 \beta c_{y_2}}
        \|\nabla g\|^2_{L^2( \rho_{\mathbf{z}^{(i)}} )}
        + \frac{c_{y_2}}{2 \beta} \| (\mu_{k, \mathbf{z}} q_{z_i} - 
                \mu_{k, \overline{\mathbf{z}}} q_{\overline{z_i}} ) \nabla \phi 
                \|^2_{L^2( \rho_{\mathbf{z}^{(i)}} )}.
\end{align*}

To convert to the desired form of \cref{eq:energy_est_pol_induction}, 
we observe 
\begin{align*}
    &\quad \| (\mu_{k, \mathbf{z}} q_{z_i} - 
                \mu_{k, \overline{\mathbf{z}}} q_{\overline{z_i}} ) \nabla \phi 
                \|^2_{L^2( \rho_{\mathbf{z}^{(i)}} )} \\
    &= \int_{ \mathbb{R}^d } \sum_{j=1}^d (\de_j \phi)^2 
        (\mu_{k, \mathbf{z}} q_{z_i} - 
                \mu_{k, \overline{\mathbf{z}}} q_{\overline{z_i}} )^2
                d \rho_{\mathbf{z}^{(i)}} \\
    &= \sum_{j=1}^d \| \de_j \phi  (\mu_{k, \mathbf{z}} q_{z_i} - 
                \mu_{k, \overline{\mathbf{z}}} q_{\overline{z_i}} ) 
                \|^2_{L^2( \rho_{\mathbf{z}^{(i)}} )} \\
    &\leq \sum_{j=1}^d \frac{1}{\lambda'} 
            \| \nabla (\de_j \phi  (\mu_{k, \mathbf{z}} q_{z_i} - 
                    \mu_{k, \overline{\mathbf{z}}} q_{\overline{z_i}} ) )
            \|^2_{L^2( \rho_{\mathbf{z}^{(i)}} )}
            + \| \de_j \phi \|^2_{L^2( \rho_{\mathbf{z}^{(i)}} )}
            \frac{C_{\mu_k}^2}{n^2} , 
\end{align*}
where we used the modified form of the Poincar\'{e} inequality \eqref{eq:poincare_pol}
and the bound on $|\overline g|$. 
Putting everything together, we have the following bound 
\begin{equation}
\label{eq:energy_est_pol_t2}
\begin{aligned}
    T_2 
    &\leq 
    \frac{1}{2 \beta c_{y_2}}
        \|\nabla g\|^2_{L^2( \rho_{\mathbf{z}^{(i)}} )}
        + \frac{c_{y_2}}{2 \beta} 
        \sum_{j=1}^d \frac{1}{\lambda'} 
            \| \nabla (\de_j \phi  (\mu_{k, \mathbf{z}} q_{z_i} - 
                    \mu_{k, \overline{\mathbf{z}}} q_{\overline{z_i}} ) )
            \|^2_{L^2( \rho_{\mathbf{z}^{(i)}} )} \\
    &\quad + \frac{c_{y_2}}{2 \beta} \sum_{j=1}^d
            \| \de_j \phi \|^2_{L^2( \rho_{\mathbf{z}^{(i)}} )}
            \frac{C_{\mu_k}^2}{n^2}.
\end{aligned}
\end{equation}

To control $T_3$, we will apply a similar approach as $T_2$. 
We will start with Young's inequality 
to isolate $g$ first 
\begin{align*}
    T_3 
    &\leq 
        \frac{1}{\beta} \int_{\mathbb{R}^d}
        \frac{1}{2 c_{y_3}} g^2 
        + \frac{c_{y_3}}{2} \langle \nabla \mu_{k, \mathbf{z}} q_{z_i} - 
                \nabla \mu_{k, \overline{\mathbf{z}}} q_{\overline{z_i}}, 
                \nabla \phi \rangle^2
        d \rho_{\mathbf{z}^{(i)}} \\
    &\leq \frac{1}{2 c_{y_3} \lambda' \beta} 
        \| \nabla g \|^2_{L^2( \rho_{\mathbf{z}^{(i)}} )}
        + \frac{c_{y_3}}{2 \beta} \| |\nabla \mu_{k, \mathbf{z}} q_{z_i} - 
                \nabla \mu_{k, \overline{\mathbf{z}}} q_{\overline{z_i}}| 
                \, |\nabla \phi| \|^2_{L^2( \rho_{\mathbf{z}^{(i)}} )}.
\end{align*}

Next we can rewrite in terms of a sum again 
\[ \| |\nabla \mu_{k, \mathbf{z}} q_{z_i} - 
                \nabla \mu_{k, \overline{\mathbf{z}}} q_{\overline{z_i}}| 
                \, |\nabla \phi| \|^2_{L^2( \rho_{\mathbf{z}^{(i)}} )}
    = \sum_{j=1}^d \| \de_j \phi
                (\nabla \mu_{k, \mathbf{z}} q_{z_i} - 
                \nabla \mu_{k, \overline{\mathbf{z}}} q_{\overline{z_i}}) 
                \|^2_{L^2( \rho_{\mathbf{z}^{(i)}} )}.
\]

Using the product rule, we can write 
\[ \nabla \mu_{k, \mathbf{z}} q_{z_i} \de_j \phi
    = \nabla (\mu_{k, \mathbf{z}} q_{z_i} \de_j \phi)
    - \mu_{k, \mathbf{z}} \nabla (q_{z_i} \de_j \phi),
\]
then we use the inequality $(a+b)^2 \leq 2 a^2 + 2 b^2$ twice to write
\begin{align*}
    &\quad \| \de_j \phi
        (\nabla \mu_{k, \mathbf{z}} q_{z_i} - 
        \nabla \mu_{k, \overline{\mathbf{z}}} q_{\overline{z_i}}) 
        \|^2_{L^2( \rho_{\mathbf{z}^{(i)}} )} \\
    &\leq 
        2 \| \nabla ( \de_j \phi \mu_{k, \mathbf{z}} q_{z_i} 
            - \de_j \phi \mu_{k, \overline{\mathbf{z}}} q_{\overline{z_i}}) 
                \|^2_{L^2( \rho_{\mathbf{z}^{(i)}} )}
        + 2 \| \mu_{k, \mathbf{z}} \nabla (q_{z_i} \de_j \phi)
            - \mu_{k, \overline{\mathbf{z}}} \nabla (q_{\overline{z_i}} \de_j \phi)
                \|^2_{L^2( \rho_{\mathbf{z}^{(i)}} )} \\
    &\leq 
        2 \| \nabla ( \de_j \phi \mu_{k, \mathbf{z}} q_{z_i} 
            - \de_j \phi \mu_{k, \overline{\mathbf{z}}} q_{\overline{z_i}}) 
                \|^2_{L^2( \rho_{\mathbf{z}^{(i)}} )}
        + 4 \| ( \mu_{k, \mathbf{z}} \nabla q_{z_i} 
                    - \mu_{k, \overline{\mathbf{z}}} \nabla q_{\overline{z_i}} ) 
                    \de_j \phi
                \|^2_{L^2( \rho_{\mathbf{z}^{(i)}} )} \\
    &\quad + 4 \| ( \mu_{k, \mathbf{z}} q_{z_i}
                - \mu_{k, \overline{\mathbf{z}}} q_{\overline{z_i}} ) 
                \nabla \de_j \phi
                \|^2_{L^2( \rho_{\mathbf{z}^{(i)}} )}.
\end{align*}

Recalling $\nabla q_{z_i} = q_{z_i} \frac{-\beta}{n} \nabla f(x, z_i)$, 
we have
\begin{align*}
    &\quad 
        \| ( \mu_{k, \mathbf{z}} \nabla q_{z_i} 
                    - \mu_{k, \overline{\mathbf{z}}} \nabla q_{\overline{z_i}} ) 
                    \de_j \phi
                \|^2_{L^2( \rho_{\mathbf{z}^{(i)}} )} \\
    &\leq \frac{\beta^2}{n^2}
        \| ( \mu_{k, \mathbf{z}} q_{z_i} \nabla f(x, z_i)
                    - \mu_{k, \overline{\mathbf{z}}} q_{\overline{z_i}} 
                    \nabla f(x, \overline{z_i})) 
                    \de_j \phi
                \|^2_{L^2( \rho_{\mathbf{z}^{(i)}} )} \\
    &=: \frac{\beta^2}{n^2}
        \| h_{1,j} \|^2_{L^2( \rho_{\mathbf{z}^{(i)}} )}.
\end{align*}

Since $\de_j \phi$ is only a function of $x_j$, 
we also have 
\begin{align*}
    &\quad \| ( \mu_{k, \mathbf{z}} q_{z_i}
                - \mu_{k, \overline{\mathbf{z}}} q_{\overline{z_i}} ) 
                \nabla \de_j \phi
                \|^2_{L^2( \rho_{\mathbf{z}^{(i)}} )} \\
    &= \| ( \mu_{k, \mathbf{z}} q_{z_i}
                - \mu_{k, \overline{\mathbf{z}}} q_{\overline{z_i}} ) 
                \de_j^2 \phi 
                \|^2_{L^2( \rho_{\mathbf{z}^{(i)}} )} \\
    &\leq \frac{1}{\lambda'} \| \nabla (( \mu_{k, \mathbf{z}} q_{z_i}
                    - \mu_{k, \overline{\mathbf{z}}} q_{\overline{z_i}} ) 
                    \de_j^2 \phi )
                \|^2_{L^2( \rho_{\mathbf{z}^{(i)}} )}
            + \| \de_j^2 \phi \|^2_{L^2( \rho_{\mathbf{z}^{(i)}} )}
            \frac{C_{\mu_k}^2}{n^2}
\end{align*}

Putting everything together, we have the following bound 
\begin{equation}
\label{eq:energy_est_pol_t3}
\begin{aligned}
    T_3 
    &\leq
        \frac{1}{2 c_{y_3} \lambda' \beta} 
        \| \nabla g \|^2_{L^2( \rho_{\mathbf{z}^{(i)}} )}
    + \frac{c_{y_3}}{2 \beta} \sum_{j=1}^d
        2 \| \nabla ( \de_j \phi \mu_{k, \mathbf{z}} q_{z_i} 
            - \de_j \phi \mu_{k, \overline{\mathbf{z}}} q_{\overline{z_i}}) 
                \|^2_{L^2( \rho_{\mathbf{z}^{(i)}} )} \\
    &\quad 
    + \frac{c_{y_3}}{2 \beta} \sum_{j=1}^d \left(
        4 \frac{\beta^2}{n^2}
        \| h_{1,j} \|^2_{L^2( \rho_{\mathbf{z}^{(i)}} )}
        + \frac{4}{\lambda'} \| \nabla (( \mu_{k, \mathbf{z}} q_{z_i}
                        - \mu_{k, \overline{\mathbf{z}}} q_{\overline{z_i}} ) 
                        \de_j^2 \phi)
                \|^2_{L^2( \rho_{\mathbf{z}^{(i)}} )} 
    + \frac{4 C_{\mu_k}^2}{n^2}
    \right),
\end{aligned}
\end{equation}
which satisfies the desired form in \cref{eq:energy_est_pol_induction}.

Lastly we have control over $T_4$ due to 
the computation $\nabla q_{z_i} = q_{z_i} \frac{-\beta}{n} \nabla f(x,z_i)$,
which leads to 
\begin{equation}
\label{eq:energy_est_pol_t4}
\begin{aligned}
    T_4 
    &\leq 
        \int_{\mathbb{R}^d} \frac{1}{\beta}
        |\nabla g| \frac{\beta}{n} 
        | \phi (\mu_{k,\mathbf{z}} q_{z_i} \nabla f(x,z_i)
            - \mu_{k,\overline{\mathbf{z}}} q_{\overline{z_i}} 
                \nabla f(x,\overline{z_i}))|
        d\rho_{\mathbf{z}^{(i)}} \\
    &\leq
        \frac{1}{2 c_{y_4}} \|\nabla g\|^2_{L^2( \rho_{\mathbf{z}^{(i)}} )}
        + \frac{c_{y_4}}{2n^2} 
            \|h_2\|^2_{L^2( \rho_{\mathbf{z}^{(i)}} )},
\end{aligned}
\end{equation}
where we define 
$h_2 := | \phi (\mu_{k,\mathbf{z}} q_{z_i} \nabla f(x,z_i)
            - \mu_{k,\overline{\mathbf{z}}} q_{\overline{z_i}} 
                \nabla f(x,\overline{z_i}))| $.

Putting together \cref{eq:energy_est_pol_t1,eq:energy_est_pol_t2,eq:energy_est_pol_t3,eq:energy_est_pol_t4}, we have a bound of the desired form 
\begin{equation}
\label{eq:energy_est_pol_induction_final}
\begin{aligned}
    &\quad \| \nabla ( \phi
            ( \mu_{k,\mathbf{z}} q_{z_i} - 
                    \mu_{k,\overline{\mathbf{z}}} q_{\overline{z_i}} ) )
        \|^2_{L^2( \rho_{\mathbf{z}^{(i)}} )} \\
    &\leq \frac{\widehat C_{1,\phi}}{n^2}
    + \widehat C_{2,\phi} \sum_{j=1}^d \| \nabla (\de_j \phi
            ( \mu_{k,\mathbf{z}} q_{z_i} - 
                    \mu_{k,\overline{\mathbf{z}}} q_{\overline{z_i}} ) )
        \|^2_{L^2( \rho_{\mathbf{z}^{(i)}} )} \\
    &\quad + \widehat C_{3,\phi} \sum_{j=1}^d
        \| \nabla ( \de_j^2 \phi
            ( \mu_{k,\mathbf{z}} q_{z_i} - 
                    \mu_{k,\overline{\mathbf{z}}} q_{\overline{z_i}} ) )
        \|^2_{L^2( \rho_{\mathbf{z}^{(i)}} )}.
\end{aligned}
\end{equation}
where we have the constants 
\begin{align*}
    \widehat C_{0,\phi} &= \frac{1}{\beta} - \frac{1}{2 c_{y_1} \lambda'}
        - \frac{1}{2\beta c_{y_2}}
        - \frac{1}{2 c_{y_3} \beta \lambda'}
        - \frac{1}{2 c_{y_4}} \, ,\\
    \widehat C_{1,\phi} &= \frac{1}{\widehat C_{0,\phi}} \bigg(
        \frac{ C_{\mu_k}^2 }{ 2 c_{y_1} } 
            \|\phi\|^2_{L^2( \rho_{\mathbf{z}^{(i)}} )}
        + \frac{c_{y_1} C_{G_k}^2 }{2} 
        + \frac{c_{y_2}}{2 \beta} 
            \| \nabla \phi \|^2_{L^2( \rho_{\mathbf{z}^{(i)}} )}
            C_{\mu_k}^2 
            \\
        &\quad 
        + 2 c_{y_3} \beta \sum_{j=1}^d
        \| h_{1,j} \|^2_{L^2( \rho_{\mathbf{z}^{(i)}} )}
        + \frac{2 c_{y_3} C_{\mu_k}^2 d }{\beta}
        + \frac{c_{y_4}}{ 2 } 
            \|h_2\|^2_{L^2( \rho_{\mathbf{z}^{(i)}} )}
        \bigg) \, ,
        \\
    \widehat C_{2,\phi} &= \frac{1}{\widehat C_{0,\phi}}
        \frac{c_{y_2}}{2 \beta \lambda'} \, , \\
    \widehat C_{3,\phi} &= \frac{1}{\widehat C_{0,\phi}}
        \frac{2 c_{y_3}}{\beta \lambda'} \, .
\end{align*}

Here we note the terms $h_{1,j}, h_2$ 
have bounded norms due 
Lemma \ref{lm:appendix_moment_bound}. 

\vspace{0.2cm}
\noindent
\textbf{Step 4. Completing the Proof}

We start by observing that 
$\| \nabla (\de_j \phi
            ( \mu_{k,\mathbf{z}} q_{z_i} - 
                    \mu_{k,\overline{\mathbf{z}}} q_{\overline{z_i}} ) )
        \|^2_{L^2( \rho_{\mathbf{z}^{(i)}} )}$
satisfy an inequality in the same form as 
\cref{eq:energy_est_pol_induction_final}, i.e. we can get a bound in the form of 
\begin{equation*}
\begin{aligned}
    &\quad \| \nabla ( \phi
            ( \mu_{k,\mathbf{z}} q_{z_i} - 
                    \mu_{k,\overline{\mathbf{z}}} q_{\overline{z_i}} ) )
        \|^2_{L^2( \rho_{\mathbf{z}^{(i)}} )} \\
    &\leq \frac{\overline C_{1,l}}{n^2}
    + \sum_{j=1}^d \overline C_{2,l,j} \| \nabla (\de_j^{l+1} \phi
            ( \mu_{k,\mathbf{z}} q_{z_i} - 
                    \mu_{k,\overline{\mathbf{z}}} q_{\overline{z_i}} ) )
        \|^2_{L^2( \rho_{\mathbf{z}^{(i)}} )} \\
    &\quad + \sum_{j=1}^d \overline C_{3,l,j}
        \| \nabla ( \de_j^{l+2} \phi
            ( \mu_{k,\mathbf{z}} q_{z_i} - 
                    \mu_{k,\overline{\mathbf{z}}} q_{\overline{z_i}} ) )
        \|^2_{L^2( \rho_{\mathbf{z}^{(i)}} )},
\end{aligned}
\end{equation*}
where we have the following recursion update for constants 
\begin{align*}
    \overline C_{1,l+1} &= 
        \overline C_{1,l} + 
        \sum_{j=1}^d \overline C_{2,l,j} \widehat C_{1, \de_j^{l+1} \phi}
        \, , \\
    \overline C_{2,l+1,j} &= 
        \overline C_{2,l,j} \widehat C_{2, \de_j^{l+1}} + \overline C_{3,l,j}
        \, , \\
    \overline C_{3,l+1,j} &= \overline C_{2,l,j} \widehat C_{3,\de_j^{l+1} \phi}
        \, .
\end{align*}

Finally, we obtain the desired bound 
from using the modified Poincar\'{e} \eqref{eq:poincare_pol} 
and Cauchy-Schwarz inequalities 
\begin{align*}
    &\quad \| \phi
            ( \mu_{k,\mathbf{z}} q_{z_i} - 
                    \mu_{k,\overline{\mathbf{z}}} q_{\overline{z_i}} ) 
        \|^2_{L^2( \rho_{\mathbf{z}^{(i)}} )} \\
    &\leq \frac{1}{\lambda'} \| \nabla ( \phi
            ( \mu_{k,\mathbf{z}} q_{z_i} - 
                    \mu_{k,\overline{\mathbf{z}}} q_{\overline{z_i}} ) )
        \|^2_{L^2( \rho_{\mathbf{z}^{(i)}} )} 
        +
        \|\phi\|^2_{L^2( \rho_{\mathbf{z}^{(i)}} )} 
        \| \mu_{k,\mathbf{z}} q_{z_i} - 
            \mu_{k,\overline{\mathbf{z}}} q_{\overline{z_i}} 
        \|^2_{L^2( \rho_{\mathbf{z}^{(i)}} )} 
        \\
    &\leq \frac{\overline C_{1,k_\phi}}{\lambda' n^2}
        + \left( 
            \frac{d \overline C_{2, k_\phi} C_\phi^2 (k_\phi !)^2 }{
            \lambda' }
            + \frac{ \|\phi\|^2_{L^2( \rho_{\mathbf{z}^{(i)}} )}  }{
                \lambda'
            }
        \right)
            \|\nabla ( \mu_{k,\mathbf{z}} q_{z_i} - 
                    \mu_{k,\overline{\mathbf{z}}} q_{\overline{z_i}} ) 
            \|^2_{L^2( \rho_{\mathbf{z}^{(i)}} )} \\
    &\leq \left(
        \frac{ \overline C_{1,k_\phi} }{\lambda'}
        + d \overline C_{2, k_\phi} C_\phi^2 (k_\phi !)^2 
        C_{\mu_k}^2
        + \|\phi\|^2_{L^2( \rho_{\mathbf{z}^{(i)}} )} 
        \right)
    \frac{1}{n^2} \, .
\end{align*}

\end{proof}

\subsection{Proof of Lemma \ref{lm:energy_est2}}
\label{subsec:energy_est2}

We will once again restate the Lemma for easier reference.

\begin{lemma}[Higher Order Energy Estimates]
Fix $k \in \mathbb{N}$. 
If for all $\phi \in C^\infty_{\pol}(\mathbb{R}^d)$, 
$\ell < k$, 
and multi-index $\alpha \in \mathbb{N}^d$, 
there exists a constant $C_{\phi \de_\alpha G_\ell} > 0$, 
such that we have
\[ \| \phi (\de_\alpha G_{\ell, \mathbf{z}}) q_{z_i} - 
    \phi (\de_\alpha G_{\ell, \overline{\mathbf{z}}}) q_{\overline{z_i}} \|_{
            L^2( \rho_{\mathbf{z}^{(i)}} )} 
    \leq \frac{C_{\phi \de_\alpha G_\ell}}{n}, 
\]
then for all $J \in \mathbb{N}$, 
and degree-$J$ differential operator $L$ 
with coefficients in $C^\infty_{\pol}(\mathbb{R}^d)$ (i.e. \\
$ L := \sum_{0 \leq |\alpha| \leq J} \phi^\alpha(x) \de_\alpha,
$
where $\phi^\alpha(x) \in C^\infty_{\pol}(\mathbb{R}^d)$ for each 
$|\alpha| \leq J$), 
there exist a constant $C_{L \mu_k} > 0$ such that 
\[ \| L \mu_{k,\mathbf{z}} q_{z_i} - 
    L \mu_{k,\overline{\mathbf{z}}} q_{\overline{z_i}} \|_{
            L^2( \rho_{\mathbf{z}^{(i)}} )} 
    \leq \frac{C_{L \mu_k}}{n}.
\]

In particular, the above inequality holds for 
$L = \phi' \de_{\alpha'} L^*_{\ell, \mathbf{z}}$
with $\phi' \in C^\infty_{\pol}(\R^d)$, 
$\alpha' \in \mathbb{N}^d$, 
and $J = 2\ell + 2 + |\alpha'|$, 
therefore there exists a constant $C_{\phi' \de_{\alpha'} G_k} > 0$ 
such that 
\[ \| \phi' (\de_{\alpha'} G_{k, \mathbf{z}}) q_{z_i} - 
    \phi' (\de_{\alpha'} G_{k, \overline{\mathbf{z}}}) q_{\overline{z_i}} \|_{
            L^2( \rho_{\mathbf{z}^{(i)}} )} 
    \leq \frac{C_{\phi' \de_{\alpha'} G_k}}{n}, 
\]
hence proving the induction step from $k-1$ to $k$. 
\end{lemma}

\begin{proof}
We start by using the triangle inequality on the definition of $L$ to get 
\[ \| L \mu_{k,\mathbf{z}} q_{z_i} - 
    L \mu_{k,\overline{\mathbf{z}}} q_{\overline{z_i}}
    \|_{
            L^2( \rho_{\mathbf{z}^{(i)}} )}
    \leq
    \sum_{ 0 \leq |\alpha| \leq J }
    \| \phi^\alpha ( \de_\alpha \mu_{k,\mathbf{z}} q_{z_i} - 
        \de_\alpha \mu_{k,\overline{\mathbf{z}}} q_{\overline{z_i}} )
    \|_{L^2( \rho_{\mathbf{z}^{(i)}} )} \,.
\]
Therefore it is sufficient to prove the case when
$L = \phi^\alpha \de_\alpha$, 
and sum the up the constants after.  

At this point, we will prove the desired result 
using induction on $J = |\alpha|$.

\vspace{0.2cm}
\noindent
\textbf{Step 1. Induction Case $J=1$}

The case when $\phi(x) = 1, J = 1$ follows from 
Corollary \ref{cor:energy_est1_improve}.
To prove this result for all $\phi \in C^\infty_{\pol}(\mathbb{R}^d)$, 
we start by computing 
\[ \de_\alpha (L_\mathbf{z} \mu_{k, \mathbf{z}} ) 
    = (\de_\alpha L_\mathbf{z}) \mu_{k, \mathbf{z}}
    + L_\mathbf{z} ( \de_\alpha \mu_{k, \mathbf{z}} ), 
\]
where we define $\de_\alpha L_\mathbf{z}$ 
as the derivative only on its coefficients, i.e. 
\[ \de_\alpha L_\mathbf{z} := 
    \sum_{\mathbf{k}} \de_\alpha \phi^\mathbf{k} \de_\mathbf{k} \, .
\]

Now we observe that $\de_{\alpha} \mu_{k,\mathbf{z}}$ 
satisfies a Poisson equation 
\begin{align*}
    L_\mathbf{z} \de_{\alpha} \mu_{k,\mathbf{z}} 
    &= \de_\alpha ( L_\mathbf{z} \mu_{k,\mathbf{z}} )
        - (\de_\alpha L_\mathbf{z}) \mu_{k,\mathbf{z}} \\
    &= \de_\alpha G_{k, \mathbf{z}}
        + \langle \nabla \de_\alpha F_\mathbf{z},
            \nabla \mu_{k, \mathbf{z}}
        \rangle \\
    &=: \widetilde G_\mathbf{z}.
\end{align*}

Using this equation, 
it is now sufficient to check 
the conditions of Corollary \ref{cor:gen_energy_est}. 
To this end, we write 
\begin{align*}
    \| \widetilde G_\mathbf{z} q_{z_i} 
    - \widetilde G_\mathbf{z} q_{\overline{z_i}}
    \|_{L^2( \rho_{\mathbf{z}^{(i)}} )}
    &\leq 
    \| \de_\alpha G_{k,\mathbf{z}} q_{z_i}
        - \de_\alpha G_{k,\overline{\mathbf{z}}} q_{\overline{z_i}} 
    \|_{L^2( \rho_{\mathbf{z}^{(i)}} )} \\
    &\quad 
        + \| \langle \nabla \de_\alpha F_{\mathbf{z}^{(i)}}, 
        \nabla \mu_{k,\mathbf{z}} q_{z_i} 
        - \nabla \mu_{k,\overline{\mathbf{z}}} q_{\overline{z_i}}
        \rangle
    \|_{L^2( \rho_{\mathbf{z}^{(i)}} )} \\
    &\quad 
        + \frac{1}{n} \| \langle \nabla \de_\alpha f(x, z_i),
            \nabla \mu_{k,\mathbf{z}} q_{z_i} 
            \rangle \|_{L^2( \rho_{\mathbf{z}^{(i)}} )} \\
    &\quad 
        + \frac{1}{n} \| \langle \nabla \de_\alpha f(x, \overline{z_i}),
            \nabla \mu_{k,\overline{\mathbf{z}}} q_{\overline{z_i}} 
            \rangle \|_{L^2( \rho_{\mathbf{z}^{(i)}} )} \\
    &= T_1 + T_2 + T_3 + T_4 \, .
\end{align*}

Notice that $T_1$ is bounded by assumption, 
and $T_3, T_4$ are bounded by 
Lemma \ref{lm:appendix_moment_bound}. 
We now turn to $T_2$, denoting 
$\phi = \de_\alpha F_{\mathbf{z}^{(i)}}$, 
we can write
\begin{align*}
    T_2
    &\leq \| \, | \nabla \phi | \,
        | \nabla \mu_{k,\mathbf{z}} q_{z_i} 
            - \nabla \mu_{k,\overline{\mathbf{z}}} q_{\overline{z_i}} | \,
        \|_{L^2( \rho_{\mathbf{z}^{(i)}} )} \\
    &= \sum_{j=1}^d \| \de_j \phi | \nabla \mu_{k,\mathbf{z}} q_{z_i} 
            - \nabla \mu_{k,\overline{\mathbf{z}}} q_{\overline{z_i}} | \,
        \|_{L^2( \rho_{\mathbf{z}^{(i)}} )} \\
    &\leq \sum_{j=1}^d \| \nabla ( \de_j \phi \mu_{k,\mathbf{z}}) q_{z_i} 
            - \nabla (\de_j \phi \mu_{k,\overline{\mathbf{z}}}) q_{\overline{z_i}} 
        \|_{L^2( \rho_{\mathbf{z}^{(i)}} )} \\
    &\quad + \sum_{j=1}^d \| \nabla \de_j \phi \mu_{k,\mathbf{z}} q_{z_i} 
            - \nabla \de_j \phi \mu_{k,\overline{\mathbf{z}}} q_{\overline{z_i}} 
        \|_{L^2( \rho_{\mathbf{z}^{(i)}} )} \, \\
    &\leq \frac{1}{n} \sum_{j=1}^d C_{\nabla (\de_j\phi \mu_k)} 
        + C_{|\nabla \de_j \phi| \mu_k} \, ,
\end{align*}
where both bounds follow from 
Lemma \ref{lm:energy_est1_pol}.
To summarize we have the following bound 
\[ \| \widetilde G_\mathbf{z} q_{z_i} 
    - \widetilde G_\mathbf{z} q_{\overline{z_i}}
    \|_{L^2( \rho_{\mathbf{z}^{(i)}} )}
    \leq 
    \frac{1}{n}
    \left[ C_{\de_\alpha G_k} 
    + \sum_{j=1}^d C_{\nabla (\de_j\phi \mu_k)} 
    + C_{|\nabla \de_j \phi| \mu_k}
    + \|h_1\|_{L^2( \rho_{\mathbf{z}^{(i)}} )}
    + \|h_2\|_{L^2( \rho_{\mathbf{z}^{(i)}} )}
    \right] \, ,
\]
where we define 
$h_1 := |\langle \nabla \de_\alpha f(x, z_i),
            \nabla \mu_{k,\mathbf{z}} q_{z_i} 
            \rangle|, 
h_2 := |\langle \nabla \de_\alpha f(x, \overline{z_i}),
            \nabla \mu_{k,\overline{\mathbf{z}}} q_{\overline{z_i}} 
            \rangle|$.

\vspace{0.2cm}
\noindent
\textbf{Step 2. Induction Step}

Assuming the estimates in the Lemma statement are true for $1, 2, \cdots, J-1$, 
we will now prove the inequality for the case $J$.
We begin by computing the product rule
\[ \de_\alpha (L_\mathbf{z} \mu_{k, \mathbf{z}})
    = L_{\mathbf{z}} \de_\alpha \mu_{k, \mathbf{z}}
    + \sum_{l=1}^{J} \sum_{
    \substack{\alpha_1 + \alpha_2 = \alpha \\
        |\alpha_1| = l }
    }
    (\de_{\alpha_1} L_\mathbf{z}) (\de_{\alpha_2} \mu_{k, \mathbf{z}}), 
\]
where we define for all $|\alpha| > 0$
\[ (\de_\alpha L_\mathbf{z}) \phi := 
    - \langle \nabla \de_\alpha F_\mathbf{z}, \nabla \phi \rangle. 
\]
Here we also observe the Laplacian term is 
contained in $L_{\mathbf{z}} \de_\alpha \mu_{k, \mathbf{z}}$.

By invoking the original equation 
$L_\mathbf{z} \mu_{k,\mathbf{z}} = G_{k, \mathbf{z}}$, 
we can write a new Poisson equation of the form 
\begin{equation}
\label{eq:poisson_higher_order}
    L_{\mathbf{z}} \de_\alpha \mu_{k, \mathbf{z}} 
    = \de_\alpha G_{k, \mathbf{z}}
    - \sum_{l=1}^{J} \sum_{
    \substack{\alpha_1 + \alpha_2 = \alpha \\
        |\alpha_1| = l }
    }
    (\de_{\alpha_1} L_\mathbf{z}) (\de_{\alpha_2} \mu_{k, \mathbf{z}}), 
\end{equation}
where we note that $|\alpha_2| < J$ on the right hand side. 

By using Lemma \ref{lm:energy_est1} 
and Corollary \ref{cor:gen_energy_est}, 
we observe it is sufficient to provide an 
$L^2(\rho_{\mathbf{z}^{(i)}})$-norm bound on 
quantities of the type 
\[ g := (\de_{\alpha_1} L_\mathbf{z}) (\de_{\alpha_2} \mu_{k, \mathbf{z}}) q_{z_i}
    - (\de_{\alpha_1} L_{\overline{\mathbf{z}}}) 
    (\de_{\alpha_2} \mu_{k, {\overline{\mathbf{z}}}}) q_{\overline{z_i}} . 
\]

Then we can rewrite $g$ by decomposing into 
more familiar terms
\begin{align*}
    g &= - \langle \nabla \de_{\alpha_1} F_{\mathbf{z}^{(i)}}, 
        (\nabla \de_{\alpha_2} \mu_{k, \mathbf{z}}) q_{z_i} 
        - (\nabla \de_{\alpha_2} \mu_{k, {\overline{\mathbf{z}}}}) q_{\overline{z_i}}
        \rangle \\
    &\quad - \frac{1}{n} \langle \nabla \de_{\alpha_1} f(x,z_i), 
        (\nabla \de_{\alpha_2} \mu_{k, \mathbf{z}}) q_{z_i} 
        \rangle \\
    &\quad + \frac{1}{n} \langle \nabla \de_{\alpha_1} f(x,\overline{z_i}), 
        (\nabla \de_{\alpha_2} \mu_{k, {\overline{\mathbf{z}}}}) q_{\overline{z_i}}
        \rangle \\
    &=: T_1 + T_2 + T_3.
\end{align*}

To control $T_1$, we start by denoting 
$\phi = | \nabla \de_{\alpha_1} F_{\mathbf{z}^{(i)}} |$
and use Cauchy-Schwarz and the product rule to get
\begin{align*}
    \|T_1\|_{ L^2 ( \rho_{ \mathbf{z}^{(i)}} ) }
    &\leq 
    \| \phi
        (\nabla \de_{\alpha_2} \mu_{k, \mathbf{z}}) q_{z_i} 
    - \phi(\nabla \de_{\alpha_2} \mu_{k, {\overline{\mathbf{z}}}}) q_{\overline{z_i}} 
    \|_{ L^2 ( \rho_{ \mathbf{z}^{(i)}} ) } \\
    &\leq 
    \| \nabla (\phi \de_{\alpha_2} \mu_{k, \mathbf{z}}) q_{z_i} 
    - \nabla(\phi \de_{\alpha_2} \mu_{k, {\overline{\mathbf{z}}}}) q_{\overline{z_i}} 
    \|_{ L^2 ( \rho_{ \mathbf{z}^{(i)}} ) } \\
    &\quad +
    \| \nabla \phi
        ( \de_{\alpha_2} \mu_{k, \mathbf{z}}) q_{z_i} 
    - \nabla \phi( \de_{\alpha_2} \mu_{k, {\overline{\mathbf{z}}}}) q_{\overline{z_i}} 
    \|_{ L^2 ( \rho_{ \mathbf{z}^{(i)}} ) } \\
    &\leq
        \frac{C_{\nabla (\phi \de_{\alpha_2} \mu_k) }}{n}
        + \sum_{j=1}^d \frac{C_{\de_j \phi \de_{\alpha_2} \mu_k }}{n},
\end{align*}
where we get the bound from the induction assumption 
since $|\alpha_2| < J$. 

To control $T_2, T_3$, it is sufficient to apply 
Lemma \ref{lm:appendix_moment_bound}. 

To complete the proof, 
it is sufficient to invoke Lemma \ref{lm:energy_est1_pol}
on \cref{eq:poisson_higher_order}, 
so that we can handle operators of the form 
$L = \phi^\alpha \de_\alpha$, 
where $\phi^\alpha \in C^\infty_{\pol}(\R^d)$.

\end{proof}

\subsection{Proof of Lemma \ref{lm:coefficients}}
\label{subsec:coefficients}

We will once again state the Lemma for easier reference. 

\begin{lemma}
[$C^\infty_{\pol}(\R^d)$ Coefficients]
For all $\ell \geq 0$ and $\alpha \in \N^d$, 
the operator $L_{\ell, \mathbf{z}}^*$ has $C^\infty_{\pol}(\R^d)$ coefficients. 
I.e. there exist functions $\phi_{\ell, \alpha} \in C^{\infty}_{pol}(\R^d)$ 
such that we can write 
\[
L^*_{\ell, \mathbf{z}}  = \sum_{0 \leq \abs{\alpha} \leq 2\ell + 2}  \phi_{\ell, \alpha}(x)\partial_{\alpha} \,.
\]
\end{lemma}

\begin{proof}

We start by writing out the recursive definition of the operator 
$A_{N+1}$, where if $A_N = \sum_{\alpha} A_N^\alpha(x) \de_\alpha$, 
then we have
\[ A_{N+1} = \sum_\alpha \sum_{j=1}^d A_N^\alpha(x) 
    \left( - \de_j F_\mathbf{z} 
    \de_j + A_N^\alpha \frac{1}{\beta} \de_{jj} 
    \right) \de_\alpha.
\]

Since the only coefficients of $L_\mathbf{z}$ are $F_\mathbf{z}$ and $1/\beta$, 
all products of such coefficients must also be $C^\infty_{\pol}(\R^d)$. 

Next we will prove $L_N$ has $C^\infty_{\pol}(\R^d)$
coefficients by induction on $N$. 
The $N=0$ case follows trivially since $L_0 = I$. 
Now assuming the case for $0,1, \ldots, N-1$, 
we will prove the case for $L_N$. 

We recall the definition of $L_N$
\[ L_N = A_{N+1} + \sum_{l=1}^N \frac{B_l}{l!} 
    \sum_{n_1 + \cdots n_{l+1} = n-l} L_{n_1} \cdots L_{n_l} A_{n_{l+1} + 1 }
    \, . 
\]

Observe any composition of operators with $C^\infty_{\pol}(\mathbb{R}^d)$ is still an operator with $C^\infty_{\pol}(\mathbb{R}^d)$ coefficients. 
And since $L_N$ is defined recursively using $L_j$ with $j<N$, we have that $L_N$ must have $C^\infty_{\pol}(\mathbb{R}^d)$ coefficients. 

We will now compute the adjoint operator using integration by parts. First we let $\psi_\alpha \in C^\infty_{\pol}(\mathbb{R}^d)$ be the coefficients of $L_N$, i.e. 
\[ L_N = \sum_{0 \leq |\alpha| \leq 2N + 2}
    \psi_\alpha \de_\alpha \,.
\]

Then for any $f,g \in C^\infty_{\pol}(\mathbb{R}^d)$, we have 
\begin{align*}
    \int_{\mathbb{R}^d} f L_N g \rho dx 
    &= \int_{\mathbb{R}^d} g \sum_{0 \leq \abs{\alpha} \leq 2N + 2} \psi_{\alpha}(x)\partial_{\alpha} g \rho dx \\
    &= \int_{\mathbb{R}^d} \sum_{0 \leq \abs{\alpha} \leq 2N + 2} (-1)^{|\alpha|} f \de_\alpha
    (g \psi_\alpha \rho) dx \,.
\end{align*}

Since $g, \psi_\alpha, F_\mathbf{z} \in C^\infty_{\pol}(\mathbb{R}^d)$, we must then also have that $\de_\alpha (g \psi_\alpha \rho) =: \phi_\alpha \rho$ with $\phi_\alpha \in C^\infty_{\pol}(\mathbb{R}^d)$, which is the desired result.

\end{proof}

\subsection{Proof of Lemma \ref{lm:closeness_of_operators}}
\label{subsec:closeness_of_operators}

We will once again start by restating the Lemma.

\begin{lemma} 
For all $\ell > 0$, there exists a differential operator $\widehat L_\ell:= \widehat L_\ell(x)$ of order $2\ell+2$ with $C^\infty_{\pol}(\mathbb{R}^d)$ coefficients, and independent of $n$, 
such that we can write 
\[ \frac{1}{n} \widehat L_\ell^* 
    = L^*_{\ell, \mathbf{z}} - L^*_{\ell, \overline{\mathbf{z}}} \,.
\]
Furthermore, for all $\ell \geq 0$, 
$\phi \in C^\infty_{\pol}(\R^d)$, 
and $\alpha \in \mathbb{N}^d$, 
there exist non-negative constants $C_{\ell, \phi \de_\alpha}$, depending on the $L^2\left( \rho_{\mathbf{z}^{(i)}} \right)$-norm of $\mu_{k-\ell, \overline{\mathbf{z}}}$ and its derivatives up to order $2\ell+2 + |\alpha|$, such that
\[
\norm{\phi \de_\alpha (( L^*_{\ell, \mathbf{z}} - L^*_{\ell, \overline{\mathbf{z}}} )
            \mu_{k-\ell, \overline{\mathbf{z}}}) q_{\overline{z_i}} }_{ L^2\left( \rho_{\mathbf{z}^{(i)}} \right)} \leq \frac{C_{\ell,\phi\de_\alpha}}{n} \,.
\]
\end{lemma}

\begin{proof}

We will separate the proof into several steps. 

\vspace{0.2cm}
\noindent
\textbf{Step 1. Write 
$\frac{1}{n} \widehat A_N = A_{N, \mathbf{z}} - A_{N, \overline{\mathbf{z}}}$}

We start by recalling $\zeta_0$ is a uniform subsample 
of $\mathbf{z}$ of size $n_b$, 
but we can further define without loss of generality 
$\overline{\zeta_0}$ 
as a uniform subsample of $\overline{\mathbf{z}}$, 
such that they $\zeta_0 = \overline{\zeta_0}$ 
whenever $z_i \notin \zeta_0$. 
Then we have 
\[ A_{N, \mathbf{z}} = \mathbb{E}_\mathbf{z} 
    \widetilde{A}_{N,\zeta_0},
\]
where the expectation is over the randomness of 
$\zeta_0$ only. 

We will then write out the recursive definition of the operator 
$\widetilde{A}_{N+1}(x, \zeta_0)$, 
where we expand the coefficients of the operator as 
$\widetilde{A}_N(x, \zeta_0) = 
\sum_{\alpha} \widetilde{A}_N^\alpha(x,\zeta_0) \de_\alpha$, 
then we can write
\[ \widetilde{A}_{N+1}(x,\zeta_0) 
  = \sum_\alpha \sum_{j=1}^d \widetilde{A}_N^\alpha(x,\zeta_0) 
    \left( - \de_j F_{\zeta_0} \de_j 
      + \frac{1}{\beta} \de_{jj} 
    \right) \de_\alpha.
\]

Since $A_0 = I, A_1 = L$, we can observe this forms a binomial type expansion, 
where \textbf{formally} if we define a sense of multiplication such that 
the differential operators $\de_\alpha$ does not interact with the coefficients, 
we can write
\[ `` \widetilde{A}_N(x,\zeta_0) = \left( \sum_{j=1}^d - \de_j F_{\zeta_0}
    \de_j + \frac{1}{\beta} \de_{jj} 
    \right)^N " \,,
\]
where we write in quotation marks to denote the fact that 
we are defining a new multiply operation. 

Rigorously, we can write all terms of $\widetilde{A}_N(x,\zeta_0)$ as follows 
\begin{align*}
    \widetilde{A}_N(x,\zeta_0) 
    &= \sum_{l=0}^N \binom{N}{l} \sum_{i_1, \ldots, i_N = 1}^d 
    (-1)^l \de_{i_1} F_{\zeta_0} \cdots \de_{i_l} F_{\zeta_0}
        \frac{1}{\beta^{N-l}} \de_{i_1} \cdots \de_{i_l} 
        (\de_{i_{l+1}} \cdots \de_{i_N})^2 \\
    &=: \sum_{\alpha} c_\alpha 
        \de_{i_1} F_{\zeta_0} \cdots \de_{i_l} F_{\zeta_0}
        \de_\alpha \, ,
\end{align*}
where we slightly abuse the notation in the last sum, 
such that the sum is still over $l, i_1, \ldots, i_N$, 
and therefore $\alpha$ may repeat. 

Now consider the same definition for $\overline{\mathbf{z}}$
and $\overline{\zeta_0}$, we can write 
\[ \widetilde{A}_{N}(x,\zeta_0) - \widetilde{A}_{N}(x,\overline{\zeta_0})
    = \sum_{\alpha} c_\alpha 
        (\de_{i_1} F_{\zeta_0} \cdots \de_{i_l} F_{\zeta_0}
        - \de_{i_1} F_{\overline{{\zeta_0}}} \cdots 
        \de_{i_l} F_{\overline{{\zeta_0}}})
        \de_\alpha \, .
\]

To find an operator $\widehat A_N$ such that 
$\frac{1}{n} \widehat A_N = A_{N, \mathbf{z}} - A_{N, \overline{\mathbf{z}}}$, 
it is sufficient to show that for each $\alpha$, 
we can find a function $\phi^\alpha$ independent of $n$ 
such that 
\[ \frac{1}{n} \phi^\alpha 
    = \mathbb{E}_\mathbf{z} \bigg( 
      \de_{i_1} F_{\zeta_0} \cdots \de_{i_l} F_{\zeta_0}
        - \de_{i_1} F_{\overline{{\zeta_0}}} \cdots 
        \de_{i_l} F_{\overline{{\zeta_0}}} 
        \bigg)
        \, .
\] 

To this end, we can use the decompositions
\[ F_\mathbf{z} = F_{\mathbf{z}^{(i)}} + \frac{1}{n} f(x,z_i), \quad 
  F_{\zeta_0} = F_{\zeta_0^{(i)}} + \frac{1}{n_b} f(x,z_i) \mathds{1}_B,
\]
where $\mathds{1}$ is the indicator function, 
and the event $B$ is defined as 
$B = \{z_i \in \zeta_0\} = \{\overline{z_i} \in \overline{\zeta_0}\}$.
 
This leads to 
\begin{align*}
    &\quad \de_{i_1} F_{\zeta_0} \cdots \de_{i_l} F_{\zeta_0}
        - \de_{i_1} F_{\overline{{\zeta_0}}} \cdots 
        \de_{i_l} F_{\overline{{\zeta_0}}} \\
    &= 
    \de_{i_1} F_{\zeta_0^{(i)}} 
    ( \de_{i_2} F_{\zeta_0} \cdots \de_{i_l} F_{\zeta_0}
            - \de_{i_2} F_{\overline{\zeta_0}} \cdots 
            \de_{i_l} F_{\overline{\zeta_0}} ) \\
    &\quad + \frac{1}{n_b} \mathds{1}_B
            \left(\de_{i_1} f(x,z_i)
            \de_{i_2} F_{\zeta_0} \cdots \de_{i_l} F_{\zeta_0}
            - \de_{i_1} f(x, \overline{z_i}) 
            \de_{i_2} F_{\overline{\zeta_0}} \cdots 
            \de_{i_l} F_{\overline{\zeta_0}}
            \right),
\end{align*}
where we observe that the terms inside the second bracket 
are $n$-independent. 
If we take the expectation $\mathbb{E}_\mathbf{z}$ on these terms, 
we will be averaging over $\binom{n}{n_b}$ terms, 
and only $\frac{n_b}{n} \binom{n}{n_b} = \binom{n-1}{n_b - 1}$ 
of these terms will be non-zero 
due to the indicator function $\mathds{1}_B$. 

Hence if we let $\phi_{\zeta_0}$ be any $n$-independent function, 
we will have 
\begin{equation}
\label{eq:zeta_avg_ratio}
  \mathbb{E}_\mathbf{z} 
    \frac{1}{n_b} \mathds{1}_B
    \phi_{\zeta_0}
  = \frac{1}{n_b} 
    \frac{1}{n} \sum_{\zeta_0 \subset \mathbf{z}} \mathds{1}_B \phi_{\zeta_0} 
  = \frac{1}{n} \widetilde{\phi}, 
\end{equation}
where $\widetilde{\phi}$ is also $n$-independent.

We can continue taking the expansion to get 
\begin{align*}
  &\quad \de_{i_1} F_{\zeta_0} \cdots \de_{i_l} F_{\zeta_0}
        - \de_{i_1} F_{\overline{{\zeta_0}}} \cdots 
        \de_{i_l} F_{\overline{{\zeta_0}}} \\
  &= \frac{1}{n_b} \mathds{1}_B \sum_{j=1}^l
        \de_{i_1} F_{\zeta_0^{(i)}} \cdots 
        \de_{i_{j-1}} F_{\zeta_0^{(i)}} 
        \de_{i_j} f(x,z_i)
        \de_{i_{j+1}} F_{\zeta_0} \cdots
        \de_{i_l} F_{\zeta_0} \\
    &\quad - \frac{1}{n_b} \mathds{1}_B \sum_{j=1}^l
        \de_{i_1} F_{\zeta_0^{(i)}} \cdots 
        \de_{i_{j-1}} F_{\zeta_0^{(i)}} 
        \de_{i_j} f(x,\overline{z_i})
        \de_{i_{j+1}} F_{\overline{\zeta_0}} \cdots
        \de_{i_l} F_{\overline{\zeta_0}} \\
    &=: \frac{1}{n_b} \mathds{1}_B \phi^\alpha_{\zeta_0, \overline{\zeta_0}} \, ,
\end{align*}
where the sum follows from 
recursively applying the first step to the terms 
$\de_{i_1} F_{\zeta_0^{(i)}} 
( \de_{i_2} F_{\zeta_0} \cdots \de_{i_l} F_{\zeta_0}
- \de_{i_2} F_{\overline{\zeta_0}} \cdots 
\de_{i_l} F_{\overline{\zeta_0}} )$.

Computing the expectation using Equation \eqref{eq:zeta_avg_ratio}, 
we get that 
\[ \mathbb{E}_\mathbf{z} \frac{1}{n_b} \mathds{1}_B 
  \phi^\alpha_{\zeta_0, \overline{\zeta_0}}
  = \frac{1}{n} \phi^\alpha.
\]

This implies we can define the desired operator as 
\[ \widehat A_N = \sum_\alpha c_\alpha \phi^\alpha \de_\alpha,
\]
hence completing the proof for the first step.

\vspace{0.2cm}
\noindent
\textbf{Step 2. Finding $\widehat L_N$}

We will prove by induction over $N$
that there exists an operator $\widehat L_N$ 
independent of $n$ such that 
$\frac{1}{n} \widehat L_N = L_{N,\mathbf{z}} - L_{N, \overline{\mathbf{z}}}$.

For the $N=1$ case, it follows by computing
\[ L_{N,\mathbf{z}} - L_{N, \overline{\mathbf{z}}}
    = L_\mathbf{z} - L_{\overline{\mathbf{z}}}
    = \frac{1}{n} \sum_{j=1}^d 
        - \de_j ( f(x, z_i) - f(x, \overline{z_i}) ) \de_j.
\]

Now assuming the case for $N-1$, 
we will prove the statement for case $N$.
We first recall the definition of $L_N$
\[ L_N = A_{N+1} + \sum_{l=1}^N \frac{B_l}{l!} 
    \sum_{n_1 + \cdots n_{l+1} = n-l} L_{n_1} \cdots L_{n_l} A_{n_{l+1} + 1 }
    \, . 
\]

Since $A_{N, \mathbf{z}} - A_{N, \overline{\mathbf{z}}} 
= \frac{1}{n} \widehat A_N$, 
it is sufficient to analyze 
\begin{align*}
    &\quad
    L_{n_1, \mathbf{z}} \cdots L_{n_l, \mathbf{z}} A_{n_{l+1} + 1, \mathbf{z} }
    - L_{n_1, \overline{\mathbf{z}}} \cdots L_{n_l, \overline{\mathbf{z}}}
    A_{n_{l+1} + 1, \overline{\mathbf{z}} } \\
    &= ( L_{n_1, \mathbf{z}} \cdots L_{n_l, \mathbf{z}}
        - L_{n_1, \overline{\mathbf{z}}} \cdots L_{n_l, \overline{\mathbf{z}}} )
     A_{n_{l+1} + 1, \mathbf{z} }
     + L_{n_1, \overline{\mathbf{z}}} \cdots L_{n_l, \overline{\mathbf{z}}}
        \frac{1}{n} \widehat A_{n_{l+1} + 1}
\end{align*}

Now observe that since $n_1, \cdots, n_l < N$, 
we have $\frac{1}{n} \widehat L_{n_j} 
= L_{n_j, \mathbf{z}} - L_{n_j, \overline{\mathbf{z}}}$
for all $j = 1, \cdots, l$.
This allows us to write  
\begin{align*}
    &\quad L_{n_1, \mathbf{z}} \cdots L_{n_l, \mathbf{z}}
        - L_{n_1, \overline{\mathbf{z}}} \cdots L_{n_l, \overline{\mathbf{z}}} \\
    &= ( L_{n_1, \mathbf{z}} - L_{n_1, \overline{\mathbf{z}}} )
        L_{n_2, \mathbf{z}} \cdots L_{n_l, \mathbf{z}}
        + L_{n_1, \overline{\mathbf{z}}} 
            ( L_{n_2, \mathbf{z}} \cdots L_{n_l, \mathbf{z}} 
        - L_{n_2, \overline{\mathbf{z}}} \cdots L_{n_l, \overline{\mathbf{z}}} )
        \\
    &= \frac{1}{n} \widehat L_{n_1} 
        L_{n_2, \mathbf{z}} \cdots L_{n_l, \mathbf{z}}
        + L_{n_1, \overline{\mathbf{z}}} 
        ( L_{n_2, \mathbf{z}} \cdots L_{n_l, \mathbf{z}}
        - L_{n_2, \overline{\mathbf{z}}} \cdots L_{n_l, \overline{\mathbf{z}}}
        ) \\
    &= \frac{1}{n} \sum_{j=1}^l 
        L_{n_1, \overline{\mathbf{z}}}
        \cdots 
        L_{n_{j-1}, \overline{\mathbf{z}}}
        \widehat L_{n_j}
        L_{n_{j+1}, \mathbf{z}} \cdots L_{n_l, \mathbf{z}}
\end{align*}

Putting it together we have 
\begin{align*}
    &\quad \widehat L_N \\
    &= \widehat A_{N+1} + \sum_{l=1}^N \frac{B_l}{l!} 
    \sum_{n_1 + \cdots n_{l+1} = n-l}
    \sum_{j=1}^l 
        L_{n_1, \overline{\mathbf{z}}}
        \cdots 
        L_{n_{j-1}, \overline{\mathbf{z}}}
        \widehat L_{n_j}
        L_{n_{j+1}, \mathbf{z}} \cdots L_{n_l, \mathbf{z}}
    A_{n_{l+1} + 1, \mathbf{z} } \\
    &\quad 
    + \sum_{l=1}^N \frac{B_l}{l!} 
        \sum_{n_1 + \cdots n_{l+1} = n-l}
        L_{n_1, \overline{\mathbf{z}}} \cdots L_{n_l, \overline{\mathbf{z}}}
        \widehat A_{n_{l+1} + 1} \, .
\end{align*}

Since the adjoint operation is linear, 
we can compute the adjoint for each operator separately. 
Hence we have the desired result 
\[ \frac{1}{n} \widehat L^*_N 
    = L_{N, \mathbf{z}} - L_{N, \overline{\mathbf{z}}} \, .
\]

\vspace{0.2cm}
\noindent
\textbf{Step 3. Providing the Bound}

To complete the bound it is sufficient 
to observe 
\[ \norm{\phi \de_\alpha (( L^*_{\ell, \mathbf{z}} - L^*_{\ell, \overline{\mathbf{z}}} )
            \mu_{k-\ell, \overline{\mathbf{z}}}) q_{\overline{z_i}} }_{ L^2\left( \rho_{\mathbf{z}^{(i)}} \right)}
    = \frac{1}{n} 
    \norm{\phi \de_\alpha ( \widehat L_N^*
            \mu_{k-\ell, \overline{\mathbf{z}}}) q_{\overline{z_i}} }_{ L^2\left( \rho_{\mathbf{z}^{(i)}} \right)}
    \leq \frac{C}{n}, 
\]
where the bound follows from 
Lemma \ref{lm:appendix_moment_bound}.

\end{proof}

\section{Runtime Complexity: Proof of \cref{cor:runtime}}
\label{sec:appendix_proof_corollary} 

We will restate and prove the result here. 

\begin{corollary}
[Runtime Complexity]
Suppose $\{X_k\}_{k\geq 0}$ is any discretization 
of Langevin diffusion \eqref{eq:langevin_diffusion} 
admiting an approximate stationary distribution $\pi^N_{\mb{z}}$ 
of the type in \cref{thm:backward_analysis_main}. 
Then there exists a constant $C>0$ (depending on $N$), 
such that for all 
\begin{equation}
    \epsilon > 0 \,, \quad 
	0 < \eta < \min\left\{\frac{2m}{M^2} \,, C \epsilon^{1/N} \right\} \,, \quad 
	n \geq \frac{C}{\epsilon (1-\eta)} \,, \quad 
	k \geq \frac{C}{\epsilon^{1/N}} \log \frac{1}{\epsilon} \,, 
\end{equation}
we achieve the following expected generalization bound 
\begin{equation}
	\left| \, 
		\mathbb{E} 
		\left[ F( X_k )
			- F_{\mb{z}}( X_k )
		\right] \, 
	\right| 
	\leq \epsilon \,. 
\end{equation}
\end{corollary}

\begin{proof}

We will start by decomposing the generalization error 
via an approximation step with $\pi^N_{\mb{z}}$ 
\begin{equation}
\begin{aligned}
	& \left| \, 
		\mathbb{E} 
		\left[ F( X_k )
			- F_{\mb{z}}( X_k )
		\right] \, 
	\right| \\
	\leq& 
		\left| \, 
			\mathbb{E} 
			\left[ F( X_k )
				- \pi^N_{\mb{z}}( F ) 
			\right] \, 
		\right| 
		+ 
		\left| \, 
			\mathbb{E} 
			\left[ 
				\pi^N_{\mb{z}}( F ) 
				- 
				\pi^N_{\mb{z}}( F_{\mb{z}} ) 
			\right] \, 
		\right| 
		+ 
		\left| \, 
			\mathbb{E} 
			\left[ 
				\pi^N_{\mb{z}}( F_{\mb{z}} ) 
				- 
				F_{\mb{z}}( X_k ) 
			\right] \, 
		\right| 
		\\ 
	\leq& 
		C ( e^{-\lambda k \eta / 2} + \eta^N ) 
		+ 
		\frac{ C }{ n (1-\eta) } 
		\,, 
\end{aligned}
\end{equation}
where we used the results of \cref{thm:backward_analysis_main,thm:pi_gen_bound} 
and absorbed dependence on $x, F, F_{\mb{z}}$ 
into the constant $C$. 

It is then sufficient to confirm all three terms are of order $O(\epsilon)$. 
Observe clearly choosing $\eta = O(\epsilon^{1/N})$ 
and $n = \Omega( \epsilon^{-1} )$ is sufficient for the latter two terms. 
We will then observe that 
\begin{equation}
	e^{-\lambda k \eta / 2} = O(\epsilon) 
	\iff 
	k = \Omega( \eta^{-1} \log(\epsilon^{-1}) ) \,, 
\end{equation}
and substituting in the choice of $\eta^{-1} = \epsilon^{-1/N}$ 
gives us the desired result. 

\end{proof}

\appendix

\addcontentsline{toc}{section}{References}
\bibliographystyle{plainnat}
\bibliography{backward.bib}

\section{Signed Measure Results}
\label{sec:appendix_signed_measure}

We start this section by mentioning that the approximate stationary measure 
constructed using weak backward error analysis   
can be a signed measure. 

\begin{proposition}
[The Approximate Stationary Measure is Signed]
\label{prop:signed_measure}
Consider the following simple Ornstein-Uhlenbeck process
in $\mathbb{R}^d$
\[ dX(t) = - X(t) dt + \sqrt{\frac{2}{\beta}} dW(t),
\]
with $f(x, z) = \frac{1}{2} |x|^2$ for all $z \in \mathcal{Z}$. 
We have that whenever $\frac{2}{d} < \eta$, 
the approximate density $\pi^{1}(x)$ is not always positive, 
specifically 
\[ \pi^{1}(0) = (1 + \eta \mu_1(0)) \rho(0) < 0.
\] 
\end{proposition}

\begin{proof}

We first write down the generator of this process 
\[ L \phi = \sum_{i=1}^d - x_i \de_i \phi + \frac{1}{\beta} \de_{ii} \phi,
\]
and therefore leading to the following operator $A_2$,
\begin{align*}
    A_2 \phi &= 
        \sum_{\mathbf{k}} \sum_{i=1}^d
        A^\mathbf{k}_1(x) \left[
        (-x_i) \de_i + \frac{1}{\beta} \de_{ii} 
        \right] 
        \de_\mathbf{k} \phi \\
    &= \sum_{\mathbf{k}} \sum_{i,j=1}^d
    x_i x_j \de_{ij} \phi - \frac{2}{\beta} x_i \de_{ijj} \phi
        + \frac{1}{\beta^2} \de_{iijj} \phi.
\end{align*}

Following the definition of $L_1$ \eqref{eq:appendix_L_j}, we can compute 
\[ L_1 = A_2 + B_1 L_0 A_1 = A_2 + B_1 L^2,
\]
where $B_1$ is the first Bernoulli number. 
Now observe that since $L$ is self-adjoint with respect to $L^2(\rho)$, 
i.e. $L^\star = L$, we have
\[ L_1^\star 1 = A_2^\star + B_1 (L^\star)^2 1 = A_2^\star 1. \]

Here we will compute explicitly $A_2^\star 1$, and we start by writing
\[ (A_2^\star 1) \rho = \sum_{i,j=1}^d 
    \de_{ij}(x_i x_j \rho) + \frac{2}{\beta} \de_{ijj} (x_i \rho)
    + \frac{1}{\beta^2} \de_{iijj} \rho
    =: T_1 + T_2 + T_3.
\]

We first compute one derivative of $T_3$ to match $T_2$ to get 
\[ T_3 = \sum_{i,j=1}^d \de_{ijj} \frac{1}{\beta} ( -x_i \rho),
\]
this implies that we can combine $T_2$ and $T_3$
\[ T_2 + T_3 = \sum_{i,j=1}^d  \de_{ijj} \frac{1}{\beta} (x_i \rho)
    = \sum_{i,j=1}^d  
        \de_{ij} \left( \frac{\delta_{ij}}{\beta} \rho - x_i x_j \rho \right),
\]
where $\delta_{ij}$ denotes the Kronecker delta, 
and we add the $T_1$ term as well to get 
\[ T_1 + T_2 + T_3 = 
    \sum_{i,j=1}^d  \de_{ij} \frac{\delta_{ij}}{\beta} \rho
    = \sum_{i=1}^d \de_{ii} \frac{1}{\beta} \rho 
    = \sum_{i=1}^d \de_i \left( - x_i \rho \right)
    = \sum_{i=1}^d \beta x_i^2 \rho - d \rho.
\]

This implies we have the following PDE for $\mu_1$
\[ L \mu_1 = - L_1^\star 1 \implies
    \sum_{i=1}^d - x_i \de_i \mu_1 + \frac{1}{\beta} \de_{ii} \mu_1 = 
    d - \sum_{i=1}^d \beta x_i^2.
\]

Since the equation has a unique solution that satisfies the integral constraint 
$\int \mu_1 d \rho = 0$ 
(see Proposition \ref{prop:poisson_exist_unique}), 
we can explicitly guess the solution 
\[ \mu_1 = \frac{\beta}{2} |x|^2 - c,
\]
where to satisfy the integral constraint, we must have 
\[ c = \int \frac{\beta}{2} |x|^2 \frac{1}{ (2\pi / \beta )^{d/2}} 
    \exp\left( \frac{ -\beta |x|^2 }{2} \right) dx = \frac{d}{2}. 
\]

This implies 
\[ \inf_x \mu_1 = \mu_1(0) = - \frac{d}{2}.
\]

Finally this implies that whenever $\eta > \frac{2}{d}$, 
we have
\[ \pi^{1}(0) = \left(1 - \eta \frac{d}{2} \right) \rho(0) < 0.
\]

\end{proof}

Now that we know our approximate measure $\pi^N$ can be signed, 
we can no longer define uniform stability with respect to 
an expectation over the random algorithm. 
Instead we naturally extend the definition based on 
the integral with respect to the signed measure instead. 

\begin{definition}[Uniform Stability]
\label{def:uniform_stability}
A collection of distributions $\{\pi_\mathbf{z}\}$ 
on $\mathbb{R}^d$ indexed by $\mathbf{z} \in \mathcal{Z}^n$ 
is said to be \textbf{$\epsilon$-uniformly stable} if 
for all $\mathbf{z}, \overline{\mathbf{z}} \in \mathcal{Z}^n$
with only one differing coordinate 
\begin{equation}
    \sup_{z \in \mathcal{Z}} 
    \left| 
        \pi_{\mb{z}}( f(\cdot, z) ) - 
        \pi_{\overline{\mb{z}}}( f(\cdot, z) )
    \right| 
    \leq \epsilon.
\end{equation}
\end{definition}

\begin{proposition}
[Generalization]
Suppose the collection of distributions $\{\pi_\mathbf{z}\}$
is $\epsilon$-uniformly stable, 
and that for all $(\mathbf{z}, z) \in \mathcal{Z}^{n+1}$, 
we also have $f(\cdot, z) \in L^1( \pi_{\mb{z}} )$. 
Then $\{\pi_\mathbf{z}\}$ has $\epsilon$-expected generalization error, 
or more precisely 
\begin{equation}
    \left| \, 
        \mathbb{E} \left[
        \pi_\mathbf{z}( F_{\mb{z}} )
        - 
        \pi_\mathbf{z}( F ) 
        \right] \, 
    \right| 
    \leq 
        \epsilon \,.
\end{equation}
\end{proposition}

\begin{proof}
it is sufficient to realize that since $f(\cdot, z)$
is integrable with respect to $\pi_\mathbf{z}$, 
all the usual manipulations are well defined. 
We will include the proof for completeness 
as it is an unintuitive claim.  

We will start by denoting $\mathbf{z} = (z_1, \ldots, z_n)$, 
$\widehat{\mathbf{z}} = (\overline{z_1}, \ldots, \overline{z_n})$, 
and also the replaced one data set
$\widehat{\mathbf{z}}^{(i)} = (z_1, \ldots, z_{i-1}, \overline{z_i}, z_{i+1}, \ldots, z_n)$.

With this we can write 
\begin{align*}
    \mathbb{E}_{\mathbf{z}\sim \mathcal{D}^n} \int_{\mathbb{R}^d}
        \frac{1}{n} \sum_{i=1}^n f(x, z_i) d\pi_\mathbf{z}(x)
    &= \mathbb{E}_{(\mathbf{z}, \widehat{\mathbf{z}}) \sim \mathcal{D}^{2n}} 
        \int_{\mathbb{R}^d} \frac{1}{n} \sum_{i=1}^n 
        f(x, \overline{z_i}) d\pi_{\widehat{\mathbf{z}}^{(i)}}(x) \\
    &= \mathbb{E}_{(\mathbf{z}, \widehat{\mathbf{z}}) \sim \mathcal{D}^{2n}} 
        \int_{\mathbb{R}^d} \frac{1}{n} \sum_{i=1}^n 
        f(x, \overline{z_i}) d\pi_{\mathbf{z}}(x) + \delta,
\end{align*}
where we define $\delta$ as 
\[ \delta := 
    \mathbb{E}_{(\mathbf{z}, \widehat{\mathbf{z}}) \sim \mathcal{D}^{2n}} 
        \int_{\mathbb{R}^d} \frac{1}{n} \sum_{i=1}^n 
        f(x, \overline{z_i}) d\pi_{\widehat{\mathbf{z}}^{(i)}}(x)
        - \int_{\mathbb{R}^d} \frac{1}{n} \sum_{i=1}^n 
        f(x, \overline{z_i}) d\pi_{\mathbf{z}}(x).
\]

Therefore, $\epsilon$-uniform stability implies 
the difference between the two integrals is at most $\epsilon$, 
hence we have the generalization error is $|\delta| \leq \epsilon$.

\end{proof}

We observe that uniform stability is a stronger notion
than needed here. In fact, we will only need a sense of stability in expectation for generalization in expectation \citep{SSSSS10}.  
Regardless, we have shown that if $\pi^N_\mathbf{z}$ 
is $\epsilon$-uniformly stable, then we have also have that $|T_2| \leq \epsilon$ as desired.

\section{Additional Calculations}
\label{sec:calc}

In this section, we will consider the Ornstein-Uhlenbeck process 
in $\mathbb{R}$, or $d=1$
\[ dX(t) = - X(t)dt + \sqrt{2/\beta} \, dW(t),
\]
where $W(t)$ is the standard \textbf{one-dimensional} Brownian motion. 
This corresponds to the loss function being $F(x) = \frac{1}{2} x^2$.
This implies we have the Gibbs distribution is a one-dimensional Gaussian 
\[ \rho(x) = \frac{1}{\sqrt{2\pi / \beta}} 
        \exp\left(-\frac{\beta x^2}{2}\right).
\]

\subsection{Computing the Relevant Operators}

The above definition leads us to the following simple generator 
\[ L \phi = -x \de \phi + \frac{1}{\beta} \de^2 \phi, 
\]
and therefore leading to the operators 
\[ A_k = \sum_{\ell = 0}^k \binom{k}{\ell} (-x)^\ell \frac{1}{\beta^{k-\ell}}
  \de^{2k - \ell}.
\]

We will also obtain the operators 
\begin{align*}
  L_0 &= L, \\
  L_1 &= A_2 + B_1 L_0 A_1 = A_2 + B_1 L^2, \\
  L_2 &= A_3 + B_1 (L_0 A_2 + L_1 A_1) + B_2 (L_0 L_0 A_1). 
\end{align*}

Next we will compute several adjoint operators 
with respect to $L^2(\rho)$. 

First we recall $L$ is self-adjoint, i.e. $L^* = L$.

Then we can compute for all $f,g \in C^\infty_{\pol}(\mathbb{R})$
\[ \int \de f g \rho dx 
  = \int - f \de g \rho  - f g \de \rho dx
  = \int f ( - \de + \beta x ) g \rho dx,
\]
which means $(\de)^* = -\de + \beta x$. 

At this point, we will use Mathematica \citep*{Mathematica} 
to compute all adjoint operators by composition of 
the first order adjoint operator. 
For example, we can compute all of the following using a recursion scheme, 
where $a \in C^\infty_{\pol}(\mathbb{R})$
\begin{align*}
  (a \de)^* &= -a \de - \de a + \beta x a, \\
  (\de^2)^* &= \de^2 - 2\beta x \de + \beta^2 x^2 - \beta, \\
  (a \de^2)^* &= a \de^2 + 2( \de a - \beta x a) \de 
      + (\de^2 a - \beta a - 2 \beta x \de a + \beta^2 x^2 a), \\
  (\de^3)^* &= - \de^3 + 3 \beta x \de^2 + (3\beta - 3\beta^2 x^2) \de
      + (\beta^3 x^3 - 3 \beta^2 x).
\end{align*}

Next we will compute $A_2^*, A_3^*$
\begin{align*}
  A_2^* 
    &= \left(x^2 \de^2 - \frac{2 x}{\beta} \de^3 + \frac{1}{\beta^2} \de^4
      \right)^* \\
    &= \frac{1}{\beta^*} \de^4 - \frac{2x}{\beta} \de^3 
      + x^2 \de^2 - 2x \de + (\beta x - 1). \\
  A_3^* 
    &= \left(- x^3 \de^3 + \frac{3x^2}{\beta} \de^4 + \frac{-3x}{\beta^2} \de^5
      + \frac{1}{\beta^3} \de^6 \right)^* \\
    &= \frac{1}{\beta^3} \de^6 - \frac{3x}{\beta^2} \de^5
      + \frac{3x^2}{\beta} \de^4 - \left( x^3 + \frac{6x}{\beta} \right) \de^3
      + \left(9x^2 - \frac{9}{\beta} \right) \de^2 
      - (9x - 3\beta x^3) \de.
\end{align*}

Finally, we can conclude 
\begin{align*}
  L_1^* &= A_2^* + B_1 L^2, \\
  L_2^* &= A_3^* + B_1 ( A_2^* L + L L_1^* ) + B_2 L^3.
\end{align*}

\subsection{Solving the PDEs}

To compute the approximations for $\pi^2 = \rho( 1 + \eta \mu_1 + \eta^2 \mu_2)$, 
we will need to solve the following two PDEs
\[ L \mu_1 = - L_1^* 1, \quad 
    L \mu_2 = - L_1^* \mu_1 - L_2^* 1, 
\]
with the natural constraint that $\int \mu_1 d\rho = \int \mu_2 d\rho = 0$.

\subsubsection{First PDE}

From the previous section, we have that 
\[ - L_1^* 1 = - (A_2^* + B_1 L^2) 1 = - A_2^* 1 = 1 - \beta x^2.
\]

This means we need to first solve 
\[ - x \de \mu_1 + \frac{1}{\beta} \de^2 \mu_1 = 1 - \beta x^2.
\]

Here we can guess the solution $\mu_1 = \frac{\beta}{2} x^2 + c$, 
which gives us 
\[ -x (\beta x) + \frac{1}{\beta} \beta = 1 - \beta x^2,
\]
and the constant is naturally 
\[ c = - \int \frac{\beta}{2} x^2 d\rho = - \frac{\beta}{2} \frac{1}{\beta} 
  = \frac{-1}{2},
\]
where we used the fact that the Gaussian variance is $1/\beta$.

\subsubsection{Second PDE}

We need to compute $ - L_1^* \mu_1 = - (A_2^* + B_1 L^2) \mu_1$, 
starting with the first term $A_2^* \mu_1$
\[ \left( \frac{1}{\beta^2} \de^4 - \frac{2x}{\beta} \de^3 + x^2 \de^2
    - \frac{1}{\beta} \de^2 + \beta x^2 - 1 \right) 
    \left( \frac{\beta}{2} x^2 - \frac{1}{2}
    \right)
    = \frac{1}{2}(\beta x^2 - 1).
\]

\subsection{Toy Example}
\label{subsec:toy_example}

In this subsection, we consider a one-dimensional problem of 
optimizing a loss function $F(x) = \frac{1}{2} x^2$ 
with deterministic gradient. 
The Langevin update can be written as 
\[ X_{k+1} = X_k - \eta X_k + \sqrt{ \frac{2\eta}{\beta} } \xi_k,
    \quad X_0 = x.
\]
where $\xi_k \sim N(0,1)$ are i.i.d. random variables. 

In this case, the corresponding continuous time Langevin process is 
described by the following SDE
\[ dX(t) = - X(t) dt + \sqrt{ \frac{2}{\beta} } dW(t), 
    \quad X(0) = x, 
\]
where $W(t)$ is a standard Brownian motion. 

Observe we can compute the time-marginal distributions of both processes. 
Let's start by observing that $X_k$ is a sum of i.i.d. normal random variables, 
hence it must also be a normal random variable. 
We simply need to estimate the mean and variance. 

To compute the parameters, we start by rewriting the update rule as
\begin{align*}
  X_{k+1} &= X_k - \eta X_k + \sqrt{ \frac{2\eta}{\beta} } \xi_k \\
  &= (1-\eta) X_k + \sqrt{2\eta}{\beta} \xi_k \\
  &= (1-\eta) \left( (1-\eta) X_{k-1} + \sqrt{\frac{2\eta}{\beta}} \xi_{k-1} 
    \right) + \sqrt{2\eta}{\beta} \xi_k \\
  &= (1-\eta)^{k+1} X_0 + \sum_{\ell = 0}^k 
    \sqrt{2\eta}{\beta} (1-\eta)^\ell \xi_\ell, 
\end{align*}
this implies $\mathbb{E}X_k = (1-\eta)^{k+1} x$, 
and we also have the variance 
\[ \mathbb{E} (X_k - \mathbb{E}X_k )^2 
  = \sum_{\ell = 0}^k \frac{2\eta}{\beta} (1-\eta)^{2\ell}
  = \frac{2}{\beta (2-\eta)} \left[ 1 - (1-\eta)^{2(k+1)} \right]. 
\]

It is well known that the continuous time SDE is the Ornstein-Uhlenbeck process, 
which has the following solution \citep{kuo06StocInt}
\[ X(t) = X(0) e^{-t} + \int_0^t e^{-(t-s)} dW(s),
\]
which is a Gaussian process. The mean and variance can also be computed as 
\begin{align*}
  \mathbb{E} X(t) &= x e^{-t}, \\
  \mathbb{E} ( X(t) - \mathbb{E} X(t))^2 &= 
  \int_0^t e^{-2(t-s)} ds 
  = \frac{1}{\beta} (1 - e^{-2t}).
\end{align*}

In the actual plot of Figure \ref{fg:backward_error} (b), 
we used parameters $\beta = 20, \eta = 0.5$. 
Additionally, the stationary distribution of the discrete Langevin algorithm 
is found by a kernel density smoothing for $10,000$ steps of simulation 
of $X_k$, and we smoothed using a normal density with 
parameter $\sigma = 0.1$.

\subsection{Plot for Figure \ref{fg:backward_error} (a)}

In the plot, we consider the forward Euler discretization of the ODE
\[ \frac{dy}{dt} = y^2, \quad y(0) = 1,
\]
which gives the following update
\[ y_{k+1} = y_k + \eta y_k^2, \quad y_0 = 1.
\]

The true solution is given by 
\[ y(t) = \frac{1}{1 - t},
\]
which the modified equation is given by 
\[ \frac{dv}{dt} = v^2 - \eta v^3 + \eta^2 \frac{3}{2} v^4 + \cdots, 
  \quad v(0) = 1.
\]

To generate the plot in Figure \ref{fg:backward_error} (a), 
we used a step size $\eta = \frac{1}{6}$ 
for the forward Euler solver, 
and for the modified equation we approximated the solution 
using forward Euler with step size $\eta' = \frac{1}{600}$ 
so it is sufficiently close to the true solution.

\end{document}